\documentclass{article}

\usepackage{microtype}
\usepackage{graphicx}
\usepackage{subfigure}
\usepackage{booktabs} 
\usepackage{hyperref}

\usepackage[accepted]{arxiv_cut}

\usepackage{amsmath}
\usepackage{amssymb}
\usepackage{mathtools}
\usepackage{amsthm}
\usepackage{multirow}
\usepackage[capitalize,noabbrev]{cleveref}

\theoremstyle{plain}
\newtheorem{theorem}{Theorem}[section]
\newtheorem{proposition}[theorem]{Proposition}
\newtheorem{lemma}[theorem]{Lemma}

\theoremstyle{definition}
\newtheorem{definition}[theorem]{Definition}
\newtheorem{assumption}[theorem]{Assumption}
\newtheorem{hypothesis}[theorem]{Induction Hypothesis}
\newtheorem{parameter}[theorem]{Parameter Assumption}
\newtheorem{fact}[theorem]{Fact}
\newtheorem{claim}[theorem]{Claim}
\theoremstyle{remark}

\usepackage[textsize=tiny]{todonotes}
\newcommand{\openone}{\leavevmode\hbox{\small1\normalsize\kern-.33em1}}

\newcommand{\polylogk}{\mathrm{polylog}(k)}
\newcommand{\polyk}{\mathrm{poly}(k)}
\newcommand{\logit}{\mathbf{logit}}
\newcommand{\Beta}{\text{Beta}}
\newcommand{\Bern}{\text{Bern}}

\newcommand{\calA}{\mathcal{A}}

\newcommand{\calD}{\mathcal{D}}
\newcommand{\calE}{\mathcal{E}}

\newcommand{\calH}{\mathcal{H}}

\newcommand{\calM}{\mathcal{M}}

\newcommand{\calP}{\mathcal{P}}

\newcommand{\calV}{\mathcal{V}}

\newcommand{\calZ}{\mathcal{Z}}

\newcommand{\tcalZ}{\tilde{\calZ}}

\newcommand{\bbE}{\mathbb{E}}

\newcommand{\bbI}{\mathbb{I}}

\DeclareMathAlphabet{\mathbsf}{OT1}{cmss}{bx}{n}
\DeclareMathAlphabet{\mathssf}{OT1}{cmss}{m}{sl}

\DeclareSymbolFont{bsfletters}{OT1}{cmss}{bx}{n}  
\DeclareSymbolFont{ssfletters}{OT1}{cmss}{m}{n}
\DeclareMathSymbol{\bsfGamma}{0}{bsfletters}{'000}
\DeclareMathSymbol{\ssfGamma}{0}{ssfletters}{'000}
\DeclareMathSymbol{\bsfDelta}{0}{bsfletters}{'001}
\DeclareMathSymbol{\ssfDelta}{0}{ssfletters}{'001}
\DeclareMathSymbol{\bsfTheta}{0}{bsfletters}{'002}
\DeclareMathSymbol{\ssfTheta}{0}{ssfletters}{'002}
\DeclareMathSymbol{\bsfLambda}{0}{bsfletters}{'003}
\DeclareMathSymbol{\ssfLambda}{0}{ssfletters}{'003}
\DeclareMathSymbol{\bsfXi}{0}{bsfletters}{'004}
\DeclareMathSymbol{\ssfXi}{0}{ssfletters}{'004}
\DeclareMathSymbol{\bsfPi}{0}{bsfletters}{'005}
\DeclareMathSymbol{\ssfPi}{0}{ssfletters}{'005}
\DeclareMathSymbol{\bsfSigma}{0}{bsfletters}{'006}
\DeclareMathSymbol{\ssfSigma}{0}{ssfletters}{'006}
\DeclareMathSymbol{\bsfUpsilon}{0}{bsfletters}{'007}
\DeclareMathSymbol{\ssfUpsilon}{0}{ssfletters}{'007}
\DeclareMathSymbol{\bsfPhi}{0}{bsfletters}{'010}
\DeclareMathSymbol{\ssfPhi}{0}{ssfletters}{'010}
\DeclareMathSymbol{\bsfPsi}{0}{bsfletters}{'011}
\DeclareMathSymbol{\ssfPsi}{0}{ssfletters}{'011}
\DeclareMathSymbol{\bsfOmega}{0}{bsfletters}{'012}
\DeclareMathSymbol{\ssfOmega}{0}{ssfletters}{'012}

\newcommand{\trelu}{\overline{\operatorname{ReLU}}}

\DeclareMathOperator*{\argmax}{arg\,max}

\begin{document}

\twocolumn[
\icmltitle{Towards Understanding Why Data Augmentation Improves Generalization}

\begin{icmlauthorlist}
\icmlauthor{Jingyang Li}{nus}
\icmlauthor{Jiachun Pan}{nus}
\icmlauthor{Kim-Chuan Toh}{nus}
\icmlauthor{Pan Zhou}{smu}
\end{icmlauthorlist}

\icmlaffiliation{nus}{National University of Singapore}
\icmlaffiliation{smu}{Singapore Management University}
\icmlcorrespondingauthor{Pan Zhou}{panzhou@smu.edu.sg}

\vskip 0.3in
]

\printAffiliationsAndNotice{}

\begin{abstract}
Data augmentation is a cornerstone technique in deep learning, widely used to improve model generalization.  Traditional methods like random cropping and color jittering, as well as advanced techniques such as CutOut, Mixup, and CutMix, have achieved notable success across various domains. However, the mechanisms by which data augmentation improves generalization remain poorly understood, and existing theoretical analyses typically focus on individual techniques without a unified explanation. In this work, we present a unified theoretical framework that elucidates how data augmentation enhances generalization through two key effects: partial semantic feature removal and feature mixing. Partial semantic feature removal reduces the model’s reliance on individual feature, promoting diverse feature learning and better generalization. Feature mixing, by scaling down original semantic features and introducing noise, increases training complexity, driving the model to develop more robust features.   Advanced methods like CutMix integrate both effects, achieving complementary benefits. Our theoretical insights are further supported by experimental results, validating the effectiveness of this unified perspective.
\end{abstract}

\section{Introduction}
Deep learning has driven transformative progress across a wide range of fields, including computer vision \cite{he2016deep,dosovitskiy2020image}, natural language processing \cite{brown2020language}, and healthcare \cite{lee2020biobert}. Its success is largely underpinned by the ability of models to generalize effectively to unseen data \cite{zhang2021understanding}, a critical property that determines their reliability and impact in real-world applications. Enhancing generalization has thus become a central focus of deep learning research \cite{srivastava2014dropout,volpi2018generalizing,he2022masked}.  

Data augmentation has emerged as a powerful and widely adopted technique to improve generalization \cite{shorten2019survey}. By expanding the diversity of training datasets through modified versions of existing data, it enables models to learn more robust and versatile feature representations \cite{rebuffi2021data}. Traditional augmentation methods, such as geometric transformations (e.g., random cropping) and photometric transformations (e.g., color jittering) \cite{krizhevsky2012imagenet}, have been complemented by advanced techniques, including CutOut \cite{devries2017improved}, Mixup \cite{zhang2017mixup}, CutMix \cite{yun2019cutmix}, and SaliencyMix \cite{uddin2020saliencymix}. All these data augmentations  have shown their effectiveness on  enhancing neural network generalization performance.  

While the empirical benefits of data augmentation are well-documented, the theoretical understanding of how these techniques improve generalization remains incomplete. Existing analyses often focus on specific augmentation types in isolation. For example, geometric transformations have been explored by \citet{shen2022data}, Mixup by \citet{zhang2020does,zou2023benefits,chidambaram2023provably}, and CutOut and CutMix by \citet{oh2024provable}.  However, a unified theoretical framework that explains the shared principles and combined effects of these diverse techniques is still lacking. Developing such a framework is crucial for several reasons: it can uncover the common mechanisms underlying different augmentation strategies, enable systematic comparisons to reveal their relative strengths, and provide practical guidance on combining methods to maximize complementary benefits. 

To address this gap, we propose a unified theoretical framework that examines the effects of \textit{partial semantic feature removal} and \textit{feature mixing}, two fundamental principles underlying many data augmentation techniques, on the generalization of deep neural networks.  Among them, partial semantic feature removal involves discarding specific regions or features from input data, encouraging the model to learn diverse features from the remaining information. For instance, random cropping excludes features outside the cropped region, color jittering modifies color-related information, and CutOut masks a square patch of the image.  Regarding feature mixing, it  incorporates features from one image into another, reducing the prominence of original semantic features while introducing noisy features. This increases training complexity, compelling the model to learn more robust feature representations. 
Mixup blends pixel values from two images via linear interpolation, while CutMix replaces a rectangular region in one image with a patch from another, combining partial semantic feature removal and feature mixing for complementary benefits.

Our framework provides a unified explanation of how these two principles drive generalization improvements across techniques such as random cropping, color jittering, CutOut, Mixup, and CutMix. Specifically, it demonstrates how these methods promote learning of diverse and robust feature representations, enabling models to better handle unseen data.  Our main contributions are as follows:  

\begin{itemize}
	\item We present a theoretical framework that unifies the effects of partial semantic feature removal and feature mixing on generalization, using a 3-layer convolutional neural network (CNN) as a model example.  
	\item We theoretically prove that partial semantic feature removal enhances generalization by promoting diversity in feature learning, feature mixing enhances generalization by fostering robustness, and their combination yields complementary benefits.  
	\item We validate our theoretical findings through extensive experiments on benchmark datasets, offering new insights into the interplay between data augmentation techniques and model properties.  
\end{itemize}

By shedding light on the underlying principles of data augmentation, our work provides a deeper understanding of its role in improving generalization and offers guidance for designing more effective augmentation strategies.

\section{Related Work}
\paragraph{Data Augmentation} Traditional data augmentations have been analyzed from various perspectives: \citet{bishop1995training,dao2019kernel} interpret them as regularization, \citet{hanin2021data} investigate their optimization impact, \citet{chen2020group} view them as a mechanism for invariant learning, and \citet{rajput2019does} study their effect on decision margins. \citet{shen2022data} further examine geometric transformations through feature learning. For advanced methods, \citet{zhang2020does,carratino2022mixup} study Mixup as a regularization technique, while \citet{park2022unified} analyze Mixup and CutMix together. Additionally, \citet{chidambaram2023provably,zou2023benefits} investigate the feature learning dynamics of Mixup, and \citet{oh2024provable} analyze the feature learning effects of CutOut and CutMix.

\paragraph{Feature Learning Analysis}
Prior works on feature learning have provided valuable insights into how neural networks learn and represent data \citep{wen2021toward,wen2022mechanism,allen2022feature,allen-zhu2023towards,jelassi2022towards,huang2023understanding,chen2024does}. For instance, \citet{allen-zhu2023towards} investigated the mechanisms by which ensemble methods and knowledge distillation enhance model generalization.
Building on the analytical framework proposed by \citet{allen-zhu2023towards}, this work extends the scope by offering a unified analysis of data augmentation across traditional and advanced methods.

\section{Problem Setup} \label{sec:problem_setup}
In this section, we begin by outlining the data distribution assumptions adopted in this study for the \(k\)-class classification problem. Next, we present the formulation of supervised learning (SL) with data augmentation on a three-layer CNN. For simplicity, we use \(O\), \(\Omega\), and \(\Theta\) to hide constants related to \(k\), and \(\tilde{O}(\cdot)\), \(\tilde{\Omega}(\cdot)\), and \(\tilde{\Theta}(\cdot)\) to omit polylogarithmic factors. Additionally, we denote \(\operatorname{poly}(k)\) and \(\operatorname{polylog}(k)\) as \(\Theta(k^C)\) and \(\Theta(\log^C k)\), respectively, for some constant \(C > 0\). The notation \([n]\) represents the set \(\{1, 2, \ldots, n\}\).

\subsection{Data Distribution}
Building on the multi-view data assumption proposed by \citet{allen-zhu2023towards}, we assume that each semantic class is characterized by multiple distinct features—such as the noise and tail of an aircraft—that independently enable accurate classification. To validate this hypothesis, we utilize Grad-CAM \citep{selvaraju2017grad} to identify class-specific regions in images. As shown in Fig.~\ref{fig:gcam}, Grad-CAM highlights distinct and non-overlapping regions, such as the noise and tail of an aircraft, that contribute to its recognition. These observations corroborate the findings of \citet{allen-zhu2023towards}, reinforcing the existence of multiple independent discriminative features within each semantic class.

\begin{figure}[t]
\begin{center}
\setlength{\tabcolsep}{0.0pt}  
    \scalebox{1.}{\begin{tabular}{cccccc} 
    \includegraphics[width=0.16\linewidth]{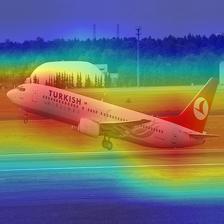}&
    \includegraphics[width=0.16\linewidth]{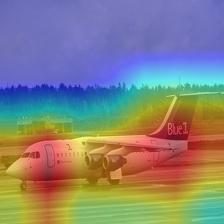}&
    \includegraphics[width=0.16\linewidth]{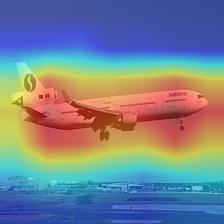}&
    \includegraphics[width=0.16\linewidth]{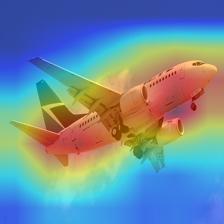}&
    \includegraphics[width=0.16\linewidth]{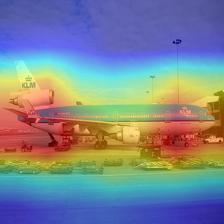}&
    \includegraphics[width=0.16\linewidth]{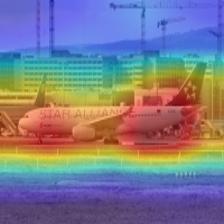}
    \vspace{-0.4em}\\
    \includegraphics[width=0.16\linewidth]{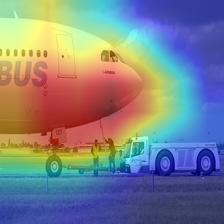}&
    \includegraphics[width=0.16\linewidth]{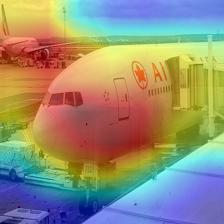}&
    \includegraphics[width=0.16\linewidth]{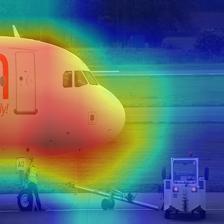}&
    \includegraphics[width=0.16\linewidth]{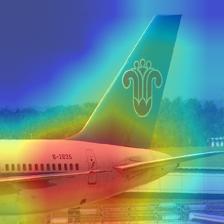}&
    \includegraphics[width=0.16\linewidth]{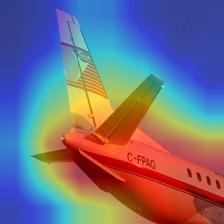}&
    \includegraphics[width=0.16\linewidth]{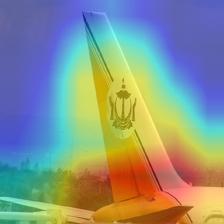}
    \vspace{-0.4em}\\
    \end{tabular}}
\end{center}
\vspace{-0.8em}
\caption{Visualization of pretrained ResNet-50 \citep{he2016deep} using Grad-CAM on airplane images from ImageNet.}
\vspace{-1em}
\label{fig:gcam}
\end{figure}

Under this assumption, each  sample \((X, y)\) consists of an input image \(X\), containing a set of \(P\) patches \(\{x_p \in \mathbb{R}^d\}_{p=1}^P\), and a class label \(y \in [k]\). Each semantic class \(i\) contains two discriminative features, \(v_{i,1}\) and \(v_{i,2} \in \mathbb{R}^d\), each capable of independently ensuring correct classification. While the analysis primarily focuses on two features per class, the approach generalizes to more features. We define \(\mathcal{V}\) as the set of all discriminative features across \(k\) classes:
\[
\setlength{\abovedisplayskip}{3pt}
\setlength{\belowdisplayskip}{3pt}
\setlength{\abovedisplayshortskip}{3pt}
\setlength{\belowdisplayshortskip}{3pt}
\mathcal{V} = \left\{ v_{i, 1}, v_{i, 2} \left| 
\begin{aligned}
& \|v_{i, 1}\|_2 = \|v_{i, 2}\|_2 = 1, \\
& v_{i, l} \perp v_{i^{\prime}, l^{\prime}} \text{ if } (i, l) \neq (i^{\prime}, l^{\prime})
\end{aligned}
\right.\right\}_{i=1}^k,
\]
where the constraints ensure the distinctiveness of features.  

Based on this framework, we define the multi-view and single-view distributions, \(\mathcal{D}_m\) and \(\mathcal{D}_s\), respectively. Samples from \(\mathcal{D}_m\) are characterized by two semantic features, such as the first-row images in Fig. \ref{fig:gcam}. While those from \(\mathcal{D}_s\) are associated with a single semantic feature, such as the second-row images in Fig. \ref{fig:gcam}. Additionally, we define the sparsity parameter \(s = \operatorname{polylog}(k)\) and a constant \(C_p\), both of which play a key role in the dataset structure.

\begin{definition}[Informal, Data Distribution \citep{allen-zhu2023towards}]
\label{def:multi-view}
The data distribution \(\mathcal{D}\) consists of samples drawn from the multi-view data distribution \(\mathcal{D}_m\) with probability \(1 - \mu\), and from the single-view data distribution \(\mathcal{D}_s\) with probability \(\mu = \frac{1}{\polyk}\). A sample \((X, y) \sim \mathcal{D}\) is generated by randomly selecting a label \(y \in [k]\) uniformly, and constructing the corresponding data \(X\) as follows:\\
\textbf{(a)} Sample a set of noisy features \(\mathcal{V}^{\prime}\) uniformly at random from \(\left\{v_{i, 1}, v_{i, 2}\right\}_{i \neq y}\), with each feature included with probability \(s / k\). The complete feature set of \(X\) is then defined as \(\mathcal{V}(X) = \mathcal{V}^{\prime} \cup \{v_{y, 1}, v_{y, 2}\}\), which combines the noisy features \(\mathcal{V}^{\prime}\) with the semantic features \(\{v_{y, 1}, v_{y, 2}\}\). \\
\textbf{(b)} For each feature \(v \in \mathcal{V}(X)\), select \(C_p\) disjoint patches from \([P]\), denoted as \(\mathcal{P}_v(X)\). For each  \(p \in \mathcal{P}_v(X)\), set
\[
x_p = z_p v + \text{``noises"} \in \mathbb{R}^d,
\]
where the coefficients \(z_p \geq 0\) satisfy:
\vspace{-1em}
\begin{itemize}
    \item \textbf{(b1)} For multi-view data \((X, y) \in \mathcal{D}_m\), we have \(\sum_{p \in \mathcal{P}_v(X)} z_p \in [1, O(1)]\) if \(v \in \{v_{y, 1}, v_{y, 2}\}\), and \(\sum_{p \in \mathcal{P}_v(X)} z_p \in [\Omega(1), 0.4]\) if \(v \in \mathcal{V}(X) \setminus \{v_{y, 1}, v_{y, 2}\}\).\vspace{-0.5em}
    \item \textbf{(b2)} For single-view data \((X, y) \in \mathcal{D}_s\), select \(l^* \in [2]\) uniformly at random to determine the semantic feature index. Then, \(\sum_{p \in \mathcal{P}_v(X)} z_p \in [1, O(1)]\) if \(v = v_{y, l^*}\), \(\sum_{p \in \mathcal{P}_v(X)} z_p \in [\rho, O(\rho)]\) (\(\rho = k^{-0.1}\)) if \(v = v_{y, 3-l^*}\), and \(\sum_{p \in \mathcal{P}_v(X)} z_p = \frac{1}{\polylogk}\) if \(v \in \mathcal{V}(X) \setminus \{v_{y, 1}, v_{y, 2}\}\).
\end{itemize}
\vspace{-1em}
\textbf{(c)} For purely noisy patches \(p \in [P] \setminus \bigcup_{v \in \mathcal{V}(X)} \mathcal{P}_v(X)\), set \(x_p = \text{``noises"}\).
\end{definition}

By definition, a multi-view sample \((X, y) \in \mathcal{D}_m\) includes patches that contain both semantic features \(v_{y, 1}\) and \(v_{y, 2}\), along with additional noisy features. Conversely, a single-view sample \((X, y) \in \mathcal{D}_s\) contains patches with only one semantic feature, either \(v_{y, 1}\) or \(v_{y, 2}\), combined with noise.

For a neural network capable of learning diverse semantic features during training, it can correctly classify both multi-view and single-view test samples. In contrast, if the network learns only partial semantic features, it will misclassify single-view samples lacking the learned features, leading to inferior generalization performance. Furthermore, a network that learns robust feature representations can correctly classify noisy images with a larger ``noises'' scale. Such robustness enhances the model's ability to generalize to unseen or noisy data, thereby improving generalization.

\subsection{Supervised Learning with Data Augmentation}
In this work, we focus on supervised learning (SL) with data augmentation for CNNs on a $k$-class classification problem.

\paragraph{Neural Network}
We consider a three-layer CNN with \(mk\) convolutional kernels \(\{w_{i,r}\}_{i \in [k], r \in [m]}\). The classification probability \(\logit_i(F, X)\) for class \(i \in [k]\) is defined as:
\begin{equation}
\setlength{\abovedisplayskip}{3pt}
\setlength{\belowdisplayskip}{3pt}
\setlength{\abovedisplayshortskip}{3pt}
\setlength{\belowdisplayshortskip}{3pt}
\logit_i(F, X) = \frac{\exp(F_i(X))}{\sum\nolimits_{j \in [k]} \exp(F_j(X))},
\end{equation}
where \(F(X) = (F_1(X), \cdots, F_k(X)) \in \mathbb{R}^d\) is defined by:
\begin{equation}
\setlength{\abovedisplayskip}{3pt}
\setlength{\belowdisplayskip}{3pt}
\setlength{\abovedisplayshortskip}{3pt}
\setlength{\belowdisplayshortskip}{3pt}
F_i(X) = \sum\nolimits_{r \in [m]} \sum\nolimits_{p \in [P]} \trelu(\langle w_{i,r}, x_p \rangle).
\end{equation}
Here, \(\trelu\)~\citep{allen-zhu2023towards} is a smoothed ReLU function that outputs zero for negative values, suppresses small positive values to reduce noise, and maintains a linear relationship for larger inputs. This design ensures that \(\trelu\) focuses on significant features while filtering out noise. Its detailed formulation is provided in Appendix~\ref{appsec:data_assum}.

This three-layer network incorporates fundamental components of neural networks: linear mapping, activation, and a softmax layer. Its analysis provides valuable insights into SL with data augmentation. Notably, several theoretical studies have used shallow networks, such as two-layer models, to gain understanding of deeper networks and the impact of data augmentation~\citep{li2017convergence, arora2019fine, zhang2021understanding, shen2022data, zou2023benefits, chidambaram2023provably, oh2024provable}. Moreover, this architecture matches the network structure analyzed in \citet{allen-zhu2023towards}, enabling direct comparisons between our results on SL with data augmentation and the vanilla SL results reported by \citet{allen-zhu2023towards}.

\paragraph{Supervised Learning} For SL on the \(k\)-class classification task, we adopt the classic cross-entropy loss. Given the training dataset \(\mathcal{Z}\), the loss  at iteration \(t\) is defined as:
\begin{equation} \label{eq:sl_loss}
\setlength{\abovedisplayskip}{3pt}
\setlength{\belowdisplayskip}{3pt}
\setlength{\abovedisplayshortskip}{3pt}
\setlength{\belowdisplayshortskip}{3pt}
L^{(t)} = \mathbb{E}_{(X, y) \sim \mathcal{Z}} \big[ - \log \logit_y(F^{(t)}, X) \big],
\end{equation}
where \(\logit_y(F^{(t)}, X)\) denotes the predicted probability for the true label \(y\) at iteration \(t\). Then we update the model parameters \(\{w_{i,r}\}_{i \in [k], r \in [m]}\)  via gradient descent (GD):
\begin{equation}\label{eq:gd}
\setlength{\abovedisplayskip}{3pt}
\setlength{\belowdisplayskip}{3pt}
\setlength{\abovedisplayshortskip}{3pt}
\setlength{\belowdisplayshortskip}{3pt}
w_{i,r}^{(t+1)} = w_{i,r}^{(t)} - \eta \nabla_{w_{i,r}} L^{(t)},
\end{equation} 
where \(\eta \geq 0\) is the learning rate. This iterative update process adjusts the convolutional kernels \(\{w_{i,r}\}\) to minimize the training loss \(L^{(t)}\). Through this optimization, the network progressively aligns the parameters \(\{w_{i,r}\}\) with the underlying data distribution, enabling the model to effectively capture class-specific semantic features.

\begin{figure*}[t]
\begin{center}
    \setlength{\tabcolsep}{0.0pt} 
    \scalebox{1.0}{\begin{tabular}{cccccccccccc}
        \includegraphics[width=0.095\linewidth]{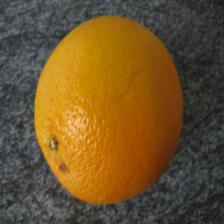} &
        \includegraphics[width=0.095\linewidth]{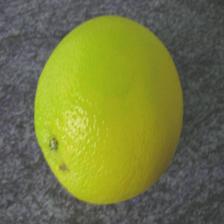} &
        \includegraphics[width=0.095\linewidth]{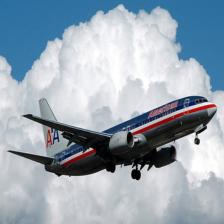} &
        \includegraphics[width=0.095\linewidth]{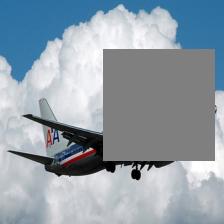} & \hspace{1em} &
        \includegraphics[width=0.095\linewidth]{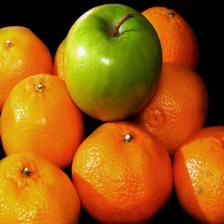} &
        \includegraphics[width=0.095\linewidth]{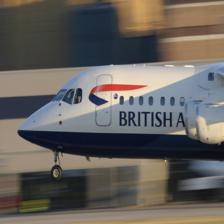} &
        \includegraphics[width=0.095\linewidth]{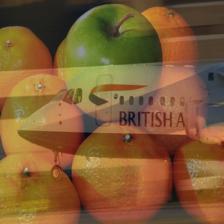} & \hspace{1em} &
        \includegraphics[width=0.095\linewidth]{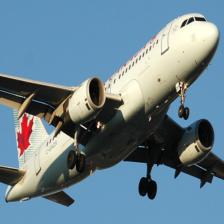} &
        \includegraphics[width=0.095\linewidth]{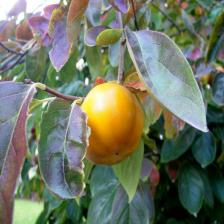} &
        \includegraphics[width=0.095\linewidth]{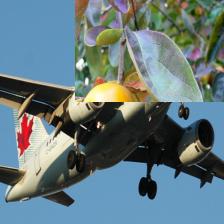} \\

        \multicolumn{4}{c}{\footnotesize{(a) Partial semantic feature removal}} & &
        \multicolumn{3}{c}{\footnotesize{(b) Feature mixing}} & &
        \multicolumn{3}{c}{\footnotesize{(c) Combined effect}}
    \end{tabular}}
\end{center}
\vspace{-1.0em}
\caption{Visualization of data augmentation effects on ImageNet images. (a) Partial semantic feature removal (\(\calA_1\)): original image (left) and augmented version (right) are shown for each pair. (b) Feature mixing (\(\calA_2\)): original images (left and middle) blended into the augmented version (right). (c) Combined effect (\(\calA_3\)): original images (left and middle) together create the augmented version (right).}
\label{fig:effect}
\vspace{-0.6em}
\end{figure*}

\paragraph{Data Augmentation} Adopting data augmentation \(\calA\) to training samples, the supervised training loss becomes
\begin{equation} \label{eq:aug_sl_loss}
\setlength{\abovedisplayskip}{3pt}
\setlength{\belowdisplayskip}{3pt}
\setlength{\abovedisplayshortskip}{3pt}
\setlength{\belowdisplayshortskip}{3pt}
L_\calA^{(t)} \!=\! \bbE_{(X,y)\sim \calZ} [- \log \logit_y(F^{(t)},\calA(X))].
\end{equation}
In this work, we focus on two key effects of data augmentation that underpin our theoretical analysis: \textit{partial semantic feature removal}, denoted as \(\mathcal{A}_1\), and \textit{feature mixing}, denoted as \(\mathcal{A}_2\). These effects represent fundamental principles for the success of  many augmentations.  

\noindent\textbf{Partial Semantic Feature Removal (\(\mathcal{A}_1\))}   
Data augmentations like random cropping, color jittering \cite{krizhevsky2012imagenet}, and CutOut \cite{devries2017improved}, probabilistically remove portions of semantic information from input images. For example, color jittering modifies an image's color properties, effectively eliminating color-related semantic features, while CutOut masks specific regions of the image and thus removes partial semantics. Fig.~\ref{fig:effect}(a) highlights the partial semantic feature removal effect.

\noindent\textbf{Feature Mixing (\(\mathcal{A}_2\))}
Augmentations like Mixup \cite{zhang2017mixup} integrate semantic features from multiple images, introducing noisy features while reducing the dominance of original features. Specifically, Mixup linearly interpolates two input images to create a mixed sample with blended features. This feature mixing effect is shown in Fig.~\ref{fig:effect}(b).

\noindent\textbf{Combined Effects (\(\mathcal{A}_3\))}  
Augmentations like CutMix~\cite{yun2019cutmix} combine the effects of \(\mathcal{A}_1\) and \(\mathcal{A}_2\), offering complementary benefits. Specifically, CutMix replaces a rectangular region of one image with a patch from another, simultaneously removing partial semantic features from the original image and introducing noisy features from the second. This dual effect is shown in Fig.~\ref{fig:effect}(c). Similarly, combining independent methods like Mixup and CutOut can also achieve the combined effects of \(\mathcal{A}_3\).

By categorizing data augmentations through these effects, we establish a unified perspective that lays the groundwork for analyzing their contributions to model generalization.  

\section{Main Results} \label{sec:main_results}
In this section, we formally define the effects of data augmentations on input images and present our main theoretical results. In Sec.~\ref{sec:a1}, we establish that SL with data augmentation \(\mathcal{A}_1\) promotes the learning of diverse features. In Sec.~\ref{sec:a2}, we show that SL with data augmentation \(\mathcal{A}_2\) fosters more robust feature representations. In Sec.~\ref{sec:a3}, we analyze the complementary benefits of SL with data augmentation \(\mathcal{A}_3\), which integrates the effects of \(\mathcal{A}_1\) and \(\mathcal{A}_2\) to achieve both diversity and robustness in feature learning.

\subsection{Diverse Feature Learning with \(\calA_1\)} \label{sec:a1}
Here, we analyze the feature learning process of SL with data augmentation \(\mathcal{A}_1\), and compare its  generalization performance with vanilla SL without  using  \(\mathcal{A}_1\). 

\noindent{\textbf{Assumptions.}} We first formally define the  partial semantic feature removal effect of \(\mathcal{A}_1\) on  feature patches in Definition \ref{def:multi-view}(b). These feature patches are expressed as \(x_p = z_p v + \text{``noises"} \in \mathbb{R}^d\), where \(z_p\) quantifies the semantic content, \(v\) denotes the semantic feature, and the term \(\text{``noises"}\) captures the perturbations inherent to the patch.

\begin{assumption} \label{assum:a1}
For an image \((X, y) \in \mathcal{Z}\), suppose that with a probability \(\pi_1\), the data augmentation \(\mathcal{A}_1\) can   remove one semantic feature, and  the augmented image satisfies: \\
\textbf{(a)} If \((X, y) \in \mathcal{Z}_s\), then \( \sum_{p \in \mathcal{P}_{v}(X)} z_p \in [C_1, O(1)] \) for \(v = v_{y, l^*}\), where \(v_{y, l^*}\) is the semantic feature  in \(X\). \\
\textbf{(b)} If \((X, y) \in \mathcal{Z}_m\), then \( \sum_{p \in \mathcal{P}_{v}(X)} z_p \in [C_1, O(1)] \) for \(v = v_{y, l}\) (\(l \in [2]\) is randomly chosen), while it holds that \(\sum_{p \in \mathcal{P}_{v}(X)} z_p \in [1, O(1)]\) for \(v = v_{y, 3-l}\). \\
Here, \(C_1 \in (0, 0.4)\) is a constant representing the scale of the semantic feature after applying data augmentation \(\mathcal{A}_1\).
\end{assumption}

\cref{assum:a1} captures the partial semantic feature removal effect of the data augmentation \(\mathcal{A}_1\), which has a probability \(\pi_1\) to significantly reduce the scale of one semantic feature in vanilla training sample \(X\). This reduction is evident from \cref{def:multi-view}, where the original scale of all semantic features satisfies \(\sum_{p \in \mathcal{P}_{v}(X)} z_p \in [1, O(1)]\). For the ``noises" part, its reduction or preservation does not impact our proof, so we does not require assumption on it.

Then, we outline the necessary assumptions regarding the dataset and model initialization.

\begin{assumption}\label{assum1}
    \textbf{(a)} The training dataset \(\mathcal{Z}\) follows the distribution \(\mathcal{D}\), and its size satisfies \(N := |\mathcal{Z}| = \polyk\). \\
    \textbf{(b)} Each convolution kernel \(w_{i,r}^{(0)}\) (\(i \in [k], r \in [m]\)) is initialized using a Gaussian distribution \(\mathcal{N}(0, \sigma_0^2 \mathbf{I})\), where \(\sigma_0^{q-2} = 1/k\) and \(q \geq 3\) is given in the definition of \(\trelu\).
\end{assumption}

Assumption~\ref{assum1}(a) assumes significantly more training samples than the number of classes, a standard consideration in SL analysis~\cite{zhang2021understanding, cao2022benign, allen-zhu2023towards}. The Gaussian initialization in Assumption~\ref{assum1}(b) is mild and consistent with standard practices.

\noindent{\textbf{Results.}} To monitor the feature learning process, we define 
\begin{equation}\label{afafas}
    \setlength{\abovedisplayskip}{3pt}
    \setlength{\belowdisplayskip}{3pt}
    \setlength{\abovedisplayshortskip}{3pt}
    \setlength{\belowdisplayshortskip}{3pt}
    \Phi_{i,l}^{(t)}: = \sum\nolimits_{r\in[m]} [\langle w_{i,r}^{(t)}, v_{i,l}\rangle]^{+}
\end{equation}
as a feature learning indicator of feature $v_{i,l}$ in class $i$ (\(i \in [k], l \in [2]\)). This metric represents the total positive correlation score between the $l$-th feature $v_{i,l}$ of class $i$ and all $m$ convolution kernels $w_{i,r}$ \((r \in [m])\) at the $t$-th iteration. A larger $\Phi_{i,l}^{(t)}$ indicates that the network better captures the feature $v_{i,l}$, thereby utilizing it more effectively for classification. Based on this, we  compare the feature learning process and generalization performance of SL with \(\calA_1\) and vanilla SL without  \(\calA_1\). See  its proof in \cref{appsec:proof_a1}.

\begin{theorem} \label{main:thm_a1}
Suppose Assumptions~\ref{assum:a1} and \ref{assum1} holds. For sufficiently large $k$ and \(m = \polylogk\), by setting \(\eta \leq 1 / \polyk\), after independently  running SL with  \(\calA_1\) and  vanilla SL without  \(\calA_1\) for both \(T = \polyk / \eta\) iterations, the following holds with probability at least \(1 - e^{-\Omega(\log^2k)}\):\\
\textbf{(a) Diverse Feature Learning:} Denote \(\Phi_{i}^{(T)}\) learned by SL with \(\calA_1\) and vanilla SL without \(\calA_1\)  as \(\Phi_{i}^{\calA_1}\) and \(\Phi_{i}^{\text{SL}}\) respectively. For every \(i \in [k]\),
under SL with \(\calA_1\), we have:
\begin{equation} \label{eq:a1_feat}
\setlength{\abovedisplayskip}{3pt}
\setlength{\belowdisplayskip}{3pt}
\setlength{\abovedisplayshortskip}{3pt}
\setlength{\belowdisplayshortskip}{3pt}
\Phi_{i,l}^{\calA_1} \geq \Omega(\log k), \quad (\forall l \in [2]).
\end{equation}
But for vanilla SL without  \(\calA_1\) , there is an \(l \in [2]\) such that:
\begin{equation} \label{eq:a1_vsl_feat}
\setlength{\abovedisplayskip}{3pt}
\setlength{\belowdisplayskip}{3pt}
\setlength{\abovedisplayshortskip}{3pt}
\setlength{\belowdisplayshortskip}{3pt}
\Phi_{i,l}^{\text{SL}} \geq \Omega(\log k), \quad \Phi_{i,3-l}^{\text{SL}} \leq 1 / \polylogk.
\end{equation}
\textbf{(b) Improved Generalization:} Denote \(F_i^{(T)}\) learned by SL with \(\calA_1\) and vanilla SL as \(F_i^{\calA_1}\) and \(F_i^{\text{SL}}\) for \(i\in[k]\) respectively. For a multi-view test sample \((X,y) \sim \calD_m\), both SL with \(\calA_1\) and vanilla SL achieve correct predictions. For a single-view test sample \((X,y) \sim \calD_s\), SL with \(\calA_1\) satisfies:
\begin{equation}\label{eq:a1_test}
\setlength{\abovedisplayskip}{3pt}
\setlength{\belowdisplayskip}{3pt}
\setlength{\abovedisplayshortskip}{3pt}
\setlength{\belowdisplayshortskip}{3pt}
F_y^{\calA_1}(X) \geq \max\nolimits_{j\neq y} F_j^{\calA_1}(X) + \Omega(\log k).
\end{equation}
Conversely, for vanilla SL, with probability near 50\%:
\begin{equation}\label{eq:a1_vsl_test}
\setlength{\abovedisplayskip}{3pt}
\setlength{\belowdisplayskip}{3pt}
\setlength{\abovedisplayshortskip}{3pt}
\setlength{\belowdisplayshortskip}{3pt}
\max\nolimits_{j\neq y} F_j^{\text{SL}}(X) \geq F_y^{\text{SL}}(X) + \frac{1}{\polylogk}.
\end{equation}
\end{theorem}

Theorem~\ref{main:thm_a1}(a) demonstrates that after $T=\polyk/\eta$ training iterations, SL with data augmentation \(\calA_1\) enables the network $F^{(T)}$ to effectively capture diverse features. Specifically, for each class $i$, the network learns both semantic features \(v_{i,1}\) and \(v_{i,2}\), as indicated by Eq.~\eqref{eq:a1_feat}. This diverse feature learning ensures the network is not overly reliant on a single feature, promoting robustness across varied data distributions. In contrast,  vanilla SL only learns one of the two semantic features per class, as evidenced  in Eq.~\eqref{eq:a1_vsl_feat},  leading to an incomplete feature learning.

Theorem~\ref{main:thm_a1}(b) highlights the critical role of this diverse feature learning in achieving superior generalization performance. By capturing both semantic features, SL with \(\calA_1\) ensures nearly 100\% test accuracy on both multi-view test samples \((X, y) \sim \calD_m\) and single-view test samples \((X, y) \sim \calD_s\). In contrast, vanilla SL achieves good performance only on multi-view test samples, while its test accuracy on single-view data drops to approximately 50\%. This limitation arises because vanilla SL relies solely on the single learned feature, which is insufficient when only the unlearned feature presents in single-view samples.

The improved generalization performance of SL with \(\calA_1\) is directly attributed to its ability to learn diverse features during training. The augmentation strategy \(\calA_1\) facilitates this through its partial semantic feature removal effect, which transforms a subset of multi-view training samples into ``single-view data" that predominantly contain only one semantic feature.  These ``single-view data" compel the neural network to learn comprehensive semantic features for each class, as relying on a single feature is no longer sufficient to optimize the training objective. This process effectively avoids the ``lottery winning" phenomenon observed in vanilla SL \cite{allen-zhu2023towards}, where the network learns only one semantic feature per class due to random model initialization. See more discussion in \cref{appsec:sl}.

\subsection{Robust Feature Learning with \(\calA_2\)} \label{sec:a2}
Here, we analyze the feature learning process in SL using  augmentation \(\mathcal{A}_2\), and justify its improved generalization performance compared to vanilla SL without \(\mathcal{A}_2\). 

\noindent{\textbf{Assumptions.}}  We first formalize the  feature mixing effect of data augmentation \(\mathcal{A}_2\) on    input images.

\begin{assumption} \label{assum:a2}
For a given image \((X, y) \in \mathcal{Z}\), suppose that with probability \(\pi_2\),   augmentation \(\mathcal{A}_2\) mixes the feature patches from another  sample \(X'\) into \(X\), and  the augmented image satisfies:\\
\textbf{(a)} If \((X, y) \in \mathcal{Z}_s\), then \(\sum_{p \in \mathcal{P}_{v}(X)} z_p \in [1 - C_2, O(1)]\) for \(v = v_{y, l^*}\), where \(v_{y, l^*}\) is the semantic feature  in \(X\). Additionally, \(\sum_{p \in \mathcal{P}_{v}(X)} z_p \in [\tilde{\Omega}(1), C_3]\) for \(v \in \mathcal{V}(X) \setminus \{v_{y,1}, v_{y,2}\}\). \\
\textbf{(b)} If \((X, y) \in \mathcal{Z}_m\), then \(\sum_{p \in \mathcal{P}_{v}(X)} z_p \in [1 - C_2, O(1)]\) for \(v \in \{v_{y,1}, v_{y,2}\}\). Additionally, \(\sum_{p \in \mathcal{P}_{v}(X)} z_p \in [\Omega(1), 0.4 + C_3]\) for \(v \in \mathcal{V}(X) \setminus \{v_{y,1}, v_{y,2}\}\). \\
Here, \(C_2\) and \(C_3\) are two positive constants satisfying \(C_2 + C_3 < 0.6\), where \(C_2\) represents the scale reduction of semantic features, and \(C_3\) is the  scale increase of noisy features caused by  augmentation \(\mathcal{A}_2\).
\end{assumption}

\cref{assum:a2} describes how the feature mixing effect of data augmentation \(\mathcal{A}_2\) influences the semantic and noisy feature distributions in augmented images. Specifically, feature mixing reduces the scale of original semantic features (\(C_2\)) while increasing the scale of noisy features (\(C_3\)). We require \(C_2 + C_3 < 0.6\) to ensure that semantic information from the original sample remains larger than noisy features.

Below, we use Mixup~\cite{zhang2017mixup} as an example of \(\calA_2\).  In Mixup, two training samples \((X, y)\) and \((X', y')\) are blended into an augmented sample \((X_{\text{mix}}, y_{\text{mix}})\):
\begin{equation} \label{eq:mixup}
X_{\text{mix}} = \lambda X + (1 - \lambda) X', \quad y_{\text{mix}} = \lambda y + (1 - \lambda) y',
\end{equation}
where \(\lambda \sim \text{Beta}(\alpha, \alpha)\) is drawn from a Beta distribution with \(\alpha > 0\). Then  we analyze the   Mixup training loss in Lemma~\ref{lem:mixup_loss} whose proof is in \cref{appsec:aug_loss}.

\begin{lemma} \label{lem:mixup_loss}
For Mixup, its training loss 
\begin{equation*} \label{eq:sl_lossaasdsfds}
	\setlength{\abovedisplayskip}{3pt}
	\setlength{\belowdisplayskip}{3pt}
	\setlength{\abovedisplayshortskip}{3pt}
	\setlength{\belowdisplayshortskip}{3pt}
	L_{\text{mix}}^{(t)} = \mathbb{E}_{(X_{\text{mix}}, y_{\text{mix}}) \sim \mathcal{Z}} \big[ - \log \logit_{y_{\text{mix}}}(F^{(t)}, X_{\text{mix}}) \big],
\end{equation*}
is equivalent to 
\begin{equation*} \label{eq:sl_lossasfds}
	\setlength{\abovedisplayskip}{3pt}
	\setlength{\belowdisplayskip}{3pt}
	\setlength{\abovedisplayshortskip}{3pt}
	\setlength{\belowdisplayshortskip}{3pt}
	L^{(t)} = \mathbb{E}_{(X'_{\text{mix}}, y) \sim \mathcal{Z}} \big[ - \log \logit_{y}(F^{(t)}, X'_{\text{mix}}) \big],
\end{equation*}
where  \((X, y) \in \mathcal{Z}\) and \((X_{\text{mix}}, y_{\text{mix}})\) are defined in~\eqref{eq:mixup}, and  $X'_{\text{mix}}$ is augmented by~\eqref{eq:mixup} but  \(\lambda \sim \text{Beta}(\alpha+1, \alpha)\).
\end{lemma}

{In this process, Mixup introduces feature patches from \(X'\) into \(X\), where \(X'\) with probability \(1 - \frac{1}{k}\) belongs to a class different from \(y\). As stated in Eq. \eqref{eq:mixup} and \cref{def:multi-view}(b), this blending reduces the scale of semantic features from \(X\), consistent with the scale reduction quantified by \(C_2\) in \cref{assum:a2}. At the same time, the mixed sample \(X'_{\text{mix}}\) gains additional noisy features from \(X'\), aligning with the increase in the noisy feature scale described by \(C_3\).}

{\noindent{\textbf{Results.}} Let \(\Phi_i^{(t)} := \sum_{l \in [2]} \Phi_{i,l}^{(t)}\), as defined in Eq.~\eqref{afafas}, denote the feature learning indicator for class \(i\) (\(i \in [k]\)). Consider a noisy dataset \(\mathcal{D}_{\text{noisy}}\) that follows the definition in \cref{def:multi-view} but differs by having a larger noise scale for purely noisy patches. See \cref{def:noisy_data} in \cref{appsec:data_assum} for details. We compare the performance of SL with \(\mathcal{A}_2\) to vanilla SL, with the full proof provided in \cref{appsec:proof_a2}. }

\begin{theorem} \label{main:thm_a2}
Suppose Assumptions~\ref{assum:a2} and \ref{assum1} holds. For sufficiently large $k$ and \(m = \polylogk\), by setting \(\eta \leq 1 / \polyk\), after independently running SL with  augmentation \(\calA_2\)  and vanilla L without \(\calA_2\)  for both \(T = \polyk / \eta\) iterations,  with probability at least \(1 - e^{-\Omega(\log^2k)}\), it holds \\
\textbf{(a) Robust Feature Learning:} For every \(i \in [k]\), denote \(\Phi_{i}^{(T)}\) learned by SL with \(\calA_2\) as \(\Phi_{i}^{\calA_2}\), we have:
\begin{equation} \label{eq:a2_feat}
\setlength{\abovedisplayskip}{3pt}
\setlength{\belowdisplayskip}{3pt}
\setlength{\abovedisplayshortskip}{3pt}
\setlength{\belowdisplayshortskip}{3pt}
\Phi_{i}^{\calA_2} \ge \frac{0.6}{0.6-C_2-C_3} \Phi_{i}^{\text{SL}}.
\end{equation}
\textbf{(b) Improved Generalization:} For \((X,y) \sim \calD\), SL with \(\calA_2\) has the same test accuracy as vanilla SL without  \(\calA_2\), as described in \cref{main:thm_a1}(b). For \((X,y) \sim \calD_{\text{noisy}}\), we always have the network trained by SL with \(\calA_2\) has higher probability than vanilla SL to make correct prediction. Further, if \(0.6-C_2-C_3 \le O(\frac{1}{\polylogk})\), it holds:
\begin{equation}
\setlength{\abovedisplayskip}{3pt}
\setlength{\belowdisplayskip}{3pt}
\setlength{\abovedisplayshortskip}{3pt}
\setlength{\belowdisplayshortskip}{3pt}
 F^{\calA_2}_y(X)\geq \max\nolimits_{j\neq y}F^{\calA_2}_j(X)+\Omega(\log k),
\end{equation}
while for vanilla SL, with probability \(1-o(1)\):
\begin{equation}
\setlength{\abovedisplayskip}{3pt}
\setlength{\belowdisplayskip}{3pt}
\setlength{\abovedisplayshortskip}{3pt}
\setlength{\belowdisplayshortskip}{3pt}
\max\nolimits_{j\neq y}F^{\text{SL}}_j(X) \geq F^{\text{SL}}_y(X)+\Omega(\log k).
\end{equation}
\end{theorem}
Theorem~\ref{main:thm_a2}(a) highlights the robust feature learning enabled by SL with \(\calA_2\). Eq.~\eqref{eq:a2_feat} shows that the feature learning indicator \(\Phi_i^{\calA_2}\) for SL with \(\calA_2\) is strictly larger than \(\Phi_i^{\text{SL}}\) for vanilla SL without  \(\calA_2\). This result indicates that \(\calA_2\) enhances the network's capacity to capture class-specific features, leading to more robust feature representations. Such robustness is crucial for achieving superior generalization, particularly in scenarios involving noisy data.

Theorem~\ref{main:thm_a2}(b) further establishes the link between robust feature learning and improved generalization performance. For clean test samples \((X, y) \sim \calD\), SL with \(\calA_2\) achieves test accuracy comparable to vanilla SL. However, for noisy test samples \((X, y) \sim \calD_{\text{noisy}}\), SL with \(\calA_2\) consistently outperforms vanilla SL, with a significant performance gap when the feature mixing of \(\calA_2\) results in the scales of semantic and noisy features becoming similar. This improvement arises from the enhanced robustness of the features learned by SL with \(\calA_2\), which enables the network to better distinguish meaningful semantic features from noise.

The superior ability of SL with \(\calA_2\) to learn robust features arises from its feature mixing effect. Specifically, \(\calA_2\) reduces the scale of the original semantic features while simultaneously increasing the scale of the noisy features. This transformation creates a more challenging training environment, where minimizing the training loss necessitates the learning of robust feature representations that can adapt to noisy inputs. Consequently, the network trained with \(\calA_2\) becomes better equipped to generalize effectively to unseen and noisy test samples, underscoring the efficacy of \(\calA_2\) in improving generalization performance.

\subsection{Diverse \& Robust Feature Learning via \(\calA_3\)} \label{sec:a3}
Here, we analyze the feature learning process in SL with data augmentation \(\mathcal{A}_3\) which contains the augmentation effects of \(\mathcal{A}_1\) and \(\mathcal{A}_2\), and justify  its superior  generalization performance compared to vanilla SL without using \(\mathcal{A}_3\).   

\noindent{\textbf{Assumptions.}}  We  formalize the effect assumption of data augmentation \(\mathcal{A}_3\) which combines the effects of \(\mathcal{A}_1\) and \(\mathcal{A}_2\), namely, incorporating both partial semantic feature removal and feature mixing.

\begin{assumption} \label{assum:a3}
For a given image \((X, y) \in \mathcal{Z}\), suppose that with probability \(\pi_3\), augmentation \(\mathcal{A}_3\) mixes feature patches from a training sample \(X'\) into \(X\) while partially removing one semantic feature in \(X\). As a result, the augmented image exhibits a combined feature scale property from \cref{assum:a1} and \cref{assum:a2}.
\end{assumption}

\cref{assum:a3} captures how the combined effects of partial semantic feature removal and feature mixing from \(\mathcal{A}_3\) affect the feature distribution of training samples. See more formal statements in  \cref{appassum:a3} of \cref{appsec:aug_loss}.

\noindent{\textbf{Results.}}  Now we   state our main results: the diverse and robust feature learning of SL with \(\calA_3\), and  its superior generalization performance than vanilla SL without \(\calA_3\).

\begin{theorem} \label{main:thm_a3}
Suppose Assumptions~\ref{assum:a3} and \ref{assum1} holds. For sufficiently large $k$ and \(m = \polylogk\), by setting \(\eta \leq 1 / \polyk\), after independently running SL with  augmentation \(\calA_3\) and  vanilla SL without \(\calA_3\) for both \(T = \polyk / \eta\) iterations,  with probability at least \(1 - e^{-\Omega(\log^2k)}\), it holds:\\
\textbf{(a) Diverse and Robust Feature Learning:} For every \(i \in [k]\), denote \(\Phi_{i}^{(T)}\) learned by SL with \(\calA_3\) and vanilla SL without \(\calA_3\) as \(\Phi_{i}^{\calA_3}\) and  $\Phi_{i}^{\text{SL}}$ respectively, we have:
\begin{align}
\setlength{\abovedisplayskip}{3pt}
\setlength{\belowdisplayskip}{3pt}
\setlength{\abovedisplayshortskip}{3pt}
\setlength{\belowdisplayshortskip}{3pt}
\Phi_{i,l}^{\calA_3} \geq & \Omega(\log k), \quad (\forall l \in [2]), \label{eq:a3_feat_d}\\
\Phi_{i}^{\calA_3} \ge & \frac{0.6}{0.1+C_1/2-C_2-C_3} \Phi_{i}^{\text{SL}}.\label{eq:a3_feat_r}
\end{align}
\textbf{(b) Improved Generalization:} For both multi-view and single-view test samples \((X,y) \sim \calD\), SL with \(\calA_3\) satisfies:
\begin{equation}\label{eq:a3_test}
\setlength{\abovedisplayskip}{3pt}
\setlength{\belowdisplayskip}{3pt}
\setlength{\abovedisplayshortskip}{3pt}
\setlength{\belowdisplayshortskip}{3pt}
F_y^{\calA_3}(X) \geq \max\nolimits_{j\neq y} F_j^{\calA_3}(X) + \Omega(\log k).
\end{equation}
For \((X,y) \sim \calD_{\text{noisy}}\), the network trained by SL with \(\calA_3\) has always higher probability than vanilla SL without  \(\calA_3\) to make correct prediction. Further, if \(0.1+C_1/2-C_2-C_3 \le O(\frac{1}{\polylogk})\),  it  holds:
\begin{equation}
\setlength{\abovedisplayskip}{3pt}
\setlength{\belowdisplayskip}{3pt}
\setlength{\abovedisplayshortskip}{3pt}
\setlength{\belowdisplayshortskip}{3pt}
 F^{\calA_3}_y(X)\geq \max\nolimits_{j\neq y}F^{\calA_3}_j(X)+\Omega(\log k).
\end{equation}
\end{theorem}
{\cref{main:thm_a3}, with its proof in \cref{appsec:proof_a3}, highlights the complementary benefits of SL with \(\mathcal{A}_3\). Specifically, \cref{main:thm_a3}(a) shows that \(\mathcal{A}_3\) facilitates diverse feature learning by enabling the network to capture both semantic features \(v_{i,1}\) and \(v_{i,2}\) for each class \(i\), as shown in Eq.~\eqref{eq:a3_feat_d}. Furthermore, \(\mathcal{A}_3\) enhances robustness in feature representations, as indicated by Eq.~\eqref{eq:a3_feat_r}, where the feature learning indicator \(\Phi_i^{\mathcal{A}_3}\) strictly exceeds \(\Phi_i^{\text{SL}}\). These results underscore the dual effects of \(\mathcal{A}_3\), which simultaneously improves both diversity and robustness in feature learning.}

{\cref{main:thm_a3}(b) establishes that compared with vanilla SL without \(\mathcal{A}_3\),  SL with \(\mathcal{A}_3\) achieves superior generalization performance on both clean data \((X, y) \sim \mathcal{D}\) and noisy data \((X, y) \sim \mathcal{D}_{\text{noisy}}\). As discussed in Sec.~\ref{sec:a1}, the improved generalization on \(\mathcal{D}\) is driven by the diverse feature learning enabled by the partial semantic feature removal effect in \(\mathcal{A}_3\), allowing correct predictions on single-view test samples containing only a single semantic feature. Meanwhile, as indicated in Sec.~\ref{sec:a2}, the enhanced performance on \(\mathcal{D}_{\text{noisy}}\) is attributed to robust feature learning facilitated by the feature mixing effect, enabling better distinction between semantic features and noise. These complementary effects collectively enable SL with \(\mathcal{A}_3\) to achieve consistently better generalization 
across different data distributions. } 

{Moreover, the combination of partial semantic feature removal and feature mixing in \(\mathcal{A}_3\) provides mutual benefits. Compared to \(\mathcal{A}_1\) in \cref{assum:a1}, the upper bound on \(C_1\) in \(\mathcal{A}_3\) is relaxed from \(0.4\) to \(0.4 + C_2 + C_3\), as detailed in \cref{appassum:a3}. This relaxation lowers the required scale of partial semantic feature removal for diverse feature learning, making it more feasible in practice. Furthermore, from Eq.~\eqref{eq:a3_feat_r} and Eq.~\eqref{eq:a2_feat}, it is evident that \(\Phi_i^{\mathcal{A}_3}\) has a larger scale than \(\Phi_i^{\mathcal{A}_2}\) for \(i \in [k]\), provided \(C_1 < 1\). This indicates that combining partial semantic feature removal with feature mixing enhances the learning of robust features beyond what feature mixing alone can achieve.}

{\cref{main:thm_a3} provides theoretical insights into advanced data augmentation strategies, such as CutMix~\cite{yun2019cutmix} and SaliencyMix~\cite{uddin2020saliencymix}. The SL training loss with CutMix, similar to Mixup (\cref{lem:mixup_loss}), can be reformulated with variations only on the input images, as detailed in \cref{appsec:aug_loss}. This demonstrates how CutMix integrates the effects of partial semantic feature removal and feature mixing to support the diverse and robust feature learning. SaliencyMix, as an improved variant of CutMix, further leverages the saliency map of \(X'\) to ensure that the selected region mixed into original image \(X\) contains semantic features rather than purely noisy patches. By \cref{appassum:a3}, this targeted approach increases the noisy feature scale \(C_3\), which, as shown in Eq.~\eqref{eq:a3_feat_r}, enhances the learning of robust feature representations. These results highlight how SaliencyMix amplifies the complementary effects of CutMix, leading to improved generalization performance.}

\vspace{-1em}
{\paragraph{Discussion} Compared with previous works on the feature learning analysis of data augmentation \cite{shen2022data,zou2023benefits,chidambaram2023provably,oh2024provable}, this work differs from two key aspects. (a) We provide an unified theoretical analysis on both traditional data augmentations (e.g., random cropping and color jittering) and advanced techniques (e.g., CutOut, Mixup, CutMix), while previous works focus either on geometric transformations (e.g., \citet{shen2022data} on random flipping) or a limited subset of advanced augmentation methods (e.g., \citet{zou2023benefits,chidambaram2023provably} on Mixup, \citet{oh2024provable} on CutOut and CutMix). (b) This work considers a \(k\)-class classification problem for analysis and provides both feature learning and test accuracy result, while previous works either  focus on binary classification problems \cite{shen2022data,zou2023benefits,oh2024provable} or provide only the feature learning analysis without test accuracy result \cite{chidambaram2023provably}.}

\section{Experiments} \label{sec:experiments}

In this section, we validate our theoretical findings through experiments on widely used image classification benchmarks using CNNs. Specifically, we evaluate the performance of VGG~\cite{simonyan2014very} and DenseNet~\cite{huang2017densely}, two representative CNN architectures, on CIFAR-100~\cite{krizhevsky2009learning} and Tiny-ImageNet~\cite{deng2009imagenet}.

We compare the performance of vanilla SL with SL incorporating various data augmentations, including those capturing partial semantic feature removal (\(\mathcal{A}_1\)), feature mixing (\(\mathcal{A}_2\)), and their combined effects (\(\mathcal{A}_3\)). The augmentations considered include random cropping, color jittering, CutOut~\cite{devries2017improved}, Mixup~\cite{zhang2017mixup}, CutMix~\cite{yun2019cutmix}, and SaliencyMix~\cite{uddin2020saliencymix}, as well as a combined strategy using random cropping and Mixup. The primary goal is to evaluate how the different effects of augmentations align with the theoretical insights and to quantify their impact on model generalization.

\subsection{CIFAR-100 Classification}
We evaluate our theoretical findings on CIFAR-100 \cite{krizhevsky2009learning}, a challenging image classification dataset with 100 classes and 60,000 images (50,000 for training and 10,000 for testing). We train VGG-16 and DenseNet-121 using SL with various data augmentations and compare their performance against vanilla SL.

The results in Table~\ref{tab:cifar100_results} demonstrate that SL with random cropping, color jittering, and CutOut achieves better test accuracy than vanilla SL, attributed to the partial semantic feature removal effect (\(\mathcal{A}_1\)) that promotes diverse feature learning. Similarly, SL with Mixup shows improved generalization due to the feature mixing effect (\(\mathcal{A}_2\)), which enhances robust feature representation learning. Augmentations like CutMix and SaliencyMix achieve the highest generalization performance, as they leverage the combined effects of \(\mathcal{A}_1\) and \(\mathcal{A}_2\) (\(\mathcal{A}_3\)), providing complementary benefits. Notably, the combination of random cropping and Mixup also exhibits combined effects, achieving generalization performance comparable to CutMix and SaliencyMix.

\begin{table}[ht]
\centering
\caption{Test accuracy (\%) of VGG-16 and DenseNet-121 on CIFAR-100 with different data augmentations.}
\label{tab:cifar100_results}
\resizebox{\columnwidth}{!}{%
\begin{tabular}{l|l|cc}
    \toprule
\multicolumn{2}{l|}{\multirow{1}{*}{\textbf{Data Augmentation} }}  & \textbf{VGG-16} & \textbf{DenseNet-121} \\
    \midrule
\multicolumn{2}{l|}{\multirow{1}{*}{Vanilla SL}}    & 63.69 & 73.75 \\
    \midrule
  &  SL + Color Jittering    &  64.20 & 74.34 \\
 $\mathcal{A}_1$ &	SL + CutOut  & 66.86 & 77.60 \\
 & SL + Random Cropping & 69.27 & 77.83 \\
    \midrule
$\mathcal{A}_2$	 &	SL + Mixup & 67.41 & 75.55 \\
    \midrule
 &	SL + CutMix & 71.53 & 79.24 \\
$\mathcal{A}_3$	 &	SL + SaliencyMix  & 71.69 & 79.38 \\
 & SL + Random Cropping + Mixup  & 71.39 & 79.59 \\
    \bottomrule
\end{tabular}
}
\end{table}

\subsection{Tiny-ImageNet Classification}
We further validate our theoretical findings on Tiny-ImageNet~\cite{deng2009imagenet}, a dataset with 200 classes containing 500 training images and 50 validation images per class. Similar to the CIFAR-100 setup, we train VGG-16 and DenseNet-121 with various data augmentations.

Table~\ref{tab:tinyimagenet_results} shows consistent trends with CIFAR-100 results. SL with random cropping, color jittering, and CutOut performs better than vanilla SL due to the partial semantic feature removal effect (\(\mathcal{A}_1\)), which enables diverse feature learning. Mixup demonstrates superior generalization from the feature mixing effect (\(\mathcal{A}_2\)), which improves robust feature representation. CutMix and SaliencyMix achieve the highest accuracy, as their combined effects (\(\mathcal{A}_3\)) leverage complementary benefits. The combination of random cropping and Mixup also shows comparable performance to CutMix and SaliencyMix, highlighting the effectiveness of combining \(\mathcal{A}_1\) and \(\mathcal{A}_2\) to achieve complementary effects.

\begin{table}[ht]
\centering
\caption{Test accuracy (\%) of VGG-16 and DenseNet-121 on Tiny-ImageNet with different data augmentations.}
\label{tab:tinyimagenet_results}
\resizebox{\columnwidth}{!}{%
\begin{tabular}{l|l|cc}
    \toprule
\multicolumn{2}{l|}{\multirow{1}{*}{\textbf{Data Augmentation} }}  & \textbf{VGG-16} & \textbf{DenseNet-121} \\
    \midrule
\multicolumn{2}{l|}{\multirow{1}{*}{Vanilla SL}}  & 50.65 & 62.28 \\
    \midrule
  &SL + Color Jittering & 51.12 & 62.92 \\
 $\mathcal{A}_1$ &SL	+ CutOut  & 54.56 & 64.63 \\
 & SL + Random Cropping  & 56.22 & 65.32 \\
    \midrule
$\mathcal{A}_2$	 &	SL + Mixup  & 54.93 & 64.93 \\
    \midrule
 &	SL + CutMix   & 59.13 & 68.20 \\
$\mathcal{A}_3$	 &	SL + SaliencyMix  & 59.33 & 68.31 \\
 &SL + Random Cropping + Mixup  & 58.92 & 68.05 \\
    \bottomrule
\end{tabular}
}
\end{table}

\section{Conclusion}
We presented a unified theoretical framework to analyze data augmentations, with a focus on partial semantic feature removal (\(\mathcal{A}_1\)), feature mixing (\(\mathcal{A}_2\)), and their combined effects (\(\mathcal{A}_3\)). Our analysis showed that \(\mathcal{A}_1\) improves generalization by encouraging diverse feature learning, \(\mathcal{A}_2\) enhances generalization by fostering robust feature representations, and \(\mathcal{A}_3\) achieves superior generalization by integrating the benefits of both, offering insights on  augmentation designs.   Experiments validated  these theoretical insights. 

\noindent\textbf{Limitations and Future Work.}  
(a) Our analysis is limited to CNN architectures on image classification tasks. Extending the framework to other architectures, such as Vision Transformers, and to tasks like object detection or segmentation remains an open challenge.  
(b) Our experiments are conducted on relatively small-scale datasets, such as CIFAR-100 and Tiny-ImageNet, due to computational constraints. Evaluating the framework on larger datasets and in real-world scenarios would further validate its scalability.

\bibliography{cut}
\bibliographystyle{icml2025}

\newpage
\appendix
\onecolumn

\section{Analysis of Augmented Training Loss} \label{appsec:aug_analysis}
In this section, we formally define the data distribution and relevant assumptions underlying our analysis. We then examine the effects of data augmentation, focusing on partial semantic feature removal and feature mixing, and derive the formulation of the augmented supervised training loss.

\subsection{Data Distribution and Assumption} \label{appsec:data_assum}
To formally define the data distribution, we introduce the global constant \(C_p\), the noisy feature parameter \(s = \operatorname{polylog}(k)\), \(\gamma = \frac{1}{k^{1.5}}\), and the random noise parameter \(\sigma_p = \frac{1}{\sqrt{d} \operatorname{polylog}(k)}\). A noisy feature implies that a sample \(X\) from class \(y\) primarily exhibits the semantic feature \(v_{y,l}\) (with \(l \in [2]\)), while also including minor scaled features \(v_{j,l}\) (with \(j \neq y, l \in [2]\)) from other classes.

\begin{definition}[Data distributions for single-view \(\mathcal{D}_s\) and multi-view data \(\mathcal{D}_m\) ~\cite{allen-zhu2023towards}]
\label{def:data2}
Data distribution \(\mathcal{D}\) consists of samples from multi-view distribution \(\mathcal{D}_m\) with probability \(1-\mu\) and from single-view distribution \(\mathcal{D}_s\) with probability \(\mu\). A sample \((X, y) \sim \mathcal{D}\) is generated by randomly and uniformly selecting a label \(y \in [k]\) and creating data \(X\) as follows:
\begin{itemize}
    \item[1)] Sample a set of features \(\mathcal{V}'\) uniformly at random from \(\{v_{i,1}, v_{i,2}\}_{i \neq y}\), each with probability \(\frac{s}{k}\).
    \item[2)] Define \(\mathcal{V}(X) = \mathcal{V}' \cup \{v_{y,1}, v_{y,2}\}\) as the set of feature vectors used in \(X\).
    \item[3)] For each \(v \in \mathcal{V}(X)\), select \(C_p\) disjoint patches in \([P]\), denoted as \(\mathcal{P}_v(X)\) (the distribution of these patches can be arbitrary). Let \(\mathcal{P}(X) = \cup_{v \in \mathcal{V}(X)} \mathcal{P}_v(X)\).
    \item[4)] If \(\mathcal{D} = \mathcal{D}_s\) is the single-view distribution, select a value \(l^* = l^*(X) \in [2]\) uniformly at random.
    \item[5)] For each \(p \in \mathcal{P}_v(X)\) with \(v \in \mathcal{V}(X)\), given feature noise \(\alpha_{p,v'} \in [0, \gamma]\), set feature patches
    \[
    x_p = z_p v + \sum_{v' \in \mathcal{V}} \alpha_{p,v'} v' + \xi_p,
    \]
    where \(\xi_p \sim \mathcal{N}(0, \sigma_p^2 \mathbf{I})\) is independent Gaussian noise. The coefficients \(z_p \geq 0\) satisfy:
    \begin{itemize}
        \item For multi-view data \((X, y) \in \mathcal{D}_m\), \(\sum_{p \in \mathcal{P}_v(X)} z_p \in [1, O(1)]\) for \(v \in \{v_{y,1}, v_{y,2}\}\), and the marginal distribution of \(\sum_{p \in \mathcal{P}_v(X)} z_p\) is left-close. For \(v \in \mathcal{V}(X) \setminus \{v_{y,1}, v_{y,2}\}\), \(\sum_{p \in \mathcal{P}_v(X)} z_p \in [\Omega(1), 0.4]\), and the marginal distribution is right-close.
        \item For single-view data \((X, y) \in \mathcal{D}_s\), \(\sum_{p \in \mathcal{P}_v(X)} z_p \in [1, O(1)]\) for \(v = v_{y,l^*}\). For \(v = v_{y,3-l^*}\), \(\sum_{p \in \mathcal{P}_v(X)} z_p \in [\rho, O(\rho)]\) with \(\rho = k^{-0.1}\). For \(v \in \mathcal{V}(X) \setminus \{v_{y,1}, v_{y,2}\}\), \(\sum_{p \in \mathcal{P}_v(X)} z_p \in [\Omega(\Gamma), \Gamma]\) with \(\Gamma = \frac{1}{\operatorname{polylog}(k)}\).
    \end{itemize}
    \item[6)] For each \(p \in [P] \setminus \mathcal{P}(X)\), set purely noise patches as
    \[
    x_p = \sum_{v' \in \mathcal{V}} \alpha_{p,v'} v' + \xi_p,
    \]
    where \(\xi_p \sim \mathcal{N}(0, \frac{\gamma^2 k^2}{d} \mathbf{I})\) is independent Gaussian noise, and \(\alpha_{p,v'} \in [0, \gamma]\) are the feature noise coefficients.
\end{itemize}
\end{definition}

Based on the definition of the data distribution \(\mathcal{D}\), we define the training dataset \(\mathcal{Z}\) as follows:
\begin{definition}[Training dataset]
We assume that the data distribution \(\mathcal{D}\) consists of samples from \(\mathcal{D}_m\) with probability \(1-\mu\) and samples from \(\mathcal{D}_s\) with probability \(\mu = \frac{1}{\operatorname{poly}(k)}\). Given \(N\) training samples, the training dataset is denoted as \(\mathcal{Z} = \mathcal{Z}_m \cup \mathcal{Z}_s\), where \(\mathcal{Z}_m\) and \(\mathcal{Z}_s\) represent the multi-view and single-view datasets, respectively. We write \((X, y) \sim \mathcal{Z}\) to denote a sample \((X, y)\) drawn uniformly at random from the empirical dataset. Furthermore, the number of multi-view and single-view samples in the training dataset are denoted as \(N_m = |\mathcal{Z}_m|\) and \(N_s = |\mathcal{Z}_s|\), respectively.
\end{definition}
Below, we formally define the noisy data distribution \(\calD_{\text{noisy}}\):
\begin{definition}[Noisy data distribution \(\calD_{\text{noisy}}\)]
\label{def:noisy_data}
The noisy data distribution \(\calD_{\text{noisy}}\) shares the same distribution for semantic and noisy feature patches as in \cref{def:data2} (\(x_p \text{ for } p \in \calP(X)\)), but introduces a higher noise level for purely noise patches (\(x_p \text{ for } p \in [P]\setminus\calP(X)\)). Specifically, for \((X, y) \sim \calD_{\text{noisy}}\), the purely noise patches \(x_p\) for \(p \in [P]\setminus\calP(X)\) are defined as:
\[
x_p = \sum_{v' \in \mathcal{V}} \alpha_{p,v'} v' + \xi_p,
\]
where \(\alpha_{p,v'} \in [0, \gamma]\) are the feature noise coefficients (of the same scale as in \cref{def:data2}), and \(\xi_p \sim \mathcal{N}(0, \sigma_n \mathbf{I})\) represents Gaussian noise with an increased scale \(\sigma_n = \frac{\polylogk}{d}\).
\end{definition}

Now we introduce the smoothed ReLU function \(\trelu\)~\citep{allen-zhu2023towards} in detail: for an integer \(q \geq 3\) and a threshold \(\varrho = \frac{1}{\operatorname{polylog}(k)}\),
\begin{equation} \label{appeq:trelu}
\trelu(z) =
\begin{cases} 
0, & \text{if } z \leq 0, \\
\frac{z^q}{q \varrho^{q-1}}, & \text{if } z \in [0, \varrho], \\
z - \left(1 - \frac{1}{q}\right) \varrho, & \text{if } z \geq \varrho.
\end{cases}
\end{equation}
This configuration ensures a linear relationship for large \(z\) values, while significantly reducing the impact of low-magnitude noises for small \(z\), thereby enhancing the separation of true features from noise.

Since we follow the proof framework of~\citet{allen-zhu2023towards}, we adopt the same parameter assumptions as in their work, with the exception of an additional probability parameter assumption on \(\pi_1\), \(\pi_2\), and \(\pi_3\), which is introduced by our assumptions on data augmentation \(\mathcal{A}_1\), \(\mathcal{A}_2\), and \(\mathcal{A}_3\), respectively. Below, we provide the detailed parameter assumptions:

\begin{parameter}
\label{assp:para}
    We assume that 
    \begin{itemize}
        \item $q\geq 3$ and $\sigma_0^{q-2}=1/k$, where $\sigma_0$ gives the initialization magnitude. 
        \item $\gamma\leq \tilde{O}(\frac{\sigma_0}{k})$ and $\gamma^q \leq \tilde{\Theta}(\frac{1}{k^{q-1}mP})$, where $\gamma$ controls the feature noise.
        \item The size of single-view labeled training data $N_{s} = \tilde{o}(k/\rho)$ and $N_{s}\leq \frac{k^2}{s} \rho ^{q-1}$.
        \item $N\geq N_{s}\cdot \polyk$, $\sqrt{d}\geq \eta T\cdot \polyk$.
        \item $m = \polylogk$.
        \item \(\pi_1, \pi_2, \pi_3 \ge \frac{1}{\polylogk}\).
    \end{itemize}
\end{parameter}

\subsection{Data Augmentation and Loss} \label{appsec:aug_loss}

\paragraph{Partial Semantic Feature Removal Effect of \(\mathcal{A}_1\)} 
As described in \cref{assum:a1}, the data augmentation \(\mathcal{A}_1\) has a probability \(\pi_1\) of significantly reducing the scale of one semantic feature in the original training sample \(X\). This assumption is supported by the observation, illustrated in Figure~\ref{fig:gcam}, that different semantic features in a multi-view image are spatially distinct. Consequently, the likelihood of data augmentation techniques such as random cropping or CutOut removing both features simultaneously is substantially lower than the likelihood of removing only one. To simplify the theoretical analysis, we therefore assume that \(\mathcal{A}_1\) primarily targets a single semantic feature for removal in each instance.

\paragraph{Loss Function of SL with \(\calA_1\)}
Recall that \(\logit_i(F, X) := \frac{e^{F_i(X)}}{\sum_{j \in [k]} e^{F_j(X)}}\), and denote the cross-entropy loss as \(L(F; X, y) = -\log \logit_y(F, X)\). After applying the data augmentation \(\mathcal{A}_1\) to the training samples, the augmented supervised loss function \(L_{\mathcal{A}_1}\) at training step \(t\) becomes:

\begin{equation} \label{appeq:a1_loss}
\begin{aligned}
    L_{\calA_1}^{(t)}
    =\,& \bbE_{(X,y)\sim\calZ}\left[L(F^{(t)}; \calA_1(X), y)\right] \\ =\,&
    \frac{N-N_s}{N}\, \bbE_{(X,y)\sim\calZ_m}\left[L(F^{(t)}; \calA_1(X), y)\right] + \frac{N_s}{N}\, \bbE_{(X,y)\sim\calZ_s}\left[L(F^{(t)}; \calA_1(X), y)\right]
    \\ =\,& \frac{(N-N_s)(1-\pi_1)}{N}\, \bbE_{(X,y)\sim\calZ_m '}\left[L(F^{(t)}; X, y)\right] + \frac{N_s(1-\pi_1)}{N}\, \bbE_{(X,y)\sim\calZ_s '}\left[L(F^{(t)}; X, y)\right] \\ & + \frac{(N-N_s)\pi_1}{N} \, \bbE_{(X,y)\sim\tcalZ_m}\left[L(F^{(t)}; X, y)\right] + \frac{N_s\pi_1}{N}\, \bbE_{(X,y)\sim\tcalZ_s}\left[L(F^{(t)}; X, y)\right].
\end{aligned}
\end{equation}
Here, \(\mathcal{Z}_m'\) represents the subset of multi-view data in \(\mathcal{Z}_m\) where \(\mathcal{A}_1\) manipulates only the ``noises" without affecting the feature distribution. In contrast, \(\tilde{\mathcal{Z}}_m = \tilde{\mathcal{Z}}_m^1 \cup \tilde{\mathcal{Z}}_m^2\) contains the multi-view data in \(\mathcal{Z}_m\) where \(\mathcal{A}_1\) significantly reduces one of the semantic features, as described in \cref{assum:a1}. Specifically, \(\tilde{\mathcal{Z}}_m^1\) denotes the subset where \(v_{y,2}\) is largely removed, and \(\tilde{\mathcal{Z}}_m^2\) denotes the subset where \(v_{y,1}\) is largely removed.

Similarly, \(\mathcal{Z}_s'\) represents the subset of single-view data in \(\mathcal{Z}_s\) where \(\mathcal{A}_1\) manipulates only the ``noises", and \(\tilde{\mathcal{Z}}_s\) contains the single-view data in \(\mathcal{Z}_s\) where \(\mathcal{A}_1\) significantly reduces the single semantic feature present. The augmented training dataset is then defined as \(\tilde{\mathcal{Z}}_1 = \mathcal{Z}_m' \cup \tilde{\mathcal{Z}}_m \cup \mathcal{Z}_s' \cup \tilde{\mathcal{Z}}_s\).

\paragraph{Feature Mixing Effect of \(\calA_2\)}
As outlined in \cref{assum:a2}, the data augmentation \(\mathcal{A}_2\) has a probability \(\pi_2\) of decreasing the scale of original semantic features while increasing the scale of noisy features. To illustrate this feature mixing effect of \(\mathcal{A}_2\), we use Mixup as an example and provide a proof for \cref{lem:mixup_loss}.

In Mixup, two data samples \((X_i, y_i), (X_j, y_j) \in \mathcal{Z}\) (where \(i \neq j \in [N]\)) are randomly selected to generate an augmented sample \((\tilde{X}_{i,j}(\lambda), \tilde{y}_{i,j}(\lambda))\), defined as:
\begin{equation} \label{appeq:data_mixup}
\begin{aligned}
\tilde{X}_{i,j}(\lambda) &= \lambda X_i + (1 - \lambda) X_j, \\
\tilde{y}_{i,j}(\lambda) &= \lambda y_i + (1 -\lambda) y_j,
\end{aligned}
\end{equation}
where \(\lambda \sim \text{Beta}(\alpha, \alpha)\) is a mixing coefficient sampled from a Beta distribution with \(\alpha > 0\).

Denoting the cross-entropy loss as \(L(F; X, y) = -\log \logit_y(F, X)\), the loss function for SL with Mixup can be formulated as:
\begin{equation}
L_{\text{Mixup}} = \frac{1}{N^2} \sum_{i,j=1}^{N} \mathbb{E}_{\lambda \sim \Beta(\alpha, \alpha) }\mathbb{E}_{M(\lambda)} \, [L(F; \tilde{X}_{i,j}(\lambda), \tilde{y}_{i,j}(\lambda))].
\end{equation}
Now, we reformulate this loss function using a similar approach to that of \citet{carratino2022mixup, zhang2020does}:
\begin{align*}
L_{\text{Mixup}} =& \frac{1}{N^2} \sum_{i,j=1}^{N} \mathbb{E}_{\lambda \sim \Beta(\alpha, \alpha)}[\lambda \cdot L(F; \tilde{X}_{i,j}(\lambda), y_i) + (1-\lambda) \cdot L(F; \tilde{X}_{i,j}(\lambda), y_j)] \\
=& \frac{1}{N^2} \sum_{i,j=1}^{N} \mathbb{E}_{\lambda \sim \Beta(\alpha, \alpha)}\mathbb{E}_{B \sim \Bern(\lambda)} [B \cdot L(F; \tilde{X}_{i,j}(\lambda), y_i) + (1-B)\cdot  L(F; \tilde{X}_{i,j}(\lambda), y_j)] \\
=& \frac{1}{N^2} \sum_{i,j=1}^{N} \big\{ \frac{1}{2} \mathbb{E}_{\lambda \sim \Beta(\alpha+1, \alpha)} [L(F; \tilde{X}_{i,j}(\lambda), y_i)] + \frac{1}{2} \mathbb{E}_{\lambda \sim \Beta(\alpha, \alpha+1)} [L(F; \tilde{X}_{i,j}(\lambda), y_i)]\big\}.
\end{align*}

The above equation holds because \(\lambda \sim \Beta(\alpha, \alpha)\) follows a Beta distribution, and \(B | \lambda \sim \Bern(\lambda)\) follows a Bernoulli distribution. By conjugacy, we can exchange them in order and have \(
B \sim \Bern ( \frac{1}{2} ),\, \lambda | B \sim \Beta(\alpha + B, \alpha + 1 - B)\).  

Using the fact that \(1 - \Beta(\alpha, \alpha + 1)\) and \(\Beta(\alpha + 1, \alpha)\) share the same distribution, and observing that \(\tilde{X}_{ij}(1 - \lambda) = \tilde{X}_{ji}(\lambda)\), we derive:
\begin{align*}
\sum_{i,j=1}^N \mathbb{E}_{\lambda \sim \Beta(\alpha,\alpha+1)} [L(F; \tilde{X}_{i,j}(\lambda), y_j)]
= \sum_{i,j=1}^N \mathbb{E}_{\lambda \sim \Beta(\alpha+1,\alpha)} [ L(F; \tilde{X}_{i,j}(\lambda), y_i)].
\end{align*}
Therefore, we have 
\begin{equation}\label{appeq:mixup_loss}
\begin{aligned}
L_{\text{Mixup}} &= \frac{1}{N} \sum_{i=1}^{N} \mathbb{E}_{\lambda \sim \Beta(\alpha + 1, \alpha)} \mathbb{E}_{X' \sim \calZ_X } \, [L(F; \lambda X_i + (1 - \lambda) X', y_i)] \\
&= \mathbb{E}_{(X,y) \sim \calZ } \mathbb{E}_{X' \sim \calZ_X } \mathbb{E}_{\lambda \sim \Beta(\alpha + 1, \alpha)} [L(F; \lambda X + (1 - \lambda) X', y)],
\end{aligned}
\end{equation}
where \(\mathcal{Z}_X\) denotes the empirical distribution of training images. Equation~\eqref{appeq:mixup_loss} directly leads to \cref{lem:mixup_loss}. Based on the reformulated augmented loss in \eqref{appeq:mixup_loss}, we can characterize the feature mixing effect of \(\mathcal{A}_2\) as described in \cref{assum:a2}.

\paragraph{Loss function of SL with \(\calA_2\)} After applying the data augmentation \(\mathcal{A}_2\) to the training samples, the augmented supervised loss function \(L_{\mathcal{A}_2}\) at training step \(t\) is given by:
\begin{equation}\label{appeq:a2_loss}
\begin{aligned}
    L_{\calA_2}^{(t)}
    =\,& \bbE_{(X,y)\sim\calZ}\left[L(F^{(t)}; \calA_2(X), y)\right] \\ =\,&
    \frac{N-N_s}{N}\, \bbE_{(X,y)\sim\calZ_m}\left[L(F^{(t)}; \calA_2(X), y)\right] + \frac{N_s}{N}\, \bbE_{(X,y)\sim\calZ_s}\left[L(F^{(t)}; \calA_2(X), y)\right]
    \\ =\,& \frac{(N-N_s)(1-\pi_2)}{N}\, \bbE_{(X,y)\sim\calZ_m '}\left[L(F^{(t)}; X, y)\right] + \frac{N_s(1-\pi_2)}{N}\, \bbE_{(X,y)\sim\calZ_s '}\left[L(F^{(t)}; X, y)\right] \\ & + \frac{(N-N_s)\pi_2}{N} \, \bbE_{(X,y)\sim\tcalZ_m}\left[L(F^{(t)}; X, y)\right] + \frac{N_s\pi_2}{N}\, \bbE_{(X,y)\sim\tcalZ_s}\left[L(F^{(t)}; X, y)\right].
\end{aligned}
\end{equation}
Here, with a slight abuse of notation, we continue to use \(\mathcal{Z}_m'\) and \(\mathcal{Z}_s'\) to represent the training images where no features from other images are mixed in (only ``noises" are mixed in). Similarly, we use \(\tilde{\mathcal{Z}}_m\) and \(\tilde{\mathcal{Z}}_s\) to denote the training images with features from other images mixed in. The augmented training dataset is then defined as \(\tilde{\mathcal{Z}}_2 = \mathcal{Z}_m' \cup \tilde{\mathcal{Z}}_m \cup \mathcal{Z}_s' \cup \tilde{\mathcal{Z}}_s\).

\paragraph{Combined Effects of \(\calA_3\)}
As outlined in \cref{assum:a3}, the data augmentation \(\mathcal{A}_3\) combines the effects of partial semantic feature removal and feature mixing. Specifically, with a probability \(\pi_3\), \(\mathcal{A}_3\) partially removes one semantic feature, while simultaneously decreases the scale of original semantic features and increases the scale of noisy features. Below, we present the complete version of \cref{assum:a3}.

\begin{assumption} \label{appassum:a3}
For a given image \((X, y) \in \mathcal{Z}\), suppose the data augmentation \(\mathcal{A}_3\) has a probability \(\pi_3\) of mixing feature patches from another training sample \(X'\) into \(X\), while simultaneously partially removing one semantic feature. Under this condition, with probability \(\pi_3\), the augmented image satisfies the following: \\
\textbf{(a)} If \((X, y) \in \mathcal{Z}_s\), then \(\sum_{p \in \mathcal{P}_{v}(X)} z_p \in [C_1 - C_2, O(1)]\) for \(v = v_{y, l^*}\), where \(v_{y, l^*}\) is the semantic feature contained in \(X\). Additionally, \(\sum_{p \in \mathcal{P}_{v}(X)} z_p \in [\tilde{\Omega}(1), C_3]\) for \(v \in \mathcal{V}(X) \setminus \{v_{y,1}, v_{y,2}\}\). \\
\textbf{(b)} If \((X, y) \in \mathcal{Z}_m\), then \(\sum_{p \in \mathcal{P}_{v}(X)} z_p \in [C_1 - C_2, O(1)]\) for \(v = v_{y, l}\) (\(l \in [2]\) is randomly chosen), while it holds that \(\sum_{p \in \mathcal{P}_{v}(X)} z_p \in [1 - C_2, O(1)]\) for \(v = v_{y, 3-l}\). Additionally, \(\sum_{p \in \mathcal{P}_{v}(X)} z_p \in [\Omega(1), 0.4 + C_3]\) for \(v \in \mathcal{V}(X) \setminus \{v_{y,1}, v_{y,2}\}\).\\
Here, \(C_1\) denotes the scale of the semantic feature after partial semantic feature removal, \(C_2\) represents the scale reduction of semantic features due to feature mixing, and \(C_3\) signifies the scale increase of noisy features after feature mixing. The constants \(C_1\), \(C_2\), and \(C_3\) satisfy the following conditions: \(C_1 > C_2 + C_3\), \(C_2 + C_3 < 0.1 + \frac{C_1}{2}\), and \(C_1 < 0.4 + C_2 + C_3 \).
\end{assumption}
\cref{appassum:a3} encapsulates the combined effects of partial semantic feature removal and feature mixing introduced by the data augmentation \(\mathcal{A}_3\), highlighting their influence on the feature distribution within the training samples. The conditions \(C_1 > C_2 + C_3\) and \(C_2 + C_3 < 0.1 + \frac{C_1}{2}\) ensure that the scale of class-specific semantic features remains larger than the scale of noisy features after applying \(\mathcal{A}_3\). Additionally, the condition \(C_1 < 0.4 + C_2 + C_3\) is analogous to the requirement for \(C_1\) in \cref{assum:a1}, guaranteeing that one semantic feature is significantly reduced, thereby encouraging the network to learn comprehensive semantic features to optimize the objective. Compared to \cref{assum:a1}, the upper bound for \(C_1\) is relaxed due to the inclusion of the feature mixing effect in \(\mathcal{A}_3\).

To illustrate this combined effect of \(\mathcal{A}_3\), we use CutMix as an example. CutMix randomly selects two data samples \((X_i, y_i),(X_j,y_j) \in \calZ (i\neq j \in [N])\) and generate an augmented sample \((\tilde{X}_{i,j}(\lambda), \tilde{y}_{i,j}(\lambda))\) defined as follows:
\begin{equation}
\begin{aligned}
\tilde{X}_{i,j}(\lambda) &= M(\lambda) \odot X_i + (1 - M(\lambda)) \odot X_j, \\
\tilde{y}_{i,j}(\lambda) &= \lambda y_i + (1 -\lambda) y_j,
\end{aligned}
\end{equation}
where \(\lambda\) is a random variable subject to the Beta distribution \(\lambda \sim \Beta(\alpha,\alpha)\), and \(M(\lambda)\) denotes a binary mask indicating which patches are to remove and fill in from two images. Specifically, in CutMix, \(M(\lambda)\) is a rectangle region randomly positioned whose width and height are \(\sqrt{1-\lambda}\) times of the original image, and \(\tilde{X}_{i,j}(\lambda)\) is \(X_i\) removes this square region with the same region of \(X_j\) filled in.

Then, we can formulate the loss function for SL with CutMix as
\begin{equation}
L_{\text{CutMix}} = \frac{1}{N^2} \sum_{i,j=1}^{N} \mathbb{E}_{\lambda \sim \Beta(\alpha, \beta) }\mathbb{E}_{M(\lambda)} \, [L(F; \tilde{X}_{i,j}(\lambda), \tilde{y}_{i,j}(\lambda))].
\end{equation}
Similarly to the analysis of \(L_{\text{Mixup}}\) in the above paragraph \emph{Feature Mixing Effect of \(\mathcal{A}_2\)}, we can reformulate \(L_{\text{CutMix}}\) as:
\begin{equation}\label{appeq:cutmix_loss}
L_{\text{CutMix}} = \mathbb{E}_{(X, y) \sim \mathcal{Z}} \mathbb{E}_{X' \sim \mathcal{Z}_X} \mathbb{E}_{\lambda \sim \Beta(\alpha + 1, \alpha)} \mathbb{E}_{M(\lambda)} 
\, [L(F; M(\lambda) \odot X + (1 - M(\lambda)) \odot X', y)],
\end{equation}
where \(\mathcal{Z}_X\) denotes the empirical distribution of training images. Based on the reformulated augmented loss in \eqref{appeq:cutmix_loss}, we can characterize the combined partial semantic feature removal and feature mixing effect of \(\mathcal{A}_3\) as described in \cref{assum:a3}.

\paragraph{Loss function of SL with \(\calA_3\)}
After applying the data augmentation \(\mathcal{A}_3\) to the training samples, the augmented supervised loss function \(L_{\mathcal{A}_3}\) at training step \(t\) becomes:

\begin{equation} \label{appeq:a3_loss}
\begin{aligned}
    L_{\calA_3}^{(t)}
    =\,& \bbE_{(X,y)\sim\calZ}\left[L(F^{(t)}; \calA_3(X), y)\right] \\ =\,&
    \frac{N-N_s}{N}\, \bbE_{(X,y)\sim\calZ_m}\left[L(F^{(t)}; \calA_3(X), y)\right] + \frac{N_s}{N}\, \bbE_{(X,y)\sim\calZ_s}\left[L(F^{(t)}; \calA_3(X), y)\right]
    \\ =\,& \frac{(N-N_s)(1-\pi_3)}{N}\, \bbE_{(X,y)\sim\calZ_m '}\left[L(F^{(t)}; X, y)\right] + \frac{N_s(1-\pi_3)}{N}\, \bbE_{(X,y)\sim\calZ_s '}\left[L(F^{(t)}; X, y)\right] \\ & + \frac{(N-N_s)\pi_3}{N} \, \bbE_{(X,y)\sim\tcalZ_m}\left[L(F^{(t)}; X, y)\right] + \frac{N_s\pi_3}{N}\, \bbE_{(X,y)\sim\tcalZ_s}\left[L(F^{(t)}; X, y)\right].
\end{aligned}
\end{equation}
Here, \(\mathcal{Z}_m'\) represents the subset of multi-view data in \(\mathcal{Z}_m\) where \(\mathcal{A}_3\) manipulates only the ``noises" without significantly affecting the feature distribution. In contrast, \(\tilde{\mathcal{Z}}_m = \tilde{\mathcal{Z}}_m^1 \cup \tilde{\mathcal{Z}}_m^2\) contains the multi-view data in \(\mathcal{Z}_m\) where \(\mathcal{A}_3\) combines feature mixing with partial removal of semantic features, as described in \cref{appassum:a3}. Specifically, \(\tilde{\mathcal{Z}}_m^1\) denotes the subset where \(v_{y,2}\) is largely removed, and \(\tilde{\mathcal{Z}}_m^2\) denotes the subset where \(v_{y,1}\) is largely removed.

Similarly, \(\mathcal{Z}_s'\) represents the subset of single-view data in \(\mathcal{Z}_s\) where \(\mathcal{A}_3\) manipulates only the ``noises" without significantly affecting the semantic feature. In contrast, \(\tilde{\mathcal{Z}}_s\) contains the single-view data in \(\mathcal{Z}_s\) where \(\mathcal{A}_3\) significantly reduces the single semantic feature present while mixing in noisy features. The augmented training dataset is then defined as \(\tilde{\mathcal{Z}}_3 = \mathcal{Z}_m' \cup \tilde{\mathcal{Z}}_m \cup \mathcal{Z}_s' \cup \tilde{\mathcal{Z}}_s\).

\section{Results on Vanilla SL}
\label{appsec:sl}
In this section, we first recall the results of vanilla SL that were derived in~\citet{allen-zhu2023towards}. Before showing their main results, we first introduce some necessary notations. For every $i\in [k]$ and $l\in [2]$, define $\Phi_{i,l}^{(t)}: = \sum_{r\in[m]} [\langle w_{i,r}^{(t)}, v_{i,l}\rangle]^{+}$ and $\Phi_{i}^{(t)}:= \sum_{l\in[2]} \Phi_{i,l}^{(t)}$. Define
\begin{align*}
    \Lambda_{i}^{(t)} := \max_{r\in[m],l\in[2]} [\langle w_{i,r}^{(t)}, v_{i,l}\rangle]^{+}\quad \mbox{and}\quad \Lambda_{i,l}^{(t)} := \max_{r\in[m]} [\langle w_{i,r}^{(t)}, v_{i,l} \rangle]^{+},
\end{align*}
where $\Lambda_{i,l}$ indicates the largest correlation between the feature vector $v_{i,l}$ and all neurons $ w_{i,r}\; (r\in[m])$ from class $i$.  

Then, suppose we denote by $S_{i,l} := \mathbb{E}_{(X,y) \sim \mathcal{Z}_m} [ \bbI_{y=i} \sum_{p \in P_{v_{i,l}}(X)} z_p^q ]$, we can define the ``view lottery winning" set as follows:
\[
\mathcal{M} := \left\{ (i, l) \in [k] \times [2] \,\middle|\, \Lambda_{i,l}^{(0)} \geq \Lambda_{i, 3-l}^{(0)} \left( \frac{S_{i, 3-l}}{S_{i,l}} \right)^{\frac{1}{q-2}} \left( 1 + \frac{1}{\log^2(m)} \right) \right\}.
\]

The intuition behind $\calM$ is that, subject to training dataset distribution and model initialization, if $(i,l)\in\calM$, then the feature $v_{i,l}$ will be learned by the model during vanilla supervised training and the feature $v_{i,3-l}$ will be missed. 



The ``lottery winning" set $\calM$ satisfies the following property (refer to the Proposition C.2. of~\citet{allen-zhu2023towards}):
\begin{proposition}
    Suppose $m\leq\polyk$. For every $i\in[k]$, $\Pr[(i,1)\in\calM~\mbox{or } (i,2)\in\calM]\geq 1-o(1)$.
\end{proposition}
This means that with probability close to 1, one of the two semantic features per class belongs to ``lottery winning" set $\calM$.

Based on Theorem 1 of \citet{allen-zhu2023towards}, after training for $T$ iterations with the vanilla supervised training loss $L^{(t)}$
\begin{align*}
L^{(t)}= \bbE_{(X,y)\sim\calZ}\left[-\log\logit_y(F^{(t)},X)\right],
\end{align*}
the training accuracy on training samples is perfect and $L^{(T)}$ approaches zero, i.e., for every $(X,y)\in\calZ$, 
\begin{align*}
    \forall i\neq y: F_y^{(T)}(X)\geq F_i^{(T)}(X)+\Omega(\log k),
\end{align*}
and we have \(L^{(T)} \le \frac{1}{\polyk}\).
Besides, it satisfies that $0.4\Phi_i^{(T)} - \Phi_j^{(T)}\leq - \Omega(\log k)$ for every pair $i,j\in[k]$. This means that at least one of $\Phi_{i,1}^{(T)}$ or $\Phi_{i,2}^{(T)}$ for all $i\in[k]$ increase to a large scale of $\Theta(\log(k))$, which means at least one of $v_{i,1}$ and $v_{i,2}$ for all $i\in [k]$ is learned after supervised training for $T$ iterations. Thus, all multi-view training data are classified correctly. For single-view training data without the learned features, they are classified correctly by memorizing the noises in the data during the supervised training process.  
Then for the test accuracy, for the multi-view data point $(X,y)\sim\calD_m$, with the probability at least $1-e^{-\Omega(\log^2 k)}$, it has 
\begin{align*}
    \logit_y(F^{(T)},X)\geq 1-\tilde{O} \left(\frac{1}{s^2}\right),
\end{align*}
and 
\begin{align*}
    \Pr_{(X,y)\sim \calD_m}\left[ F_y^{(T)}(X)\geq \max_{j\neq y}F_j^{(T)}(X)+\Omega(\log k)\right]\geq 1-e^{-\Omega(\log^2k)}.
\end{align*} 
This means that the test accuracy of multi-view data is good. However, for the single-view data $(X,y)\sim\calD_s$, whenever $(i,l^*)\in\calM$, we have $\Phi_{i, 3-l^*}^{(T)}\ll \frac{1}{\polylogk}$ and 
\begin{align*}
    \Pr_{(X,y)\sim \calD_s}\left[ F_y^{(T)}(X)\geq \max_{j\neq y}F_j^{(T)}(X)-\frac{1}{\polylogk}\right]\leq \frac{1}{2}(1+o(1)),
\end{align*}
which means that the test accuracy on single-view data is nearly $50\%$.

The results in~\citet{allen-zhu2023towards} fully indicate the feature learning process of vanilla SL. The main reason for its imperfect performance is that, due to ``lottery winning", it only captures one of the two features for each semantic class during the supervised training process. Therefore for single-view data without this feature, it has low test accuracy.

\section{Induction Hypothesis} \label{appsec:induction}
In this section, to prove our theorems, similar to~\citet{allen-zhu2023towards}, we present an induction hypothesis for every training iteration $t$ of SL with data augmentation $\calA_1$ and $\calA_3$. We introduce the following induction hypothesis:

\begin{hypothesis}
\label{hyp}
For every $l\in[2]$, for every $r\in[m]$, for every $(X,y)\in\tilde{\calZ}_1, \tilde{\calZ}_3$ and $i\in[k]$, 
\begin{itemize}
    \item [(a)] For every $p\in\calP_{v_{i,l}}(X)$, we have: $\langle w_{i,r}^{(t)}, x_p\rangle = \langle w_{i,r}^{(t)}, v_{i,l}\rangle z_p \pm \tilde{o}(\sigma_0) $.
    \item[(b)] For every $p \in \calP(X)\setminus (\calP_{v_{i,1}}(X)\cup \calP_{v_{i,2}}(X))$, we have: $|\langle w_{i,r}^{(t)}, x_p\rangle| \leq \tilde{O}(\sigma_0)$.
    \item[(c)] For every $p\in[P]\setminus \calP(X)$, we have $|\langle w_{i,r}^{(t)}, x_p\rangle| \leq \tilde{O}(\sigma_0\gamma k)$.
\end{itemize}
Moreover, we have for every $i\in[k]$, every $l\in[2]$, 
\begin{itemize}
    \item [(d)] $\Phi_{i,l}^{(t)}  \geq \tilde{\Omega}(\sigma_0)$ and $\Phi_{i,l}^{(t)}   \leq \tilde{O}(1)$.
    \item [(e)] for every $r\in[m]$, it holds that $\langle w_{i,r}^{(t)}, v_{i,l}\rangle \geq -\tilde{O}(\sigma_0)$.
\end{itemize}
\end{hypothesis}
The intuition behind Induction Hypothesis~\ref{hyp} is that SL with data augmentation incorporating the partial semantic feature removal effect can effectively filter out both feature noises and background noises for multi-view and single-view data. This contrasts with Induction Hypothesis C.3 of \citet{allen-zhu2023towards}, where vanilla SL is limited to filtering out noises only for multi-view data. For SL with data augmentation \(\mathcal{A}_2\), which focuses on the feature mixing effect, the induction hypothesis remains the same as that of vanilla SL described in Induction Hypothesis C.3 of \citet{allen-zhu2023towards} and is therefore omitted here for simplicity.

\section{Gradient Calculations and Function Approximation}

\label{appsec:func}

\paragraph{Gradient Calculation.} We present the gradients calculations for the cross-entropy loss \(L(F;X,y)=-\log \logit_y(F,X)\) and function approximations of \(F\), results in this section are followed from \citet{allen-zhu2023towards}.
\begin{fact}\label{app:grad_cal}
    Given data point $(X,y)\sim\tilde{\calZ}_1, \tcalZ_2,\tcalZ_3$, for every $i\in[k], r\in[m]$,
    \begin{align}
        -\nabla_{w_{i,r}} L(F;X,y)& = (1-\logit_i(F,X)) \sum_{p\in[P]} \trelu'(\langle w_{i,r}, x_p\rangle ) x_p,\quad~\mbox{when } i=y,\\
        -\nabla_{w_{i,r}} L(F;X,y)& = - \logit_i(F,X) \sum_{p\in[P]} \trelu'(\langle w_{i,r}, x_p\rangle ) x_p,\quad~\mbox{when } i\neq y.
    \end{align}
\end{fact}

\begin{definition}
    For each data point $X$, we define a value $V_{i,r,l}(X)$ as
    \begin{align*}
        V_{i,r,l}(X):=\bbI_{v_{i,l}\in\calV(X)}\sum_{p\in\calP_{v_{i,l}}(X)} \trelu'(\langle w_{i,r}, x_p\rangle)z_p.
    \end{align*}
\end{definition}

\begin{definition}
    We also define small error terms which will be frequently used:
    \begin{align*}
        \calE_1:=\tilde{O}(\sigma_0^{q-1})\gamma s \qquad & \calE_{2,i,r}(X):= O(\gamma(V_{i,r,1}(X) + V_{i,r,2}(X)))\\
        \calE_3:= \tilde{O}(\sigma_0\gamma k)^{q-1}\gamma P \qquad & \calE_{4,j,l}(X): =\tilde{O}(\sigma_0^{q-1})\bbI_{v_{j,l}\in\calV(X)}.
    \end{align*}
\end{definition}

Then we have the following bounds for positive gradients, i.e., when $i=y$:
\begin{claim}[positive gradients]
\label{claim:pos}
    Suppose Induction Hypothesis~\ref{hyp} holds at iteration $t$. For every $(X,y)\in\tilde{\calZ}_1, \tcalZ_2,\tcalZ_3$, every $r\in[m]$, every $l\in[2]$, and $i=y$, we have
    \begin{itemize}
        \item [(a)] $\langle -\nabla_{w_{i,r}} L(F^{(t)};X,y), v_{i,l}\rangle \geq \left(V_{i,r,l}(X)-\tilde{O}(\sigma_p P)\right)(1-\logit_i(F^{(t)},X)).$
        \item [(b)] $\langle -\nabla_{w_{i,r}} L(F^{(t)};X,y), v_{i,l}\rangle \leq \left(V_{i,r,l}(X)+\calE_1+\calE_3\right)(1-\logit_i(F^{(t)},X)).$
        \item [(c)] For every $j\in[k]\setminus \{i\}$, 
        \begin{align*}
            |\langle -\nabla_{w_{i,r}} L(F^{(t)};X,y), v_{j,l}\rangle| \leq \left(\calE_1+\calE_{2,i,r}(X)+\calE_3+\calE_{4,j,l}(X)\right)(1-\logit_i(F^{(t)},X)).
        \end{align*}
    \end{itemize}
\end{claim}
We also have the following claim about the negative gradients (i.e., $i\neq y$). The proof of positive and negative gradients is identical to the proof in~\citet{allen-zhu2023towards}, as we already move the effect of data augmentation \(\calA_1,\calA_2,\calA_3\) to the change of training dataset distribution \(\calZ\) to \(\tcalZ_1,\tcalZ_2,\tcalZ_3\).
\begin{claim}[negative gradients]
\label{claim:neg}
    Suppose Induction Hypothesis~\ref{hyp} holds at iteration $t$. For every $(X,y)\sim\tcalZ_1,\tcalZ_2,\tcalZ_3$, every $r\in[m]$, every $l\in[2]$, and $i\in[k]\setminus \{y\}$, we have 
    \begin{itemize}
        \item[(a)] $\langle -\nabla_{w_{i,r}} L(F^{(t)};X,y), v_{i,l}\rangle \geq -\logit_i(F^{(t)},X)\left(\calE_1+\calE_3 + V_{i,r,l}(X)\right).$
        \item[(b)] For every $j\in[k]$: $\langle -\nabla_{w_{i,r}} L(F^{(t)};X,y), v_{j,l}\rangle \leq \logit_i(F^{(t)},X)\tilde{O}(\sigma_p P).$
        \item[(c)] For every $j\in[k]\setminus \{i\}$: $\langle -\nabla_{w_{i,r}} L(F^{(t)};X,y), v_{j,l}\rangle \geq -\logit_i(F^{(t)},X)\left(\calE_1+\calE_3 + \calE_{4,j,l}(X)\right).$
    \end{itemize}
\end{claim}

\paragraph{Function Approximation.} For the neural network prediction function $F_i^{(t)}$ ($i\in [k]$), we can derive an approximation of it:
\begin{claim}[function approximation]
\label{claim:func}
    Suppose \cref{hyp} holds at iteration $t$ and supposes $s\leq\tilde{O}(\frac{1}{\sigma_0^q m})$ and $\gamma\leq \tilde{O}(\frac{1}{\sigma_0 k (mP)^{1/q}})$. Let $Z_{i,l}^{(t)}(X):=\bbI_{v_{i,l}\in\calV(X)}\left(\sum_{p\in\calP_{v_{i,l}}(X)}z_p\right)$, we have:
    \begin{itemize}
        \item for every $t$, every $(X,y)\in\tcalZ_1,\tcalZ_2,\tcalZ_3$ and $i\in[k]$, we have 
        \begin{align*}
            F_i^{(t)}(X) = \sum_{l\in[2]} \left(\Phi_{i,l}^{(t)}\times Z_{i,l}^{(t)}(X)\right) \pm O\left(\frac{1}{\polylogk}\right).
        \end{align*}
        \item for every $(X,y)\sim\calD$, with probability at least $1-e^{-\Omega(\log^2 k)}$, it satisfies for every $i\in[k]$,
        \begin{align*}
            F_i^{(t)}(X)=\sum_{l\in[2]} \left(\Phi_{i,l}^{(t)}\times Z_{i,l}^{(t)}(X)\right) \pm O\left(\frac{1}{\polylogk}\right).
        \end{align*}
    \end{itemize}
\end{claim}

\section{Proof for SL with $\calA_1$}
\label{appsec:proof_a1}
Here we first prove the theorem for SL with $\calA_1$. We first introduce useful claims as consequences of the Induction Hypothesis~\ref{hyp}.

\subsection{Useful Claims} \label{app_phase2_growth}
In this subsection, we introduce several claims as consequences of Induction Hypothesis \ref{hyp}. They shall be useful in our later proof of the induction hypothesis.

The first claim is about the lambda growth, it shows that after a certain amount training iterations, \(\Phi_{i,l}\) (\(i\in [k], l\in [2]\)) grows to the scale of \(\Theta(1)\), simultaneously. Below we first give naive upper bounds on the \(\logit\) function based on function approximation result in Claim \ref{claim:func}.

\begin{claim}[approximation of logits]
\label{claim:logit}
Suppose Induction Hypothesis~\ref{hyp} holds at iteration $t$, and supposes $s\leq\tilde{O}(\frac{1}{\sigma_0^q m})$ and $\gamma\leq \tilde{O}(\frac{1}{\sigma_0 k (mP)^{1/q}})$, then 
\begin{itemize}
    \item for every $(X,y)\in\calZ_m',\tcalZ_m$ and \(i \in [k] \setminus \{y\}\): $\logit_i(F^{(t)}, X) = O\left(\frac{e^{0.4\Phi_i^{(t)}}}{e^{0.4\Phi_i^{(t)}} + k}\right)$.
    \item for every $(X,y)\in \calZ_s',\tcalZ_s$ and $i\in[k] \setminus \{y\}$: $\logit_i(F^{(t)}, X) = O\left(\frac{1}{k}\right)$.
    \item for every $(X,y)\in\calZ_m',\tcalZ_m$: $\logit_y(F^{(t)}, X) = O\left(\frac{e^{O(\Phi_{y}^{(t)})}}{e^{O(\Phi_{y}^{(t)})} + k}\right)$.
    \item for every $(X,y)\in\calZ_s',\tcalZ_s$: 
    $\logit_y(F^{(t)}, X) = O\left(\frac{e^{O(\Phi_{y,l^*}^{(t)})}}{e^{O(\Phi_{y,l^*}^{(t)})} + k}\right)$.
\end{itemize}
\end{claim}
\begin{proof}
    From Claim \ref{claim:func}, we know 
    \begin{align*}
    F_i^{(t)}(X) = \sum_{l\in[2]} \left(\Phi_{i,l}^{(t)}\times Z_{i,l}^{(t)}(X)\right) \pm O\left(\frac{1}{\polylogk}\right).
    \end{align*}
    According to data distribution \cref{def:data2} and \cref{assum:a1} posed on $\calA_1$, we know for \((X,y) \in \calZ_m',\tcalZ_m\) and \(i \in [k] \setminus \{y\}\), 
    \begin{align*}
    0\leq F_i^{(t)}(X) \leq \Phi_i^{(t)} \cdot 0.4 + O\left(\frac{1}{\polylogk}\right).
    \end{align*}
    For \((X,y)\in\calZ_s',\tcalZ_s\) and \(i\in [k] \setminus \{y\}\), 
    \begin{align*}
    0\leq F_i^{(t)}(X) \leq \Phi_{i}^{(t)} \cdot \Gamma + O\left(\frac{1}{\polylogk}\right) \leq O(1).
    \end{align*}
    For \((X,y) \in \calZ_m',\tcalZ_m\),
    \begin{align*}
    0\leq F_y^{(t)}(X) \leq \Phi_y^{(t)} \cdot O(1) + O\left(\frac{1}{\polylogk}\right).
    \end{align*}
    For \((X,y) \in \calZ_s',\tcalZ_s\),
    \begin{align*}
    0\leq F_y^{(t)}(X) \leq \Phi_{y,l^*}^{(t)} \cdot O(1) + O\left(\frac{1}{\polylogk}\right),
    \end{align*}
    where \(v_{y,l^*}\) is the only semantic feature \(X\) contain.
    Recall \(\logit_i (F^{(t)}, X) = \frac{e^{F_i^{(t)}(X)}}{\sum_{j=1}^k e^{ F_j^{(t)}(X)}}\), we have all the results in the claim.
\end{proof}

The following claim illustrates the initial growth of every \(\Phi_{i,l} (i\in [k], l\in [2])\) to the scale of \(\Theta(1)\).
\begin{claim}[lambda growth]
\label{claim:growth}
    Suppose Induction Hypothesis~\ref{hyp} holds at iteration $t$, then for every $i\in[k]$, suppose $\Phi_{i,l}^{(t)} \leq O(1)$, then it satisfies
    \begin{align*}
        \Phi_{i,l}^{(t+1)} = \Phi_{i,l}^{(t)}+\tilde{\Theta}\left(\frac{\eta}{k}\right)\trelu'(\Phi_{i,l}^{(t)}).
    \end{align*}
\end{claim}
\begin{proof}
    For any \(w_{i,r}\) and \(v_{i,l}\) (\(i\in[k],r\in[m],l\in[2]\)), we have
    \begin{align*}
        \langle w_{i,r}^{(t+1)}, v_{i,l}\rangle 
        =& \langle w_{i,r}^{(t)}, v_{i,l}\rangle - \eta \bbE_{(X,y)\sim\tcalZ_1}\big[\langle \nabla_{w_{i,r}} L(F^{(t)};X,y) , v_{i,l}\rangle\big] \\
        =& \langle w_{i,r}^{(t)}, v_{i,l}\rangle + \eta \bbE_{(X,y)\sim\tcalZ_1}\big[\bbI_{y=i}(1-\logit_i(F^{(t)},X)) \sum_{p\in[P]} \trelu'(\langle w_{i,r}, x_p\rangle )\langle x_p, v_{i,l}\rangle \\
        & - \bbI_{y\neq i} \logit_i(F^{(t)},X) \sum_{p\in[P]} \trelu'(\langle w_{i,r}, x_p\rangle ) \langle x_p, v_{i,l}\rangle\big].
    \end{align*}
    From Claim \ref{claim:pos} and Claim \ref{claim:neg}, we know
    \begin{align*}
    \langle w_{i,r}^{(t+1)}, v_{i,l}\rangle \geq \langle w_{i,r}^{(t)}, v_{i,l} \rangle + \eta \mathbb{E}_{(X,y) \in \tcalZ_1} [& \mathbb{I}_{y=i} ( V_{i,r,l}(X) - \tilde{O}(\sigma_p P) ) ( 1 - \logit_i(F^{(t)}, X) )\\ & - \mathbb{I}_{y \neq i} \left( \mathcal{E}_1 + \mathcal{E}_3 + V_{i,r,l}(X) \right) \logit_i(F^{(t)}, X)].
    \end{align*}
    Consider \(r = \argmax_{r\in [m]} \{\langle w_{i,r}^{(t)}, v_{i,l}\rangle\}\), then as \(m = \polylogk\), we know \(\langle w_{i,r}^{(t)}, v_{i,l}\rangle \ge \tilde{\Omega}(\Phi_{i,l}^{(t)})\).
    Recall \(V_{i,r,l}(X)=\bbI_{v_{i,l}\in\calV(X)}\sum_{p\in\calP_{v_{i,l}}(X)} \trelu'(\langle w_{i,r}^{(t)}, x_p\rangle)z_p\), according to Induction Hypothesis \ref{hyp}(a), we have
    \begin{align*}
        V_{i,r,l}(X)= \bbI_{v_{i,l}\in\calV(X)}\sum_{p\in\calP_{v_{i,l}}(X)} \trelu'(\langle w_{i,r}^{(t)}, v_{i,l}\rangle z_p + \tilde{o}(\sigma_0)) z_p.
    \end{align*}
\begin{itemize}
    \item When \(i=y\), at least for \((X,y)\in \calZ_m'\), we have \(\sum_{p\in \calP_{v_{i,l}}(X)} z_p \ge 1\), and together with \(|\calP_{v_{i,l}}| \le C_p\), we know \(V_{i,r,l} \ge \Omega(1) \cdot \trelu'(\langle w_{i,r}^{(t)}, v_{i,l}\rangle)\).
    \item When \(i \neq y\) and when \(v_{i,l} \in \calV(X)\), we can use \(\sum_{p\in \calP_{v_{i,l}}(X)} z_p \le 0.4\) to derive that \(V_{i,r,l}\le 0.4\cdot \trelu'(\langle w_{i,r}^{(t)}, v_{i,l}\rangle)\).
\end{itemize}
    Moreover, when \(\Phi_{i,l}^{(t)} \le O(1)\), by Claim \ref{claim:logit}, we have \(\logit_i(F^{(t)}, X) \le O(\frac{1}{k})\). Then we can derive that
    \begin{align*}
    \langle w_{i,r}^{(t+1)}, v_{i,l} \rangle \geq \langle w_{i,r}^{(t)}, v_{i,l} \rangle + &  \eta \mathbb{E}_{(X,y) \in \tcalZ_1}[ \mathbb{I}_{y=i} \cdot \Omega(1) - 0.4\cdot \mathbb{I}_{y \neq i} \mathbb{I}_{v_{i,l} \in \calV(X)} \cdot \frac{1}{k} ] \cdot \trelu' \langle w_{i,r}^{(t)}, v_{i,l} \rangle \\ - & \eta \tilde{O}( \frac{\sigma_p P + \mathcal{E}_1 + \mathcal{E}_3}{k}).
    \end{align*}
    Finally, recall that $\Pr(v_{i,l}\in\calV(X)|i\neq y)=\frac{s}{k}\ll o(1)$, we have that
    \begin{align*}
        \langle w_{i,r}^{(t+1)}, v_{i,l}\rangle 
        &\geq \langle w_{i,r}^{(t)}, v_{i,l}\rangle +\tilde{\Omega}\left(\frac{\eta }{k}\right)\trelu'(\langle w_{i,r}^{(t)}, v_{i,l}\rangle).
    \end{align*}
    Similarly, using Claim \ref{claim:pos} and \ref{claim:neg}, we can derive:
    \begin{align*}
    \langle w_{i,r}^{(t+1)}, v_{i,l} \rangle \leq \langle w_{i,r}^{(t)}, v_{i,l} \rangle + \eta \mathbb{E}_{(X,y) \in \tcalZ_1} \big[ & \mathbb{I}_{y=i} ( V_{i,r,l}(X) + \mathcal{E}_1 + \mathcal{E}_3 ) ( 1 - \logit_i(F^{(t)}, X) ) \\ &- \mathbb{I}_{y \neq i} ( \tilde{O}(\sigma_p P) ) \logit_i(F^{(t)}, X) \big].
    \end{align*}
    With similar analyses to the upper bound, we can derive the lower bound
    \begin{align*}
    \langle w_{i,r}^{(t+1)}, v_{i,l}\rangle 
    &\leq \langle w_{i,r}^{(t)}, v_{i,l}\rangle +\tilde{O}\left(\frac{\eta }{k}\right)\trelu'(\langle w_{i,r}^{(t)}, v_{i,l}\rangle).
\end{align*}
\end{proof}

With lambda growth analysis in Claim~\ref{claim:growth}, similar to Claim D.11 in~\citet{allen-zhu2023towards}, we could obtain the following result:
\begin{claim}
\label{claim:feature}
    Define iteration threshold $T_0:=\tilde{\Theta}\left(\frac{k}{\eta\sigma_0^{q-2}}\right)$, then for every $i\in[k], l\in[2]$ and $t\geq T$, it satisfies that $\Phi_{i, l}^{(t)}=\Theta(1)$.
\end{claim}

Next, we present the convergence of multi-view data from \(T_0\) till the end.
\begin{claim}[multi-view error till the end]
\label{claim:multi_error}
    Suppose that Induction Hypothesis~\ref{hyp} holds for every iteration $t\le T$, then 
    \begin{align*}
         \sum_{t=T_0}^{T} \bbE_{(X,y)\sim\calZ_{m}'}[1-\logit_y(F^{(t)}, X)] &\leq \tilde{O}\left(\frac{k}{\eta}\right),\\
         \sum_{t=T_0}^{T} \bbE_{(X,y)\sim\tilde{\calZ}_{m}}[1-\logit_y(F^{(t)}, X)] &\leq \tilde{O}\left(\frac{k}{\eta\pi_1}\right).
    \end{align*}
\end{claim}
\begin{proof}
Fix \(i\in[k]\) and \(l\in [2]\), we know for any \(w_{i,r}\) (\(r\in[m]\)), we have
\begin{equation}\label{eq:multi_error1}
    \begin{aligned}
        \langle w_{i,r}^{(t+1)}, v_{i,l}\rangle 
        =& \langle w_{i,r}^{(t)}, v_{i,l}\rangle - \eta \bbE_{(X,y)\sim\tcalZ_1}\big[\langle \nabla_{w_{i,r}} L(F^{(t)};X,y) , v_{i,l}\rangle\big] \\
        \\=& \langle w_{i,r}^{(t)}, v_{i,l}\rangle - \frac{\eta(N-N_s)(1-\pi_1)}{N} \bbE_{(X,y)\sim\calZ_m'}[\langle \nabla_{w_{i,r}} L(F^{(t)};X,y), v_{i,l}\rangle]\\ 
        &- \frac{\eta(N-N_s)\pi_1}{N} \bbE_{(X,y)\sim\tcalZ_m}[\langle \nabla_{w_{i,r}} L(F^{(t)};X,y), v_{i,l}\rangle] \\
        &- \frac{\eta N_s(1-\pi_1)}{N} \bbE_{(X,y)\sim\calZ_s'}[\langle \nabla_{w_{i,r}} L(F^{(t)};X,y), v_{i,l}\rangle] \\
        &- \frac{\eta N_s\pi_1}{N} \bbE_{(X,y)\sim\tcalZ_s}[\langle \nabla_{w_{i,r}} L(F^{(t)};X,y), v_{i,l}\rangle]
    \end{aligned}
\end{equation}
    Take \(r = \argmax_{r\in [m]} \{\langle w_{i,r}^{(t)}, v_{i,l}\rangle\}\), then by \(m=\polylogk\) we know \(\langle w_{i,r}^{(t)}, v_{i,l}\rangle \ge \tilde{\Omega}(\Phi_{i,l}^{(t)})=\tilde{\Omega}(1)\) for \(t\ge T_0\). By Claim \ref{claim:pos} and Claim \ref{claim:neg}, we have 
    \begin{align*}
    & -\bbE_{(X,y)\sim\calZ^*}[\langle \nabla_{w_{i,r}} L(F^{(t)};X,y), v_{i,l}\rangle] \\ =& \bbE_{(X,y)\sim\calZ^*}\big[\bbI_{y=i}(1-\logit_i(F^{(t)},X)) \sum_{p\in[P]} \trelu'(\langle w_{i,r}, x_p\rangle )\langle x_p, v_{i,l}\rangle \\ & \qquad \qquad - \bbI_{y\neq i} \logit_i(F^{(t)},X) \sum_{p\in[P]} \trelu'(\langle w_{i,r}, x_p\rangle ) \langle x_p, v_{i,l}\rangle\big] \\
    \ge&\mathbb{E}_{(X,y) \sim \calZ^*} [ \mathbb{I}_{y=i} ( V_{i,r,l}(X) - \tilde{O}(\sigma_p P) ) ( 1 - \logit_i(F^{(t)}, X) )\\ &\qquad \qquad - \mathbb{I}_{y \neq i} \left( \mathcal{E}_1 + \mathcal{E}_3 + V_{i,r,l}(X) \right) \logit_i(F^{(t)}, X)],
    \end{align*}
    where \(\calZ^*\) can be \(\calZ_m',\tcalZ_m, \calZ_s',\tcalZ_s\). 
    Recall \(V_{i,r,l}(X)= \bbI_{v_{i,l}\in\calV(X)}\sum_{p\in\calP_{v_{i,l}}(X)} \trelu'(\langle w_{i,r}^{(t)}, x_p\rangle)z_p\), according to \cref{hyp}(a), we have
    \begin{align*}
    V_{i,r,l}(X) = \bbI_{v_{i,l}\in\calV(X)} \sum\nolimits_{p\in\calP_{v_{i,l}}(X)} \trelu'(\langle w_{i,r}^{(t)}, v_{i,l}\rangle z_p + \tilde{o}(\sigma_0)) z_p.
    \end{align*}
    Since when \(t\ge T_0\), \(\langle w_{i,r}^{(t)}, v_{i,l}\rangle \ge \tilde{\Omega}(1) \gg \varrho\), and \(|\calP_{v_{i,l}}(X)|\leq O(1)\), for most of \( p\in \calP_{v_{i,l}}(X) \) must be already in the linear regime of \(\trelu\), which means we have
    \begin{align*}
        0.9 \sum\nolimits_{p\in\calP_{v_{i,l}}(X) }z_p \leq V_{i,r,l}(X) \leq \sum\nolimits_{p\in\calP_{v_{i,l}}(X) }z_p.
    \end{align*}
    Thus, for $(X,y)\sim\calZ_m'$, when $y=i$, we have $V_{i,r,l}(X)\geq 0.9$, when $y\neq i$ and $v_{i,l}\in\calV(X)$, we have $V_{i,r,l}(X)\leq 0.4$. For $(X,y)\sim\tcalZ_m^l$, when $y=i$, we have $V_{i,r,l}(X)\geq 0.9$, when $y\neq i$ and $v_{i,l} \in \calV(X)$, we have $V_{i,r,l}(X)\leq 0.4$. For $(X,y)\sim\tcalZ_m^{3-l}$, when $y=i$, we have $V_{i,r,l}(X) \ge 0.9 \cdot C_1$, when $y\neq i$ and $v_{i,l}\in\calV(X)$, we have $V_{i,r,l}(X)\leq 0.4$. For $(X,y)\sim\calZ_s'$, when $y=i$, we have $V_{i,r,l}(X)\ge 0.9$ if \(v_{i,l}\) is the class-specific feature contained in \(X\) else $V_{i,r,l}(X)\le O(\rho) \ll o(1)$, when $y\neq i$ and $v_{i,l}\in\calP(X)$, we have $V_{i,r,l}(X)\leq O(\Gamma) \ll o(1)$. For $(X,y)\sim\tcalZ_s$, when $y=i$, we have $V_{i,r,l}(X)\geq 0.9 \cdot C_1$ if \(v_{i,l}\) is the class-specific feature contained in \(X\) else $V_{i,r,l}(X)\le O(\rho) \ll o(1)$, when $y\neq i$ and $v_{i,l}\in\calP(X)$, we have $V_{i,r,l}(X)\leq O(\Gamma) \ll o(1)$.

    Recall that $\Pr(v_{i,l}\in\calP(X)|i\neq y)=\frac{s}{k}\ll o(1)$, we can derive that for \(\calZ_m'\)
\begin{equation}\label{eq:multi_error2}
    \begin{aligned}
    & -\bbE_{(X,y)\sim\calZ_m'}[\langle \nabla_{w_{i,r}} L(F^{(t)};X,y), v_{i,l}\rangle]\\
    \ge&\mathbb{E}_{(X,y) \sim \calZ_m'} [ 0.89 \cdot \mathbb{I}_{y=i} ( 1 - \logit_i(F^{(t)}, X) ) - 0.41 \cdot \frac{s}{k} \mathbb{I}_{y \neq i}  \logit_i(F^{(t)}, X)] \\
    \ge& \tilde{\Omega}(\frac{1}{k}) \mathbb{E}_{(X,y) \sim \calZ_m'} [ 1 - \logit_i(F^{(t)}, X)].
    \end{aligned}
\end{equation}
    Then for \(\tcalZ_m\), we have
\begin{equation}\label{eq:multi_error3}
    \begin{aligned}
    & -\bbE_{(X,y)\sim\tcalZ_m}[\langle \nabla_{w_{i,r}} L(F^{(t)};X,y), v_{i,l}\rangle]\\
    =& -\frac{1}{2}\bbE_{(X,y)\sim\tcalZ_m^l}[\langle \nabla_{w_{i,r}} L(F^{(t)};X,y), v_{i,l}\rangle] -\frac{1}{2}\bbE_{(X,y)\sim\tcalZ_m^{3-l}}[\langle \nabla_{w_{i,r}} L(F^{(t)};X,y), v_{i,l}\rangle]\\
    \ge&\frac{1}{2}\mathbb{E}_{(X,y) \sim \tcalZ_m^l} [ 0.89 \cdot \mathbb{I}_{y=i} ( 1 - \logit_i(F^{(t)}, X) ) - 0.41 \cdot \frac{s}{k} \mathbb{I}_{y \neq i}  \logit_i(F^{(t)}, X)] \\ 
    & + \frac{1}{2}\mathbb{E}_{(X,y) \sim \tcalZ_m^{3-l}} [ - \tilde{O}(\sigma_p P) \mathbb{I}_{y=i} ( 1 - \logit_i(F^{(t)}, X) ) - 0.41 \cdot \frac{s}{k} \mathbb{I}_{y \neq i}  \logit_i(F^{(t)}, X)] \\
    \ge&\mathbb{E}_{(X,y) \sim \tcalZ_m^l} [ 0.44 \cdot \mathbb{I}_{y=i} ( 1 - \logit_i(F^{(t)}, X) ) - 0.41 \cdot \frac{s}{k} \mathbb{I}_{y \neq i}  \logit_i(F^{(t)}, X)] \\ 
    \ge& \tilde{\Omega}(\frac{1}{k}) \mathbb{E}_{(X,y) \sim \tcalZ_m^l} [ 1 - \logit_i(F^{(t)}, X)].
    \end{aligned}
\end{equation}
    The second last step is due to \(\calA_1\) has equal probability to partial remove feature \(v_{i,l}\) or \(v_{i,3-l}\) from \((X,y)\sim \calZ_m\) to generate a sample in \(\tcalZ_m^l\) or \(\tcalZ_m^{3-l}\), so there is symmetry between \(\tcalZ_m^l\) and \(\tcalZ_m^{3-l}\).

    For similar reason, we can also derive that for \(\calZ_s'\) and \(\tcalZ_s\), we have
\begin{equation}\label{eq:multi_error4}
    \begin{aligned}
    & -\bbE_{(X,y)\sim\calZ_s'}[\langle \nabla_{w_{i,r}} L(F^{(t)};X,y), v_{i,l}\rangle]\\
    \ge& \frac{1}{2} \mathbb{E}_{(X,y) \sim {\calZ_s'}^l} [ 0.89 \cdot \mathbb{I}_{y=i} ( 1 - \logit_i(F^{(t)}, X) ) - O(\Gamma) \frac{s}{k} \mathbb{I}_{y \neq i}  \logit_i(F^{(t)}, X)] \\
    & + \frac{1}{2}\mathbb{E}_{(X,y) \sim {\calZ_s'}^{3-l}} [ - \tilde{O}(\sigma_p P) \mathbb{I}_{y=i} ( 1 - \logit_i(F^{(t)}, X) ) - O(\Gamma) \frac{s}{k} \mathbb{I}_{y \neq i}  \logit_i(F^{(t)}, X)] \\
    \ge& \mathbb{E}_{(X,y) \sim {\calZ_s'}^l} [ 0.44 \cdot \mathbb{I}_{y=i} ( 1 - \logit_i(F^{(t)}, X) ) - O(\Gamma) \frac{s}{k} \mathbb{I}_{y \neq i}  \logit_i(F^{(t)}, X)] \\ 
    \ge& \tilde{\Omega}(\frac{1}{k}) \mathbb{E}_{(X,y) \sim {\calZ_s'}^l} [ 1 - \logit_i(F^{(t)}, X)] \ge 0, 
    \end{aligned}
\end{equation}
and
\begin{equation}\label{eq:multi_error4_1}
    \begin{aligned}
    & -\bbE_{(X,y)\sim\tcalZ_s}[\langle \nabla_{w_{i,r}} L(F^{(t)};X,y), v_{i,l}\rangle]\\
    \ge& \frac{1}{2} \mathbb{E}_{(X,y) \sim \tcalZ_s^l} [ 0.89 C_1 \cdot \mathbb{I}_{y=i} ( 1 - \logit_i(F^{(t)}, X) ) - O(\Gamma) \frac{s}{k} \mathbb{I}_{y \neq i}  \logit_i(F^{(t)}, X)] \\
    & + \frac{1}{2}\mathbb{E}_{(X,y) \sim \tcalZ_s^{3-l}} [ - \tilde{O}(\sigma_p P) \mathbb{I}_{y=i} ( 1 - \logit_i(F^{(t)}, X) ) - O(\Gamma) \frac{s}{k} \mathbb{I}_{y \neq i}  \logit_i(F^{(t)}, X)] \\
    \ge&\mathbb{E}_{(X,y) \sim \tcalZ_s^l} [ \Omega(1) \cdot \mathbb{I}_{y=i} ( 1 - \logit_i(F^{(t)}, X) ) - O(\Gamma) \frac{s}{k} \mathbb{I}_{y \neq i}  \logit_i(F^{(t)}, X)] \\ 
    \ge& \tilde{\Omega}(\frac{1}{k}) \mathbb{E}_{(X,y) \sim {\tcalZ_s}^l} [ 1 - \logit_i(F^{(t)}, X)] \ge 0.
    \end{aligned}
\end{equation}
    Here we use \({\calZ_s'}^l\) (\(\tcalZ_s^l\)) to denote \((X,y) \sim \calZ_s'\) (\(\tcalZ_s\)) that has \(v_{y,l}\) as the only semantic feature, and \({\calZ_s'}^{3-l}\) (\(\tcalZ_s^{3-l}\)) to denote \((X,y) \sim \calZ_s'\) (\(\tcalZ_s\)) that has \(v_{y,3-l}\) as the only semantic feature. Now we take \eqref{eq:multi_error2}, \eqref{eq:multi_error3}, \eqref{eq:multi_error4}, \eqref{eq:multi_error4_1} into \eqref{eq:multi_error1}, we have 
\begin{equation}\label{eq:multi_error5}
    \begin{aligned}
        \langle w_{i,r}^{(t+1)}, v_{i,l}\rangle 
        \ge& \langle w_{i,r}^{(t)}, v_{i,l}\rangle + \frac{(N-N_s)(1-\pi_1)}{N} \tilde{\Omega}(\frac{\eta}{k}) \mathbb{E}_{(X,y) \sim \calZ_m'} [ 1 - \logit_i(F^{(t)}, X)] \\ 
        &+ \frac{(N-N_s)\pi_1}{N} \tilde{\Omega}(\frac{\eta}{k}) \mathbb{E}_{(X,y) \sim \tcalZ_m} [ 1 - \logit_i(F^{(t)}, X)].
    \end{aligned}
\end{equation}
Since we have \(\pi_1\ge \frac{1}{\polylogk}\), \(m=\polylogk\), and \(N \ge N_s \polyk\), when summing up all \(r\in [m]\), we have 
\begin{equation}\label{eq:multi_error6}
    \begin{aligned}
    \Phi_{i,l}^{(t+1)} \ge & \Phi_{i,l}^{(t)} + \tilde{\Omega}(\frac{\eta}{k}) \mathbb{E}_{(X,y) \sim \calZ_m'} [ 1 - \logit_i(F^{(t)}, X)] \\ 
    &+ \tilde{\Omega}(\frac{\eta\pi_1}{k}) \mathbb{E}_{(X,y) \sim \tcalZ_m} [ 1 - \logit_i(F^{(t)}, X)].
    \end{aligned}
\end{equation}
After telescoping \eqref{eq:multi_error6} from \(T_0\) to \(T\) and using \(\Phi_{i,l}\le \tilde{O}(1)\) from the Induction Hypothesis~\ref{hyp}(d), we have
\begin{align*}
\sum_{t=T_0}^{T} \bbE_{(X,y)\sim\calZ_{m}'}[1-\logit_y(F^{(t)}, X)] &\leq \tilde{O}\left(\frac{k}{\eta}\right),\\
\sum_{t=T_0}^{T} \bbE_{(X,y)\sim\tilde{\calZ}_{m}}[1-\logit_y(F^{(t)}, X)] &\leq \tilde{O}\left(\frac{k}{\eta\pi_1}\right). 
\end{align*}
\end{proof}


Next, we present a claim  that complements to the multi-view individual error bound established in Claim D.16 of \citet{allen-zhu2023towards}. The following claim states that when training error on \(\tcalZ_m\) is small enough, the model has high probability to correctly classify any individual single-view data.
\begin{claim}[single-view individual error]
\label{claim:individual}
When $\mathbb{E}_{(X,y) \sim \tcalZ_m} \left[1 - \logit_y \left( F^{(t)}, X \right) \right] \leq \frac{1}{k^4}$ is sufficiently small, we have for any \(i,j\in[k],l,l'\in [2]\),
\[
0.8 \Phi_{i,l}^{(t)} - (1+C_1)\Phi_{j,l'}^{(t)} \le -\Omega(\log(k)), \quad (0.2+C_1) \Phi_{i,l}^{(t)} \ge \Omega(\log(k)),
\]
and therefore for every \((X,y)\in \calZ_s',\tcalZ_s\), and every \((X,y) \sim \calD_s\),
\[
F_y^{(t)}(X)\geq \max_{j\neq y}F_j^{(t)}(X)+\Omega(\log k).
\]
\end{claim}

\begin{proof}
Denote by \(\tcalZ_*^l\) for the set of sample $(X,y) \in \tcalZ_m^l$ such that \(\sum_{p \in P_{v_{y,l}}(X)} z_p \leq 1 + \frac{1}{100 \log(k)}\) and \(\sum_{p \in P_{v_{y,3-l}}(X)} z_p \leq C_1 + \frac{1}{100 \log(k)}\). For a sample $(X,y) \in \tcalZ_*^l$, denote by $\mathcal{H}(X)$ as the set of all $i \in [k] \setminus \{y\}$ such that
\(\sum_{l \in [2]} \sum_{p \in P_{v_{i,l}}(X)} z_p \geq 0.8 - \frac{1}{100 \log(k)}\).

Now, suppose $1 - \logit_y( F^{(t)}, X) = \mathcal{E}(X)$, with $\min(1,\beta) \leq 2(1 - \frac{1}{1 + \beta})$, we have
\[
\min ( 1, \sum\nolimits_{i \in [k] \setminus \{y\}} e^{F_i^{(t)}(X) - F_y^{(t)}(X)} ) \leq 2 \mathcal{E}(X)
\]

By Claim \ref{claim:func}, \cref{claim:logit} and our definition of $\mathcal{H}(X)$, this implies that
\[
\min ( 1, \sum\nolimits_{i \in \mathcal{H}(X)} e^{0.4 \Phi_i^{(t)} - \Phi_{y,l}^{(t)} - C_1 \Phi_{y,3-l}^{(t)}}) \leq 4 \mathcal{E}(X).
\]
If we denote by $\psi = \mathbb{E}_{(X,y) \sim \tcalZ_m} [ 1 - \logit_y \left( F^{(t)}, X \right) ]$, then
\begin{align*}
&\mathbb{E}_{(X,y) \sim \tcalZ_m^l} \left[ \min ( 1, \sum\nolimits_{i \in \mathcal{H}(X)} e^{0.4 \Phi_i^{(t)} - \Phi_{y,l}^{(t)} - C_1 \Phi_{y,3-l}^{(t)}} ) \right] \leq O(\psi), \\
\implies &\mathbb{E}_{(X,y) \sim \tcalZ_m^l} \left[ \sum\nolimits_{i \in \mathcal{H}(X)} \min (\frac{1}{k}, e^{0.4 \Phi_i^{(t)} - \Phi_{y,l}^{(t)} - C_1 \Phi_{y,3-l}^{(t)}}) \right] \leq O(\psi).
\end{align*}

Notice that we can rewrite the LHS so that
\begin{align*}
&\mathbb{E}_{(X,y) \sim \tcalZ_m^l} \left[ \sum\nolimits_{j \in [k]} \bbI_{j=y} \sum\nolimits_{i \in [k]}  \bbI_{i \in \mathcal{H}(X)} \min ( \frac{1}{k}, e^{0.4 \Phi_i^{(t)} - \Phi_{y,l}^{(t)} - C_1 \Phi_{y,3-l}^{(t)}} ) \right] \leq O (\psi), \\
\implies &\sum\nolimits_{j \in [k]} \sum\nolimits_{i\in [k]} \bbI_{i\neq y} \mathbb{E}_{(X,y) \sim \tcalZ_m^l} \left[ \bbI_{j=y} \bbI_{i \in \mathcal{H}(X)}\right] \min ( \frac{1}{k}, e^{0.4 \Phi_i^{(t)} - \Phi_{j,l}^{(t)} - C_1 \Phi_{j,3-l}^{(t)}} )  \leq O(\psi).
\end{align*}
Note for every $i \neq j \in [k]$, the probability of generating a sample $(X,y) \in \tcalZ_*^l$ with $y = j$ and $i \in \mathcal{H}(X)$ is at least $\tilde{\Omega}(\frac{1}{k} \cdot \frac{s^2}{k^2})$. This implies
\[
 \sum\nolimits_{i \in [k] \setminus \{j\}} \min ( \frac{1}{k}, e^{0.4 \Phi_i^{(t)} - \Phi_{j,l}^{(t)} - C_1 \Phi_{j,3-l}^{(t)}} ) \leq \tilde{O} \left( \frac{k^3}{s^2} \psi \right).
\]
Then, with $1 - \frac{1}{1+\beta} \leq \min(1,\beta)$, we have for every $(X,y) \in \tcalZ_m^l$ ($l\in [2]$),

\begin{equation}\label{eq:single_error1}
\begin{aligned}
1 - \logit_y ( F^{(t)}, X ) \le & \min(1, \sum\nolimits_{i \in [k] \setminus \{y\}} 2 e^{0.4 \Phi_i^{(t)} - \Phi_{y,l}^{(t)} - C_1 \Phi_{y,3-l}^{(t)}}) \\ \leq &  k \cdot \sum\nolimits_{i \in [k] \setminus \{y\}} \min( \frac{1}{k}, e^{0.4 \Phi_i^{(t)} - \Phi_{y,l}^{(t)} - C_1 \Phi_{y,3-l}^{(t)}}) \leq \tilde{O} \left( \frac{k^4}{s^2} \psi \right).
\end{aligned}
\end{equation}

Thus, we can see that when \(\psi \le \frac{1}{k^4}\) is sufficiently small, we have for any \(i\in [k]\setminus \{y\} \text{ and } l\in [2]\),
\[
e^{0.4 \Phi_i^{(t)} - \Phi_{y,l}^{(t)} - C_1 \Phi_{y,3-l}^{(t)}} \le \frac{1}{k} \implies 0.4 \Phi_i^{(t)} - \Phi_{y,l}^{(t)} - C_1 \Phi_{y,3-l}^{(t)} \le -\Omega(\log(k)).
\]
By symmetry and non-negativity of \(\Phi_{i,l}^{(t)}\), we know for any \(i,j\in[k],l\in [2]\), we have:
\begin{equation}\label{eq:single_error2}
0.4 \Phi_{i,1}^{(t)} + 0.4 \Phi_{i,2}^{(t)} - \Phi_{j,l}^{(t)} - C_1 \Phi_{j,3-l}^{(t)} \le -\Omega(\log(k)).
\end{equation}
Since we have \(C_1 \in (0,0.4)\), this implies for any \(i,j\in[k],l,l'\in [2]\)
\begin{equation}\label{eq:single_error3}
0.8 \Phi_{i,l}^{(t)} - (1+C_1)\Phi_{j,l'}^{(t)} \le -\Omega(\log(k)) \implies (0.2+C_1) \Phi_{i,l}^{(t)} \ge \Omega(\log(k)).
\end{equation}
Since \eqref{eq:single_error3} holds for any \(i\in[k],l\in [2]\) at iteration \(t\) such that $\mathbb{E}_{(X,y) \sim \tcalZ_m} \left[1 - \logit_y \left( F^{(t)}, X \right) \right] \leq \frac{1}{k^4}$, for \((X,y) \sim \calZ_s'\) (suppose \(v_{y,l^*}\) is its only semantic feature), by Claim \ref{claim:func} we have
\begin{equation}\label{eq:single_error4}
\begin{aligned}
F_y^{(t)}(X) &\ge 1 \cdot \Phi_{y,l^*}^{(t)} - O\left(\frac{1}{\polylogk}\right) \ge \Omega(\log(k)), \\
F_j^{(t)}(X) &\le O(\Gamma) \cdot \Phi_{j,l}^{(t)} + O\left(\frac{1}{\polylogk}\right) \le O(1) \text{ for } j\neq y.
\end{aligned}
\end{equation}
Similarly, for \((X,y) \sim \tcalZ_s\) (suppose \(v_{y,l^*}\) as the only semantic feature), by \cref{claim:func} we have
\begin{equation}\label{eq:single_error5}
\begin{aligned}
F_y^{(t)}(X) &\ge C_1 \cdot \Phi_{y,l^*}^{(t)} - O\left(\frac{1}{\polylogk}\right) \ge \Omega(\log(k)), \\
F_j^{(t)}(X) &\le O(\Gamma) \cdot \Phi_{j,l}^{(t)} + O\left(\frac{1}{\polylogk}\right) \le O(1) \text{ for } j\neq y.
\end{aligned}
\end{equation}

Therefore, we have for \((X,y) \sim \calZ_s', \tcalZ_s\)
\begin{equation}\label{eq:single_error6}
\begin{aligned}
F_y^{(t)}(X)\geq \max_{j\neq y}F_j^{(t)}(X)+\Omega(\log k).
\end{aligned}
\end{equation}
\end{proof}
Let \(T_1\) be the first iteration that $\mathbb{E}_{(X,y) \sim \tcalZ_m} \left[1 - \logit_y \left( F^{(t)}, X \right) \right] \leq \frac{1}{k^4}$, then we know Eq. \eqref{eq:single_error3} always holds for \(t\ge T_1\), since the training objective \(L_1\) is \(O(1)\)-Lipschitz smooth and we are using full gradient descent, which means the objective value is monotonically non-increasing during the training process.

\subsection{Main Lemmas to Prove Induction Hypothesis~\ref{hyp}}
In this subsection, we show the lemmas that when combined together, shall prove the Induction Hypothesis~\ref{hyp} holds for every iteration \(t\). 

\subsubsection{Correlation Growth}
\begin{lemma}
\label{lemma:correlation}
    Suppose \cref{assp:para} holds and suppose Induction Hypothesis~\ref{hyp} holds for all iteration $< t$. Then, letting $\Phi_{i,l}^{(t)}=\sum_{r\in[m]}[\langle w_{i,r}^{(t)}, v_{i,l}\rangle]$, we have for every $i\in[k]$,
    \begin{align*}
        \Phi_{i,l} ^{(t)}\leq \tilde{O}(1).
    \end{align*}
\end{lemma}

\begin{proof}
Denote by \(\Phi^{(t)} := \max_{i \in [k], l \in [2]} \Phi_{i,l}^{(t)}\), suppose we are now at some iteration $t \geq T_1$. Let \( (i, l) = \argmax_{i \in [k], l \in [2]}\{\Phi_{i,l}^{(t)}\} \). Then by Claim \ref{claim:individual}, we know \( (1+C_1) \Phi_{j,l'} \ge 0.8 \Phi_{i,l} + \Omega(\log(k))\) for any \(j\in [k], l'\in [2]\).

By Claim \ref{claim:pos} and Claim \ref{claim:neg}, we have
\begin{align*}
    \langle w_{i,r}^{(t+1)}, v_{i,l}\rangle \leq \langle w_{i,r}^{(t)}, v_{i,l} \rangle + \eta \mathbb{E}_{(X,y) \sim \tcalZ_1} [& \mathbb{I}_{y=i} (V_{i,r,l}(X) + \mathcal{E}_1 + \mathcal{E}_3) (1 - \logit_i(F^{(t)}, X))\\ & - \mathbb{I}_{y \neq i} \tilde{O}(\sigma_p P) \logit_i(F^{(t)}, X)].
\end{align*}

If \(\Phi^{(t)} \ge \polylogk\), then we have
\begin{itemize}
\item for every $(X,y) \in \calZ_m'$ with $y = i$, recall from Claim \ref{claim:func} that \(F_j^{(t)}(X) = \sum_{l \in [2]} ( \Phi_{j,l}^{(t)} \times \mathbb{I}_{v_{j,l} \in \mathcal{V}(X)} ( \sum_{p \in P_{v_j,l}(X)} z_p ) ) \pm O\left(\frac{1}{\polylogk}\right)\).
By our choice of the distribution, this implies
\begin{itemize}
    \item \(
    F_j^{(t)}(X) \leq 0.8\Phi^{(t)} +  O\left(\frac{1}{\polylogk}\right) \text{ for } j \neq i, \text{ and}\)
    \item \( F_i^{(t)}(X) \geq \Phi^{(t)} - O\left(\frac{1}{\polylogk}\right) \).
\end{itemize}
In this case, since \(\Phi^{(t)} \ge \polylogk\), we have \(1 - \logit_y(F^{(t)}, X) \leq \frac{1}{k^{\Omega(\log k)}}\).

\item for every $(X,y) \in \tcalZ_m$ with $y = i$, again by Claim \ref{claim:func}, we have
\begin{itemize}
    \item \(
    F_j^{(t)}(X) \leq 0.8\Phi^{(t)} +  O\left(\frac{1}{\polylogk}\right) \text{ for } j \neq i, \text{ and}\)
    \item \( F_i^{(t)}(X) \geq C_1 \Phi^{(t)}_{i,l} + \Phi^{(t)}_{i,3-l} \geq 0.8 \Phi^{(t)} + \Omega(\log(k))\).
\end{itemize}
In this case, we have \(1 - \logit_y(F^{(t)}, X) \leq \frac{1}{\polyk}\).

\item for every $(X,y) \in \calZ_s'$ with $y = i$, again by Claim \ref{claim:func}, we have
\begin{itemize}
    \item \(
    F_j^{(t)}(X) \leq O(\Gamma)\Phi^{(t)} +  O\left(\frac{1}{\polylogk}\right) \text{ for } j \neq i, \text{ and}\)
    \item \( F_i^{(t)}(X) \geq  0.8 \Phi^{(t)} + \Omega(\log(k))\).
\end{itemize}
In this case, we have \(1 - \logit_y(F^{(t)}, X) \leq \frac{1}{k^{\Omega(\log k)}}\).

\item for every $(X,y) \in \tcalZ_s$ with $y = i$, again by Claim \ref{claim:func}, we have 
\begin{itemize}
    \item \(
    F_j^{(t)}(X) \leq O(\Gamma)\Phi^{(t)} +  O\left(\frac{1}{\polylogk}\right) \text{ for } j \neq i, \text{ and}\)
    \item \( F_i^{(t)}(X) \geq C_1 \Phi^{(t)} + \Omega(\log(k))\).
\end{itemize}
In this case, we have \(1 - \logit_y(F^{(t)}, X) \leq \frac{1}{k^{\Omega(\log k)}}\).
\end{itemize}
Thus, when summing up over all $r\in[m]$, we have 
\begin{align*}
    \Phi_{i,l} ^{(t+1)} \leq \Phi_{i,l} ^{(t)} + \frac{\eta m}{k^{\Omega(\log k)}} + \frac{\eta m \pi_1}{\polyk} + \frac{\eta m N_s}{k^{\Omega(\log k)}N}.
\end{align*}
Therefore, after $T$ iterations of training, we still have $\Phi_{i, l}^{(T)} \leq \tilde{O}(1)$.
\end{proof}

\subsubsection{Off-Diagonal Correlations are Small}
\begin{lemma}\label{lemma:off_diagonal}
    Suppose \cref{assp:para} holds and suppose Induction Hypothesis~\ref{hyp} holds for all iteration $< t$. Then,  
    \begin{align*}
        \forall i\in[k], \forall r\in[m], \forall j\in[k]\setminus \{i\}, \quad |\langle w_{i,r}^{(t)}, v_{j,l}\rangle | \leq \tilde{O}(\sigma_0). 
    \end{align*}
\end{lemma}
\begin{proof}
Let us denote by $R_i^{(t)} := \max_{r \in [m], j \in [k] \setminus \{i\}} |\langle w_{i,r}^{(t)}, v_{j,l} \rangle|$. By Claim \ref{claim:pos} and Claim \ref{claim:neg},
\begin{align*}
R_i^{(t+1)} \leq & R_i^{(t)} + \eta \mathbb{E}_{(X,y) \sim \tcalZ_1} \left[ \mathbb{I}_{y=i} ( \mathcal{E}_{2,i,r}(X) + \mathcal{E}_1 + \mathcal{E}_3 + \mathcal{E}_{4,j,l}(X)) (1 - \logit_i(F^{(t)}, X)) \right] \\
& + \eta \mathbb{E}_{(X,y) \sim \tcalZ_1} \left[ \mathbb{I}_{y \neq i} \left( \mathcal{E}_1 + \mathcal{E}_3 + \mathcal{E}_{4,j,l}(X) \right) \logit_i(F^{(t)}, X) \right].
\end{align*}
\textbf{Consider \(t\le T_0\).} During this initial stage, by Claim \ref{claim:logit} we know \(\logit_i(F^{(t)}, X) = O(\frac{1}{k})\) (\(\forall i \in [k]\)) for any \((X,y)\sim \tcalZ_1\). We also have \(\mathcal{E}_{2,i,r}(X) \leq \tilde{O}(\gamma(\Phi_i^{(t)})^{q-1}), \text{ and have } \mathcal{E}_{4,j,l}(X) \leq \tilde{O}(\sigma_0)^{q-1} \mathbb{I}_{v_{j,l} \in \mathcal{V}(X)}\) by definition. Recall when \(t\le T_0\), by Claim \ref{claim:growth}, we have \( \Phi_{i,l}^{(t+1)} = \Phi_{i,l}^{(t)}+\tilde{\Theta}\left(\frac{\eta}{k}\right)\trelu'(\Phi_{i,l}^{(t)}) \), so \(\sum_{t\le T_0} \eta (\Phi_i^{(t)})^{q-1} \le \tilde{O}(k)\). Also, \(\Pr(v_{i,l}\in\calP(X)|i\neq y)=\frac{s}{k}\). 
Therefore, for every \( t\le T_0 = \tilde{\Theta} (\frac{k}{\eta \sigma_0^{q-2}})\), we have 
\begin{align*}
R_i^{(t)} \leq R_i^{(0)} + \tilde{O}(\sigma_0) + \tilde{O}\left( \frac{\eta}{k} T_0 \right) \left( (\sigma_0^{q-1}) \gamma s + (\sigma_0 \gamma k)^{q-1} \gamma P + (\sigma_0)^{q-1} \frac{s}{k} \right) \leq \tilde{O}(\sigma_0).
\end{align*}

\textbf{Consider \(t > T_0\).} During this stage, 
we have the naive bound on  \(\mathcal{E}_{2,i,r}(X) \le O(\gamma)\), so again by Claim \ref{claim:pos} and Claim \ref{claim:neg}, we have
\begin{align*}
R_i^{(t+1)} \leq R_i^{(t)} + \frac{\eta}{k} \mathbb{E}_{(X,y) \sim \tcalZ_1} \big[ ( \gamma + (\sigma_0^{q-1}) \gamma s + (\sigma_0 \gamma k)^{q-1} \gamma P + (\sigma_0)^{q-1} \frac{s}{k}) (1 - \logit_i(F^{(t)}, X)) \big].
\end{align*}
Then, as we can obtain from the deduction of \cref{claim:multi_error} that 
\begin{equation} \label{eq:single_multi_error}
\sum_{t=T_0}^T\mathbb{E}_{(X,y) \sim \tcalZ_1} \left[1 - \logit_y ( F^{(t)}, X ) \right] \le \tilde{O}\left(\frac{k}{\eta}\right),
\end{equation}
we know \(R_i^{(t)} \leq \tilde{O}(\sigma_0)\) for any \(t>T_0\).
\end{proof}

\subsubsection{Noise Correlation is Small}
\begin{lemma}\label{lemma:noise}
    Suppose \cref{assp:para} holds and suppose Induction Hypothesis~\ref{hyp} holds for all iteration $< t$. For every $l\in[2]$,  for every $r\in[m]$, for every $(X,y)\in\tcalZ_1$ and $i\in[k]$:
    \begin{itemize}
        \item [(a)] For every $p\in\calP_{v_{i,l}}(X)$, we have: $|\langle w_{i,r}^{(t)}, \xi_p\rangle| \leq \tilde{o}(\sigma_0)$.
        \item [(b)] For every $p\in\calP(X)\setminus(\calP_{v_{i,1}}(X)\cup \calP_{v_{i,2}}(X))$, we have: $|\langle w_{i,r}^{(t)}, \xi_p\rangle| \leq \tilde{O}(\sigma_0)$.
        \item [(c)] For every $p\in[P]\setminus\calP(X)$, we have: $|\langle w_{i,r}^{(t)}, \xi_p\rangle| \leq \tilde{O}(\sigma_0\gamma k)$.
    \end{itemize}
\end{lemma}
\begin{proof}
Based on gradient calculation Fact \ref{app:grad_cal} and \(|\langle x'_{p'}, \xi_{p} \rangle| \leq \tilde{O}(\sigma_p) \leq o( \frac{1}{\sqrt{d}})
\) if \(X' \neq X\) or \(p' \neq p\), we have that for every \((X,y)\sim \tcalZ_1\) and \(p\in [P]\), if \(i=y\)
\begin{equation}
\langle w_{i,r}^{(t+1)}, \xi_p \rangle = \langle w_{i,r}^{(t)}, \xi_p \rangle + \tilde{\Theta} ( \frac{\eta}{N} ) \trelu' ( \langle w_{i,r}^{(t)}, x_p \rangle ) ( 1 - \logit_i(F^{(t)}, X) ) \pm \frac{\eta}{\sqrt{d}}.
\end{equation}

Else if $i \neq y$, 
\begin{equation}
\langle w_{i,r}^{(t+1)}, \xi_p \rangle = \langle w_{i,r}^{(t)}, \xi_p \rangle - \tilde{\Theta} ( \frac{\eta}{N} ) \trelu' ( \langle w_{i,r}^{(t)}, x_p \rangle ) \logit_i(F^{(t)}, X) \pm \frac{\eta}{\sqrt{d}}.
\end{equation}

Suppose that it satisfies that $|\langle w_{i,r}^{(t)}, x_p\rangle |\leq A$ for every $t<t_0$ where $t_0$ is any iteration $t_0\leq T$. Then we have \(\trelu' (\langle w_{i,r}^{(t)}, x_p\rangle) \le \tilde{O}(A^{q-1})\).
For \( t \le T_0 = \tilde{\Theta}(\frac{k}{\eta \sigma_0^{q-2}}) \), we have
\begin{align*}
    |\langle w_{i,r}^{(t)}, \xi_p \rangle| \leq  |\langle w_{i,r}^{(0)}, \xi_p \rangle| + \tilde{O}\left(\frac{T_0\eta A^{q-1}}{N}\right) + \frac{T_0\eta}{\sqrt{d}} \leq \tilde{O}\left(\frac{k A^{q-1}}{N\sigma_0^{q-2}}\right) + \frac{T_0\eta}{\sqrt{d}}.
\end{align*}
When \(t > T_0\), from Claim \ref{claim:multi_error} and Claim \ref{claim:individual} we know: 

For $(X, y) \in \calZ_m'$, we have
\begin{equation}\label{eq:noise1}
\begin{aligned}
& y = i \implies \sum_{t=T_0}^{T} [ 1 - \logit_y ( F^{(t)}, X )] \leq \tilde{O} \left( \frac{k^4}{s^2} \right) \sum_{t=T_0}^{T} \mathbb{E}_{(X,y) \sim \calZ_m'}[ 1 - \logit_y( F^{(t)}, X)] \leq \tilde{O} \left( \frac{k^5}{s^2 \eta} \right), \\
& y \neq i \implies \sum_{t=T_0}^{T} \logit_i ( F^{(t)}, X) \leq \sum_{t=T_0}^{T} [1 - \logit_y(F^{(t)}, X)] \leq \tilde{O} \left( \frac{k^5}{s^2 \eta} \right).
\end{aligned}
\end{equation}

Similarly, for $(X, y) \in \tcalZ_m$, by \eqref{eq:single_error1} we also have 

\begin{equation}\label{eq:noise2}
\begin{aligned}
& y = i \implies \sum_{t=T_0}^{T} [ 1 - \logit_y ( F^{(t)}, X )] \leq \tilde{O} \left( \frac{k^4}{\pi_1 s^2} \right) \sum_{t=T_0}^{T} \mathbb{E}_{(X,y) \sim \tcalZ_m}[ 1 - \logit_y( F^{(t)}, X)] \leq \tilde{O} \left( \frac{k^5}{\pi_1 s^2 \eta} \right), \\
& y \neq i \implies \sum_{t=T_0}^{T} \logit_i ( F^{(t)}, X) \leq \sum_{t=T_0}^{T} [1 - \logit_y(F^{(t)}, X)] \leq \tilde{O} \left( \frac{k^5}{\pi_1 s^2 \eta} \right).
\end{aligned}
\end{equation}

Combining \eqref{eq:noise1}, \eqref{eq:noise2} with the bound for $t \le T_0$, we have for \( 0 \le t \le T\),
\begin{equation}\label{eq:noise3}
|\langle w_{i,r}^{(t)}, \xi_p \rangle| \leq \tilde{O} \left( \frac{k A^{q-1}}{N \sigma_0^{q-2}} + \frac{k^5 A^{q-1}}{\pi_1 s^2 N} \right) + \frac{\eta T}{\sqrt{d}}.
\end{equation}
Now, we are ready to prove our main conclusion in this lemma.

When $p\in\calP_{v_{i,l}}(X)$, we have $|\langle w_{i,r}^{(t)}, x_p\rangle |\leq \tilde{O}(1)$ from Induction Hypothesis~\ref{hyp}. Then plugging in $A=\tilde{O}(1), N\geq \tilde{\Omega}\left(\frac{k}{\sigma_0^{q-1}}\right)$ and $N\geq \tilde{\Omega}\left(\frac{k^5}{\sigma_0}\right)$, we can obtain that $|\langle w_{i,r}^{(t)}, \xi_p \rangle| \leq \tilde{o}(\sigma_0)$. 

When $p\in\calP(X)\setminus(\calP_{v_{i,1}}(X)\cup \calP_{v_{i,2}}(X))$, we have $|\langle w_{i,r}^{(t)}, x_p\rangle |\leq \tilde{O}(\sigma_0)$ from the Induction Hypothesis~\ref{hyp}. Then plugging in $A=\tilde{O}(\sigma_0), N\geq \tilde{\Omega}(k^5)$, we can obtain that $|\langle w_{i,r}^{(t)}, \xi_p \rangle| \leq \tilde{O}(\sigma_0)$.

When $p\in[P]\setminus\calP(X)$, we have $|\langle w_{i,r}^{(t)}, x_p\rangle |\leq \tilde{O}(\sigma_0\gamma k)$ from Induction Hypothesis~\ref{hyp}. Then plugging in $A=\tilde{O}(\sigma_0\gamma k), N\geq \tilde{\Omega}(k^5)$, we can obtain that $|\langle w_{i,r}^{(t)}, \xi_p \rangle| \leq \tilde{O}(\sigma_0\gamma k)$.
\end{proof}

\subsubsection{Diagonal Correlations are Nearly Non-Negative}
\begin{lemma}
\label{lemma:diag_nonneg}
Suppose Parameter~\ref{assp:para} holds and suppose Induction Hypothesis~\ref{hyp} holds for all iteration $<t$. Then,  
\begin{align*}
    \forall i\in[k], \quad \forall r\in[m], \quad \forall l\in[2], \quad \langle w_{i,r}^{(t)}, v_{i,l}\rangle  \geq -\tilde{O}(\sigma_0). 
\end{align*}
\end{lemma}
\begin{proof}
Consider any iteration $t$ so that $\langle w_{i,r}^{(t)}, v_{i,l} \rangle \leq -\widetilde{\Omega}(\sigma_0)$. We start from this iteration to see how negative the next iterations can be. Without loss of generality, we consider the case when $\langle w_{i,r}^{(t')}, v_{i,l} \rangle \leq -\widetilde{\Omega}(\sigma_0)$ holds for every $t' \geq t$.
By Claim \ref{claim:pos} and Claim \ref{claim:neg},
\begin{align*}
\langle w_{i,r}^{(t+1)}, v_{i,l} \rangle \geq & \langle w_{i,r}^{(t)}, v_{i,l} \rangle + \eta \mathbb{E}_{(X,y) \sim \tcalZ_1} \Big[ \mathbb{I}_{y=i} ( V_{i,r,l}(X) - \tilde{O}(\sigma_p P)) (1 - \logit_i(F^{(t)}, X)) \\
& - \mathbb{I}_{y \neq i} \left( \mathcal{E}_1 + \mathcal{E}_3 + V_{i,r,l}(X) \right) \logit_i (F^{(t)}, X) \Big]
\end{align*}

Recall by Induction Hypothesis \ref{hyp}(a),
\[
V_{i,r,l}(X) = \sum\nolimits_{p \in P_{v_{i,l}}(X)} \trelu' ( \langle w_{i,r}^{(t)}, x_p \rangle) z_p =  \sum\nolimits_{p \in P_{v_{i,l}}(X)} \trelu' \left( \langle w_{i,r}, v_{i,l} \rangle z_p \pm \tilde{o}(\sigma_0) \right) z_p.
\]
Since we have assumed $\langle w_{i,r}^{(t)}, v_{i,l} \rangle \leq -\widetilde{\Omega}(\sigma_0)$, so \(V_{i,r,l}(X)=0\), and we have
\begin{equation}\label{eq:diag_nonneg1}
\begin{aligned}
\langle w_{i,r}^{(t+1)}, v_{i,l} \rangle \geq & \langle w_{i,r}^{(t)}, v_{i,l} \rangle - \eta \mathbb{E}_{(X,y) \sim \tcalZ_1} \Big[ \mathbb{I}_{y=i} \tilde{O}(\sigma_p P) ( 1 - \logit_i(F^{(t)}, X) ) \\
& + \mathbb{I}_{y \neq i} \left( \mathcal{E}_1 + \mathcal{E}_3 \right) \logit_i(F^{(t)}, X) \Big].
\end{aligned}
\end{equation}
We first consider every $t \leq T_0 = \tilde{\Theta} \left( \frac{k}{\eta \sigma_0^{q-2}} \right)$. Using Claim \ref{claim:logit} we have \(\logit_i(F^{(t)}, X) = O\left( \frac{1}{k} \right)\), which implies
\[
\langle w_{i,r}^{(t)}, v_{i,l} \rangle \geq -\tilde{O}(\sigma_0) - O\left(\frac{\eta T_0}{k}\right)(\mathcal{E}_1 + \mathcal{E}_3) \geq -\tilde{O}(\sigma_0).
\]
As for $t > T_0$, combining with Claim \ref{claim:multi_error} and the fact that \(\logit_i(F^{(t)}, X) \leq 1 - \logit_y(F^{(t)}, X)\) for \(i \neq y\), we have
\[
\langle w_{i,r}^{(t)}, v_{i,l} \rangle \geq \langle w_{i,r}^{(T_0)}, v_{i,l} \rangle - \tilde{O}(k) (\mathcal{E}_1 + \mathcal{E}_3) \ge \langle w_{i,r}^{(T_0)}, v_{i,l} \rangle - \tilde{O} \left(\sigma_0\right) \geq -\tilde{O}(\sigma_0).
\]
\end{proof}

\subsection{Proof of Induction Hypothesis \ref{hyp}} \label{appsec:hyp_proof}
Now we are ready to prove our Induction Induction Hypothesis \ref{hyp}, the proof is similar to Theorem D.2 in \citet{allen-zhu2023towards}.
\begin{lemma}\label{lemma:hyp}
Under Parameter Assumption \ref{assp:para}, for any $m=\polylogk$ and sufficiently small $\eta \leq \frac{1}{\polyk}$, our Induction Hypothesis \ref{hyp} holds for all iterations $t = 0, 1, \dots, T$.
\end{lemma}
\begin{proof}
At iteration $t$, we first calculate
\begin{align}
\forall p \in P_{v_{j,l}}(X):& \quad \langle w_{i,r}^{(t)}, x_p \rangle = \langle w_{i,r}^{(t)}, v_{j,l} \rangle z_p + \sum_{v' \in \mathcal{V}} \alpha_{p,v'} \langle w_{i,r}^{(t)}, v' \rangle + \langle w_{i,r}^{(t)}, \xi_p \rangle, \label{eq:hyp1} \\
\forall p \in [P] \setminus P(X):& \quad \langle w_{i,r}^{(t)}, x_p \rangle = \sum_{v' \in \mathcal{V}} \alpha_{p,v'} \langle w_{i,r}^{(t)}, v' \rangle + \langle w_{i,r}^{(t)}, \xi_p \rangle. \label{eq:hyp2}
\end{align}

It is easy to verify Induction Hypothesis \ref{hyp} holds at iteration $t = 0$ (by some high probability bounds on Gaussian random variables). Suppose Induction Hypothesis \ref{hyp} holds for all iterations $< t$. We have established several lemmas:
\begin{align}
\text{Lemma \ref{lemma:correlation}} &\implies \forall i \in [k], \forall r \in [m], \forall l \in [2]: \langle w_{i,r}^{(t)}, v_{i,l} \rangle \le \tilde{O}(1), \label{eq:hyp3} \\
\text{Lemma \ref{lemma:off_diagonal}} &\implies \forall i \in [k], \forall r \in [m], \forall j \in [k] \setminus \{i\}: |\langle w_{i,r}^{(t)}, v_{j,l} \rangle| \leq \tilde{O}(\sigma_0), \label{eq:hyp4} \\
\text{Lemma \ref{lemma:diag_nonneg}} &\implies \forall i \in [k], \forall r \in [m], \forall l \in [2]: \langle w_{i,r}^{(t)}, v_{i,l} \rangle \geq -\tilde{O}(\sigma_0). \label{eq:hyp5}
\end{align}
\begin{itemize}
\item To prove \ref{hyp}(a), it suffices to plug \eqref{eq:hyp4}, \eqref{eq:hyp5} into \eqref{eq:hyp1}, use $\alpha_{p,v'} \in [0, \gamma]$, use $|\mathcal{V}| = 2k$, and use $|\langle w_{i,r}^{(t)}, \xi_p \rangle| \leq \tilde{o}(\sigma_0)$ from Lemma \ref{lemma:noise}.
\item To prove \ref{hyp}(b), it suffices to plug \eqref{eq:hyp3}, \eqref{eq:hyp4} into \eqref{eq:hyp1}, use $\alpha_{p,v'} \in [0, \gamma]$, use $|\mathcal{V}| = 2k$, and use $|\langle w_{i,r}^{(t)}, \xi_p \rangle| \leq \tilde{O}(\sigma_0)$ from Lemma \ref{lemma:noise}.
\item To prove \ref{hyp}(c), it suffices to plug \eqref{eq:hyp3}, \eqref{eq:hyp4} into \eqref{eq:hyp2}, use $\alpha_{p,v'} \in [0, \gamma]$, use $|\mathcal{V}| = 2k$, and use $|\langle w_{i,r}^{(t)}, \xi_p \rangle| \leq \tilde{O}(\sigma_0 \gamma k)$ from Lemma \ref{lemma:noise}.
\item To prove \ref{hyp}(d), it suffices to note that \eqref{eq:hyp3} implies $\Phi_{i,l}^{(t)} \leq \tilde{O}(1)$, and note that Claim \ref{claim:growth} implies $\Phi_{i,l}^{(t)} \geq \Omega(\Phi_{i,l}^{(0)}) \geq \tilde{\Omega}(\sigma_0)$.
\item To prove \ref{hyp}(e), it suffices to invoke \eqref{eq:hyp5}.
\end{itemize}
\end{proof}

\subsection{Proof of \cref{main:thm_a1}}
As Lemma \ref{lemma:hyp} implies the Induction Hypothesis \ref{hyp} holds for every $t \leq T$, we have according to Claim \ref{claim:multi_error} that
\begin{align*}
\sum\nolimits_{t=T_0}^{T} \bbE_{(X,y)\sim\calZ_{m}'}[1-\logit_y(F^{(t)}, X)] &\leq \tilde{O}\left(\frac{k}{\eta}\right),\\
\sum\nolimits_{t=T_0}^{T} \bbE_{(X,y)\sim\tilde{\calZ}_{m}}[1-\logit_y(F^{(t)}, X)] &\leq \tilde{O}\left(\frac{k}{\eta\pi_1}\right).
\end{align*}

Then, we know when $T \geq \frac{\polyk}{\eta \pi_1}$,
\[
\frac{1}{T} \sum\nolimits_{t=T_0}^{T} \bbE_{(X,y)\sim\tilde{\calZ}_{m}} [1-\logit_y(F^{(t)}, X)] \leq \frac{1}{\polyk}.
\]

Moreover, since we are using full gradient descent and the objective function is $O(1)$-Lipschitz continuous, the objective value decreases monotonically. Specifically, this implies that
\[
\mathbb{E}_{(X,y) \sim \tcalZ_m} [1 - \logit_y(F^{(T)}, X)] \leq \frac{1}{\polyk}.
\]
for the last iteration $T$. Then recall from Claim \ref{claim:individual}, we have 
\[
0.8 \Phi_{i,l}^{(T)} - (1+C_1) \Phi_{j,l'}^{(T)} \le -\Omega(\log(k)), \quad (0.2+C_1) \Phi_{i,l}^{(T)} \ge \Omega(\log(k)),
\]
for any \(i,j \in [k]\) and \(l,l' \in [2]\). This combined with the function approximation Claim \ref{claim:func} shows that with high probability $F_y^{(T)}(X) \geq \max_{j \neq y} F_j^{(T)}(X) + \Omega(\log k)$ for every $(X, y) \in \calD_m, \calD_s$, which implies that the test accuracy on both multi-view data and single-view data is perfect.

\section{Proof for SL with $\calA_2$}
\label{appsec:proof_a2}
Here, we provide the proof for SL with data augmentation \(\mathcal{A}_2\). Since the proof follows the same framework as the proof for vanilla SL in \citet{allen-zhu2023towards} and adheres to the same Induction Hypothesis C.3 of \citet{allen-zhu2023towards}, we present only the key differences to avoid redundancy.

\subsection{Key Claims}

Recall that for SL with data augmentation \(\mathcal{A}_2\), we use \(\tilde{\mathcal{Z}}_2 = \mathcal{Z}_m' \cup \tilde{\mathcal{Z}}_m \cup \mathcal{Z}_s' \cup \tilde{\mathcal{Z}}_s\) to denote the training dataset after data augmentation. Here, \(\mathcal{Z}_m'\) and \(\tilde{\mathcal{Z}}_m\) originate from the multi-view training dataset \(\mathcal{Z}_m\) and represent samples with only noise mixing or feature mixing effects, respectively, as described in \cref{assum:a2}. Similarly, \(\mathcal{Z}_s'\) and \(\tilde{\mathcal{Z}}_s\) are derived from the single-view training dataset \(\mathcal{Z}_s\) and contain samples with only noise mixing or feature mixing effects, respectively.

The first key difference in the proof lies in the convergence of the error on multi-view data from \(T_0\) until the end of training.

\begin{claim}[multi-view error till the end]
\label{claim:a2_multi_error}
    Suppose that Induction Hypothesis C.3 in \citet{allen-zhu2023towards} holds for every iteration $t\le T$, then 
    \begin{align*}
         \sum_{t=T_0}^{T} \bbE_{(X,y)\sim\calZ_{m}'}[1-\logit_y(F^{(t)}, X)] &\leq \tilde{O}\left(\frac{k}{\eta}\right),\\
         \sum_{t=T_0}^{T} \bbE_{(X,y)\sim\tilde{\calZ}_{m}}[1-\logit_y(F^{(t)}, X)] &\leq \tilde{O}\left(\frac{k}{\eta\pi_2}\right).
    \end{align*}
\end{claim}
\begin{proof}
Fix \(i\in[k]\) and \(l\in [2]\), we know for any \(w_{i,r}\) (\(r\in[m]\)), we have
\begin{equation}\label{eq:a2_multi_error1}
    \begin{aligned}
        \langle w_{i,r}^{(t+1)}, v_{i,l}\rangle 
        =& \langle w_{i,r}^{(t)}, v_{i,l}\rangle - \eta \bbE_{(X,y)\sim\tcalZ_2}\big[\langle \nabla_{w_{i,r}} L(F^{(t)};X,y) , v_{i,l}\rangle\big] \\
        \\=& \langle w_{i,r}^{(t)}, v_{i,l}\rangle - \frac{\eta(N-N_s)(1-\pi_2)}{N} \bbE_{(X,y)\sim\calZ_m'}[\langle \nabla_{w_{i,r}} L(F^{(t)};X,y), v_{i,l}\rangle]\\ 
        &- \frac{\eta(N-N_s)\pi_2}{N} \bbE_{(X,y)\sim\tcalZ_m}[\langle \nabla_{w_{i,r}} L(F^{(t)};X,y), v_{i,l}\rangle] \\
        &- \frac{\eta N_s(1-\pi_2)}{N} \bbE_{(X,y)\sim\calZ_s'}[\langle \nabla_{w_{i,r}} L(F^{(t)};X,y), v_{i,l}\rangle] \\
        &- \frac{\eta N_s\pi_2}{N} \bbE_{(X,y)\sim\tcalZ_s}[\langle \nabla_{w_{i,r}} L(F^{(t)};X,y), v_{i,l}\rangle].
    \end{aligned}
\end{equation}
    Take \(r = \argmax_{r\in [m]} \{\langle w_{i,r}^{(t)}, v_{i,l}\rangle\}\), then by \(m=\polylogk\) we know \(\langle w_{i,r}^{(t)}, v_{i,l}\rangle \ge \tilde{\Omega}(\Phi_{i,l}^{(t)})=\tilde{\Omega}(1)\) for \(t\ge T_0\). Same as in the proof in \cref{claim:multi_error}, we know when \(t\ge T_0\), \(\langle w_{i,r}^{(t)}, v_{i,l}\rangle \ge \tilde{\Omega}(1) \gg \varrho\), and \(|\calP_{v_{i,l}}(X)|\leq O(1)\), for most of \( p\in \calP_{v_{i,l}}(X) \) must be already in the linear regime of \(\trelu\), which means we have
    \begin{align*}
        0.9 \sum\nolimits_{p\in\calP_{v_{i,l}}(X) }z_p \leq V_{i,r,l}(X) \leq \sum\nolimits_{p\in\calP_{v_{i,l}}(X) }z_p.
    \end{align*}
    
    Thus, for $(X,y)\sim\calZ_m'$, when $y=i$, we have $V_{i,r,l}(X)\geq 0.9$, when $y\neq i$ and $v_{i,l}\in\calV(X)$, we have $V_{i,r,l}(X)\leq 0.4$. For $(X,y)\sim\tcalZ_m$, when $y=i$, we have $V_{i,r,l}(X)\geq 0.9 \cdot (1-C_2)$, when $y\neq i$ and $v_{i,l} \in \calV(X)$, we have $V_{i,r,l}(X)\leq 0.4+C_3$. For $(X,y)\sim\calZ_s'$, when $y=i$, we have $V_{i,r,l}(X)\ge 0.9$ if \(v_{i,l}\) is the class-specific feature contained in \(X\) else $V_{i,r,l}(X)\le O(\rho) \ll o(1)$, when $y\neq i$ and $v_{i,l}\in\calP(X)$, we have $V_{i,r,l}(X)\leq O(\Gamma) \ll o(1)$. For $(X,y)\sim\tcalZ_s$, when $y=i$, we have $V_{i,r,l}(X)\geq 0.9 \cdot (1-C_2)$ if \(v_{i,l}\) is the class-specific feature contained in \(X\) else $V_{i,r,l}(X)\le O(\rho) \ll o(1)$, when $y\neq i$ and $v_{i,l}\in\calP(X)$, we have $V_{i,r,l}(X)\leq C_3$.

    Recall that $\Pr(v_{i,l}\in\calP(X)|i\neq y)=\frac{s}{k}\ll o(1)$, using \cref{claim:pos} and \cref{claim:neg}, we can derive that for \(\calZ_m'\)
\begin{equation}\label{eq:a2_multi_error2}
    \begin{aligned}
    & -\bbE_{(X,y)\sim\calZ_m'}[\langle \nabla_{w_{i,r}} L(F^{(t)};X,y), v_{i,l}\rangle] \\ =& \bbE_{(X,y)\sim\calZ_m'}\big[\bbI_{y=i}(1-\logit_i(F^{(t)},X)) \sum\nolimits_{p\in[P]} \trelu'(\langle w_{i,r}, x_p\rangle )\langle x_p, v_{i,l}\rangle \\ & \qquad \qquad \quad - \bbI_{y\neq i} \logit_i(F^{(t)},X) \sum\nolimits_{p\in[P]} \trelu'(\langle w_{i,r}, x_p\rangle ) \langle x_p, v_{i,l}\rangle\big] \\
    \ge&\mathbb{E}_{(X,y) \sim \calZ^*} [ \mathbb{I}_{y=i} ( V_{i,r,l}(X) - \tilde{O}(\sigma_p P) ) ( 1 - \logit_i(F^{(t)}, X) )\\ &\qquad \qquad \quad - \mathbb{I}_{y \neq i} \left( \mathcal{E}_1 + \mathcal{E}_3 + V_{i,r,l}(X) \right) \logit_i(F^{(t)}, X)] \\
    \ge&\mathbb{E}_{(X,y) \sim \calZ_m'} [ 0.89 \cdot \mathbb{I}_{y=i} ( 1 - \logit_i(F^{(t)}, X) ) - 0.41 \cdot \frac{s}{k} \mathbb{I}_{y \neq i}  \logit_i(F^{(t)}, X)] \\
    \ge& \tilde{\Omega}(\frac{1}{k}) \mathbb{E}_{(X,y) \sim \calZ_m'} [ 1 - \logit_i(F^{(t)}, X)].
    \end{aligned}
\end{equation}
    Similarly, for \(\tcalZ_m\), we have
\begin{equation}\label{eq:a2_multi_error3}
    \begin{aligned}
    & -\bbE_{(X,y)\sim\tcalZ_m}[\langle \nabla_{w_{i,r}} L(F^{(t)};X,y), v_{i,l}\rangle]\\
    \ge&\mathbb{E}_{(X,y) \sim \calZ_m'} [ (0.89-C_2) \cdot \mathbb{I}_{y=i} ( 1 - \logit_i(F^{(t)}, X) ) - (0.41+C_3) \cdot \frac{s}{k} \mathbb{I}_{y \neq i}  \logit_i(F^{(t)}, X)] \\
    \ge& \tilde{\Omega}(\frac{1}{k}) \mathbb{E}_{(X,y) \sim \calZ_m'} [ 1 - \logit_i(F^{(t)}, X)].
    \end{aligned}
\end{equation}

    Then, based on the fact that single-view data has equal probability to contain either of the two semantic features, we can derive that for \(\calZ_s'\) and \(\tcalZ_s\), we have
\begin{equation}\label{eq:a2_multi_error4}
    \begin{aligned}
    & -\bbE_{(X,y)\sim\calZ_s'}[\langle \nabla_{w_{i,r}} L(F^{(t)};X,y), v_{i,l}\rangle]\\
    \ge& \frac{1}{2} \mathbb{E}_{(X,y) \sim {\calZ_s'}^l} [ 0.89 \cdot \mathbb{I}_{y=i} ( 1 - \logit_i(F^{(t)}, X) ) - O(\Gamma) \frac{s}{k} \mathbb{I}_{y \neq i}  \logit_i(F^{(t)}, X)] \\
    & + \frac{1}{2}\mathbb{E}_{(X,y) \sim {\calZ_s'}^{3-l}} [ - \tilde{O}(\sigma_p P) \mathbb{I}_{y=i} ( 1 - \logit_i(F^{(t)}, X) ) - O(\Gamma) \frac{s}{k} \mathbb{I}_{y \neq i}  \logit_i(F^{(t)}, X)] \\
    \ge& \mathbb{E}_{(X,y) \sim {\calZ_s'}^l} [ 0.44 \cdot \mathbb{I}_{y=i} ( 1 - \logit_i(F^{(t)}, X) ) - O(\Gamma) \frac{s}{k} \mathbb{I}_{y \neq i}  \logit_i(F^{(t)}, X)] \\ 
    \ge& \tilde{\Omega}(\frac{1}{k}) \mathbb{E}_{(X,y) \sim {\calZ_s'}^l} [ 1 - \logit_i(F^{(t)}, X)] \ge 0, 
    \end{aligned}
\end{equation}
and
\begin{equation}\label{eq:a2_multi_error4_1}
    \begin{aligned}
    & -\bbE_{(X,y)\sim\tcalZ_s}[\langle \nabla_{w_{i,r}} L(F^{(t)};X,y), v_{i,l}\rangle]\\
    \ge& \frac{1}{2} \mathbb{E}_{(X,y) \sim \tcalZ_s^l} [ (0.89-C_2) \cdot \mathbb{I}_{y=i} ( 1 - \logit_i(F^{(t)}, X) ) - C_3 \frac{s}{k} \mathbb{I}_{y \neq i}  \logit_i(F^{(t)}, X)] \\
    & + \frac{1}{2}\mathbb{E}_{(X,y) \sim \tcalZ_s^{3-l}} [ - \tilde{O}(\sigma_p P) \mathbb{I}_{y=i} ( 1 - \logit_i(F^{(t)}, X) ) - C_3 \frac{s}{k} \mathbb{I}_{y \neq i}  \logit_i(F^{(t)}, X)] \\
    \ge&\mathbb{E}_{(X,y) \sim \tcalZ_s^l} [ \Omega(1) \cdot \mathbb{I}_{y=i} ( 1 - \logit_i(F^{(t)}, X) ) - C_3 \frac{s}{k} \mathbb{I}_{y \neq i}  \logit_i(F^{(t)}, X)] \\ 
    \ge& \tilde{\Omega}(\frac{1}{k}) \mathbb{E}_{(X,y) \sim {\tcalZ_s}^l} [ 1 - \logit_i(F^{(t)}, X)] \ge 0.
    \end{aligned}
\end{equation}
    Again, here we use \({\calZ_s'}^l\) (\(\tcalZ_s^l\)) to denote \((X,y) \sim \calZ_s'\) (\(\tcalZ_s\)) that has \(v_{y,l}\) as the only semantic feature, and \({\calZ_s'}^{3-l}\) (\(\tcalZ_s^{3-l}\)) to denote \((X,y) \sim \calZ_s'\) (\(\tcalZ_s\)) that has \(v_{y,3-l}\) as the only semantic feature. Now we take \eqref{eq:a2_multi_error2}, \eqref{eq:a2_multi_error3}, \eqref{eq:a2_multi_error4}, \eqref{eq:a2_multi_error4_1} into \eqref{eq:a2_multi_error1}, we have 
\begin{equation}\label{eq:a2_multi_error5}
    \begin{aligned}
        \langle w_{i,r}^{(t+1)}, v_{i,l}\rangle 
        \ge& \langle w_{i,r}^{(t)}, v_{i,l}\rangle + \frac{(N-N_s)(1-\pi_2)}{N} \tilde{\Omega}(\frac{\eta}{k}) \mathbb{E}_{(X,y) \sim \calZ_m'} [ 1 - \logit_i(F^{(t)}, X)] \\ 
        &+ \frac{(N-N_s)\pi_2}{N} \tilde{\Omega}(\frac{\eta}{k}) \mathbb{E}_{(X,y) \sim \tcalZ_m} [ 1 - \logit_i(F^{(t)}, X)].
    \end{aligned}
\end{equation}
Since we have \(\pi_2\ge \frac{1}{\polylogk}\), \(m=\polylogk\), and \(N \ge N_s \polyk\), when summing up all \(r\in [m]\), we have 
\begin{equation}\label{eq:a2_multi_error6}
    \begin{aligned}
    \Phi_{i,l}^{(t+1)} \ge & \Phi_{i,l}^{(t)} + \tilde{\Omega}(\frac{\eta}{k}) \mathbb{E}_{(X,y) \sim \calZ_m'} [ 1 - \logit_i(F^{(t)}, X)] \\ 
    &+ \tilde{\Omega}(\frac{\eta\pi_2}{k}) \mathbb{E}_{(X,y) \sim \tcalZ_m} [ 1 - \logit_i(F^{(t)}, X)].
    \end{aligned}
\end{equation}
After telescoping \eqref{eq:a2_multi_error6} from \(T_0\) to \(T\) and using \(\Phi_{i,l}\le \tilde{O}(1)\) from the Induction Hypothesis C.3(g) in \citet{allen-zhu2023towards}, we have
\begin{align*}
\sum_{t=T_0}^{T} \bbE_{(X,y)\sim\calZ_{m}'}[1-\logit_y(F^{(t)}, X)] &\leq \tilde{O}\left(\frac{k}{\eta}\right),\\
\sum_{t=T_0}^{T} \bbE_{(X,y)\sim\tilde{\calZ}_{m}}[1-\logit_y(F^{(t)}, X)] &\leq \tilde{O}\left(\frac{k}{\eta\pi_2}\right). 
\end{align*}
\end{proof}

The second core difference of the proof lies in the claim of multi-view individual error, which states that when training error on \(\tcalZ_m\) is small enough, the model has high probability to correctly classify any individual mutli-view data.

\begin{claim}[multi-view individual error]
\label{claim:a2_individual}
When $\mathbb{E}_{(X,y) \sim \tcalZ_m} \left[1 - \logit_y \left( F^{(t)}, X \right) \right] \leq \frac{1}{k^4}$ is sufficiently small, we have for any \(i,j\in[k]\),
\[
(0.4+C_3) \Phi_{i}^{(t)} - (1-C_2)\Phi_{j}^{(t)} \le -\Omega(\log(k)), \quad (0.6-C_2-C_3) \Phi_{i}^{(t)} \ge \Omega(\log(k)),
\]
and therefore for every \((X,y)\in \calZ_m',\tcalZ_m\) (and every \((X,y)\in \calD_m\) w.p. \(1 - e^{-\Omega(\log^2(k))}\)), 
\[
F_y^{(t)}(X)\geq \max_{j\neq y}F_j^{(t)}(X)+\Omega(\log k).
\]
\end{claim}

\begin{proof}
Denote by \(\tcalZ_*\) for the set of training samples $(X,y) \in \tcalZ_m$ such that \(\sum_{l \in [2]} \sum_{p \in P_{v_{y,l}}(X)} z_p \leq 2 - 2C_2+ \frac{1}{100 \log(k)}\). For a sample \((X,y)\in \tcalZ_*\), denote by \(\calH(X)\) as the set of all \(i \in [k]\setminus \{y\}\) such that \(\sum_{l \in [2]} \sum_{p \in P_{v_{i,l}}(X)} z_p \geq 0.8 + 2C_3 - \frac{1}{100 \log(k)}\).

Now, suppose $1 - \logit_y( F^{(t)}, X) = \mathcal{E}(X)$, with $\min(1,\beta) \leq 2(1 - \frac{1}{1 + \beta})$, we have
\[
\min ( 1, \sum\nolimits_{i \in [k] \setminus \{y\}} e^{F_i^{(t)}(X) - F_y^{(t)}(X)} ) \leq 2 \mathcal{E}(X).
\]

By Claim \ref{claim:func}, \cref{claim:logit} and our definition of $\mathcal{H}(X)$, this implies that
\[
\min ( 1, \sum\nolimits_{i \in \mathcal{H}(X)} e^{(0.4+C_3) \Phi_i^{(t)} - (1-C_2)\Phi_{y}^{(t)}}) \leq 4 \mathcal{E}(X).
\]
If we denote by $\psi = \mathbb{E}_{(X,y) \sim \tcalZ_m} [ 1 - \logit_y \left( F^{(t)}, X \right) ]$, then
\begin{align*}
&\mathbb{E}_{(X,y) \sim \tcalZ_m} \left[ \min ( 1, \sum\nolimits_{i \in \mathcal{H}(X)} e^{(0.4+C_3) \Phi_i^{(t)} - (1-C_2)\Phi_{y}^{(t)}} ) \right] \leq O(\psi), \\
\implies &\mathbb{E}_{(X,y) \sim \tcalZ_m} \left[ \sum\nolimits_{i \in \mathcal{H}(X)} \min (\frac{1}{k}, e^{(0.4+C_3) \Phi_i^{(t)} - (1-C_2)\Phi_{y}^{(t)}}) \right] \leq O(\psi).
\end{align*}

Notice that we can rewrite the LHS so that
\begin{align*}
&\mathbb{E}_{(X,y) \sim \tcalZ_m} \left[ \sum\nolimits_{j \in [k]} \bbI_{j=y} \sum\nolimits_{i \in [k]}  \bbI_{i \in \mathcal{H}(X)} \min ( \frac{1}{k}, e^{(0.4+C_3) \Phi_i^{(t)} - (1-C_2)\Phi_j^{(t)}} ) \right] \leq O (\psi), \\
\implies &\sum\nolimits_{j \in [k]} \sum\nolimits_{i\in [k]} \bbI_{i\neq y} \mathbb{E}_{(X,y) \sim \tcalZ_m} \left[ \bbI_{j=y} \bbI_{i \in \mathcal{H}(X)}\right] \min ( \frac{1}{k}, e^{(0.4+C_3) \Phi_i^{(t)} - (1-C_2)\Phi_j^{(t)}} )  \leq O(\psi).
\end{align*}
Note for every $i \neq j \in [k]$, the probability of generating a sample $(X,y) \in \tcalZ_*$ with $y = j$ and $i \in \mathcal{H}(X)$ is at least $\tilde{\Omega}(\frac{1}{k} \cdot \frac{s^2}{k^2})$. This implies
\[
 \sum\nolimits_{i \in [k] \setminus \{j\}} \min ( \frac{1}{k}, e^{(0.4+C_3) \Phi_i^{(t)} - (1-C_2)\Phi_j^{(t)}} ) \leq \tilde{O} \left( \frac{k^3}{s^2} \psi \right).
\]
Then, with $1 - \frac{1}{1+\beta} \leq \min(1,\beta)$, we have for every $(X,y) \in \tcalZ_m$,

\begin{equation}\label{eq:a2_single_error1}
\begin{aligned}
1 - \logit_y ( F^{(t)}, X ) \le & \min(1, \sum\nolimits_{i \in [k] \setminus \{y\}} 2 e^{(0.4+C_3) \Phi_i^{(t)} - (1-C_2)\Phi_y^{(t)}}) \\ \leq &  k \cdot \sum\nolimits_{i \in [k] \setminus \{y\}} \min( \frac{1}{k}, e^{(0.4+C_3) \Phi_i^{(t)} - (1-C_2)\Phi_y^{(t)}}) \leq \tilde{O} \left( \frac{k^4}{s^2} \psi \right).
\end{aligned}
\end{equation}

Thus, we can see that when \(\psi \le \frac{1}{k^4}\) is sufficiently small, we have for any \(i\in [k]\setminus \{y\}\)
\[
e^{(0.4+C_3) \Phi_i^{(t)} - (1-C_2)\Phi_y^{(t)}} \le \frac{1}{k} \implies (0.4+C_3) \Phi_i^{(t)} - (1-C_2)\Phi_y^{(t)} \le -\Omega(\log(k)).
\]
By symmetry and non-negativity of \(\Phi_{i}^{(t)}\), we know for any \(i,j\in[k]\), we have:
\begin{equation}\label{eq:a2_single_error2}
(0.4+C_3) \Phi_i^{(t)} - (1-C_2)\Phi_j^{(t)}  \le -\Omega(\log(k)), \quad (0.6-C_2-C_3) \Phi_{i}^{(t)} \ge \Omega(\log(k)).
\end{equation}
Since \eqref{eq:a2_single_error2} holds for any \(i\in[k]\) at iteration \(t\) such that $\mathbb{E}_{(X,y) \sim \tcalZ_m} \left[1 - \logit_y \left( F^{(t)}, X \right) \right] \leq \frac{1}{k^4}$, for \((X,y) \sim \calZ_m'\), by Claim \ref{claim:func} we have
\begin{equation}\label{eq:a2_single_error4}
\begin{aligned}
F_y^{(t)}(X) &\ge 1 \cdot \Phi_{y}^{(t)} - O\left(\frac{1}{\polylogk}\right) , \\
F_j^{(t)}(X) &\le 0.4 \cdot \Phi_{j}^{(t)} + O\left(\frac{1}{\polylogk}\right) \text{ for } j\neq y.
\end{aligned}
\end{equation}
Similarly, for \((X,y) \sim \tcalZ_m\), by \cref{claim:func} we have
\begin{equation}\label{eq:a2_single_error5}
\begin{aligned}
F_y^{(t)}(X) &\ge (1-C_2) \cdot \Phi_{y}^{(t)} - O\left(\frac{1}{\polylogk}\right) , \\
F_j^{(t)}(X) &\le (0.4+C_3) \cdot \Phi_{j}^{(t)} + O\left(\frac{1}{\polylogk}\right) \text{ for any } j\neq y.
\end{aligned}
\end{equation}
Therefore, we have for \((X,y) \sim \calZ_m', \tcalZ_m\)
\begin{equation}\label{eq:a2_single_error6}
\begin{aligned}
F_y^{(t)}(X)\geq \max_{j\neq y}F_j^{(t)}(X)+\Omega(\log k).
\end{aligned}
\end{equation}
\end{proof}

\subsection{Proof of \cref{main:thm_a2}}
First recall by definition and Induction Hypothesis C.3 in \citet{allen-zhu2023towards}, the prediction function is
\begin{equation}\label{appeq:a2_f}
\begin{aligned}
F_i^{(T)}(X) =& \sum_{r\in [m]}\sum_{p\in [P]}\trelu(\langle w_{i,r}^{(T)}, x_p \rangle)\\
=& \sum_{r \in [m]} \sum_{p \in \mathcal{P}_{v_{i,1}}(X) \cup \mathcal{P}_{v_{i,2}}(X)} \trelu(\langle w_{i,r}^{(T)}, x_p \rangle) + \sum_{r \in [m]} \sum_{p \in \mathcal{P}(X) \setminus (\mathcal{P}_{v_{i,1}}(X) \cup \mathcal{P}_{v_{i,2}}(X))} \trelu(\langle w_{i,r}^{(T)}, x_p \rangle) \\
& + \sum_{r \in [m]} \sum_{p \in [P] \setminus \calP(X)} \trelu(\langle w_{i,r}^{(T)}, x_p \rangle) \\
=& \sum_{r \in [m]} \sum_{p \in \mathcal{P}_{v_{i,1}}(X) \cup \mathcal{P}_{v_{i,2}}(X)} \trelu(\langle w_{i,r}^{(T)}, v_{i,l} \rangle z_p \pm \tilde{o}(\sigma_0)) + \sum_{r \in [m]} \sum_{p \in \mathcal{P}(X) \setminus \mathcal{P}_{v_{i,1}}(X) \cup \mathcal{P}_{v_{i,2}}(X)} \trelu(\tilde{O}(\sigma_0)) \\
& + \sum_{r \in [m]} \sum_{p \in [P] \setminus \calP(X)} \trelu(\langle w_{i,r}^{(T)}, x_p \rangle) \\
=& \sum_{l \in [2]} \sum_{r \in [m]} [\langle w_{i,r}^{(T)}, v_{i,l}\rangle]^{+} \cdot \sum_{p \in \calP_{v_{i,l}}(X)} z_p
\pm \tilde{O}(\sigma_0 m) + \tilde{O}(\sigma_0^q sm) + \sum_{r \in [m]} \sum_{p \in [P] \setminus \calP(X)} \trelu(\langle w_{i,r}^{(T)}, x_p \rangle) 
\\
=& \sum_{l \in [2]} \big( \Phi_{i, l}^{(T)} \cdot Z_{i, l}^{(T)}(X) \big) + \sum_{r \in [m]} \sum_{p \in [P] \setminus \calP(X)} \trelu(\langle w_{i,r}^{(T)}, x_p \rangle) 
\pm O\left(\frac{1}{\polylogk}\right).
\end{aligned}
\end{equation}

According to \citet{allen-zhu2023towards}, for vanilla SL, after training for \(T=\frac{\polyk}{\eta}\) iterations, we have
\begin{equation} \label{appeq:a2_vsl}
0.4 \Phi_{i}^{\text{SL}} - \Phi_{j}^{\text{SL}} \le -\Omega(\log(k)), \quad 0.6 \Phi_{i}^{\text{SL}} \ge \Omega(\log(k)).
\end{equation}
By \cref{claim:a2_individual}, for SL with \(\calA_2\), after training for \(T=\frac{\polyk}{\eta}\) iterations, we have
\begin{equation} \label{appeq:a2_res}
(0.4+C_3) \Phi_{i}^{\calA_2} - (1-C_2)\Phi_{j}^{\calA_2} \le -\Omega(\log(k)), \quad (0.6-C_2-C_3) \Phi_{i}^{\calA_2} \ge \Omega(\log(k)).
\end{equation}
Now we consider the noisy data distribution \(\calD_{\text{noisy}}\) defined in \cref{def:noisy_data}, which shares the same distribution for semantic and noisy feature patches with \cref{def:data2} (\(x_p \text{ for } p \in \calP(X)\)), but with a larger level of noise for purely noise patches (\(x_p \text{ for } p \in [P]\setminus\calP(X)\)).
Next, we analysis the value of \(\sum_{r \in [m]} \sum_{p \in [P] \setminus \calP(X)} \trelu(\langle w_{i,r}^{(T)}, x_p \rangle)\) under \(\calD_{\text{noisy}}\) as follows:
\begin{equation}
\begin{aligned}
& \sum_{r \in [m]} \sum_{p \in [P] \setminus \calP(X)} \trelu(\langle w_{i,r}^{(T)}, x_p \rangle) \\
=& \sum_{r \in [m]} \sum_{p \in [P] \setminus \calP(X)} \trelu(\langle w_{i,r}^{(T)}, \sum_{v' \in \mathcal{V}} \alpha_{p,v'} v' + \xi_p \rangle) \\
=& \sum_{r \in [m]} \sum_{p \in [P] \setminus \calP(X)} \trelu(\langle w_{i,r}^{(T)}, \xi_p \rangle) + \tilde{O}(\gamma) + \tilde{O}((\sigma_0\gamma k )^q  m P).
\end{aligned}
\end{equation}
Recall \(\xi_p \sim \mathcal{N}(0, \sigma_n \mathbf{I})\) with \(\sigma_n=\frac{\polylogk}{d}\), we know
\begin{equation} \label{appeq:a2_low}
\Pr_{(X,y)\sim\calD_{\text{noisy}}}\left[ \sum_{r \in [m]} \sum_{p \in [P] \setminus \calP(X)} \trelu(\langle w_{i,r}^{(T)}, \xi_p \rangle) \geq \Omega(\log k)\right]\geq 1 - o(1).
\end{equation}
On the other hand, we also have that
\begin{equation} \label{appeq:a2_high}
\Pr_{(X,y)\sim\calD_{\text{noisy}}}\left[ \sum_{r \in [m]} \sum_{p \in [P] \setminus \calP(X)} \trelu(\langle w_{i,r}^{(T)}, \xi_p \rangle) \leq \polylogk \right]\geq  1 - o(1).
\end{equation}
Therefore, taking \eqref{appeq:a2_vsl}, \eqref{appeq:a2_res}, \eqref{appeq:a2_low}, \eqref{appeq:a2_high} into \eqref{appeq:a2_f},  we have for any \(C_2+C_3 \in (0,0.6)\), the following holds:
\begin{equation}
\Pr_{(X,y)\sim\calD_{\text{noisy}}}\left[ F^{\calA_2}_y(X)\geq \max_{j\neq y}F^{\calA_2}_j(X)\right] > \Pr_{(X,y)\sim\calD_{\text{noisy}}}\left[ F^{\text{SL}}_y(X)\geq \max_{j\neq y}F^{\text{SL}}_j(X)\right].
\end{equation}
Furthermore, if \(0.6-C_2-C_3 = O(\frac{1}{\polylogk})\), which means after feature mixing, the scale of semantic features and noisy features are very close. In this case, we have the following result:
\begin{equation}
\Pr_{(X,y)\sim \calD_{\text{noisy}}}\left[ F^{\calA_2}_y(X)\geq \max_{j\neq y}F^{\calA_2}_j(X)+\Omega(\log k)\right]\geq 1-e^{-\Omega(\log^2k)}
\end{equation}
and
\begin{equation}
\Pr_{(X,y)\sim \calD_{\text{noisy}}}\left[ F^{\text{SL}}_y(X)\geq \max_{j\neq y}F^{\text{SL}}_j(X)-\Omega(\log k)\right]\leq o(1).
\end{equation}

\section{Proof for SL with $\calA_3$}
\label{appsec:proof_a3}
Here, we provide the proof for SL with data augmentation \(\mathcal{A}_3\). Since the proof follows the same framework as the proof for SL with \(\mathcal{A}_1\) and adheres to the same \cref{hyp}, we present only the key differences to avoid redundancy.

\subsection{Key Claims}
Recall that for SL with data augmentation \(\mathcal{A}_3\), we use \(\tilde{\mathcal{Z}}_3 = \mathcal{Z}_m' \cup \tilde{\mathcal{Z}}_m \cup \mathcal{Z}_s' \cup \tilde{\mathcal{Z}}_s\) to denote the training dataset after data augmentation. Here, \(\mathcal{Z}_m'\) and \(\tilde{\mathcal{Z}}_m = \tilde{\mathcal{Z}}_m^1 \cup \tilde{\mathcal{Z}}_m^2\) originate from the multi-view training dataset \(\mathcal{Z}_m\) and represent samples with only noise manipulation or combining feature mixing with partial removal effects, respectively, as described in \cref{assum:a3}. Similarly, \(\mathcal{Z}_s'\) and \(\tilde{\mathcal{Z}}_s\) are derived from the single-view training dataset \(\mathcal{Z}_s\) and contain samples with only noise variation or combined feature manipulation, respectively.

The first key difference in the proof lies in the convergence of the error on multi-view data from \(T_0\) until the end of training.
\begin{claim}[multi-view error till the end]
\label{claim:a3_multi_error}
    Suppose that Induction Hypothesis~\ref{hyp} holds for every iteration $t\le T$, then 
    \begin{align*}
         \sum_{t=T_0}^{T} \bbE_{(X,y)\sim\calZ_{m}'}[1-\logit_y(F^{(t)}, X)] &\leq \tilde{O}\left(\frac{k}{\eta}\right),\\
         \sum_{t=T_0}^{T} \bbE_{(X,y)\sim\tilde{\calZ}_{m}}[1-\logit_y(F^{(t)}, X)] &\leq \tilde{O}\left(\frac{k}{\eta\pi_3}\right).
    \end{align*}
\end{claim}
\begin{proof}
Fix \(i\in[k]\) and \(l\in [2]\), we know for any \(w_{i,r}\) (\(r\in[m]\)), we have
\begin{equation}\label{eq:a3_multi_error1}
    \begin{aligned}
        \langle w_{i,r}^{(t+1)}, v_{i,l}\rangle 
        =& \langle w_{i,r}^{(t)}, v_{i,l}\rangle - \eta \bbE_{(X,y)\sim\tcalZ_3}\big[\langle \nabla_{w_{i,r}} L(F^{(t)};X,y) , v_{i,l}\rangle\big] \\
        \\=& \langle w_{i,r}^{(t)}, v_{i,l}\rangle - \frac{\eta(N-N_s)(1-\pi_3)}{N} \bbE_{(X,y)\sim\calZ_m'}[\langle \nabla_{w_{i,r}} L(F^{(t)};X,y), v_{i,l}\rangle]\\ 
        &- \frac{\eta(N-N_s)\pi_3}{N} \bbE_{(X,y)\sim\tcalZ_m}[\langle \nabla_{w_{i,r}} L(F^{(t)};X,y), v_{i,l}\rangle] \\
        &- \frac{\eta N_s(1-\pi_3)}{N} \bbE_{(X,y)\sim\calZ_s'}[\langle \nabla_{w_{i,r}} L(F^{(t)};X,y), v_{i,l}\rangle] \\
        &- \frac{\eta N_s\pi_3}{N} \bbE_{(X,y)\sim\tcalZ_s}[\langle \nabla_{w_{i,r}} L(F^{(t)};X,y), v_{i,l}\rangle]
    \end{aligned}
\end{equation}
    Take \(r = \argmax_{r\in [m]} \{\langle w_{i,r}^{(t)}, v_{i,l}\rangle\}\), then by \(m=\polylogk\) we know \(\langle w_{i,r}^{(t)}, v_{i,l}\rangle \ge \tilde{\Omega}(\Phi_{i,l}^{(t)})=\tilde{\Omega}(1)\) for \(t\ge T_0\). Same as in the proof in \cref{claim:multi_error}, we know when \(t\ge T_0\), \(\langle w_{i,r}^{(t)}, v_{i,l}\rangle \ge \tilde{\Omega}(1) \gg \varrho\), and \(|\calP_{v_{i,l}}(X)|\leq O(1)\), for most of \( p\in \calP_{v_{i,l}}(X) \) must be already in the linear regime of \(\trelu\), which means we have
    \begin{align*}
        0.9 \sum\nolimits_{p\in\calP_{v_{i,l}}(X) }z_p \leq V_{i,r,l}(X) \leq \sum\nolimits_{p\in\calP_{v_{i,l}}(X) }z_p.
    \end{align*}
    
    Thus, for $(X,y)\sim\calZ_m'$, when $y=i$, we have $V_{i,r,l}(X)\geq 0.9$, when $y\neq i$ and $v_{i,l}\in\calV(X)$, we have $V_{i,r,l}(X)\leq 0.4$. For $(X,y)\sim\tcalZ_m^l$, when $y=i$, we have $V_{i,r,l}(X)\geq 0.9-C_2$, when $y\neq i$ and $v_{i,l} \in \calV(X)$, we have $V_{i,r,l}(X)\leq 0.4 + C_3$. For $(X,y)\sim\tcalZ_m^{3-l}$, when $y=i$, we have $V_{i,r,l}(X) \ge 0.9 \cdot (C_1-C_2)$, when $y\neq i$ and $v_{i,l}\in\calV(X)$, we have $V_{i,r,l}(X)\leq 0.4 + C_3$. For $(X,y)\sim\calZ_s'$, when $y=i$, we have $V_{i,r,l}(X)\ge 0.9$ if \(v_{i,l}\) is the class-specific feature contained in \(X\) else $V_{i,r,l}(X)\le O(\rho) \ll o(1)$, when $y\neq i$ and $v_{i,l}\in\calP(X)$, we have $V_{i,r,l}(X)\leq O(\Gamma) \ll o(1)$. For $(X,y)\sim\tcalZ_s$, when $y=i$, we have $V_{i,r,l}(X)\geq 0.9 \cdot (C_1-C_2)$ if \(v_{i,l}\) is the class-specific feature contained in \(X\) else $V_{i,r,l}(X)\le O(\rho) \ll o(1)$, when $y\neq i$ and $v_{i,l}\in\calP(X)$, we have $V_{i,r,l}(X)\leq C_3$.

    Recall that $\Pr(v_{i,l}\in\calP(X)|i\neq y)=\frac{s}{k}\ll o(1)$, using \cref{claim:pos} and \cref{claim:neg}, we can derive that for \(\calZ_m'\), same as \cref{claim:multi_error}, we have \eqref{eq:multi_error2}.
    Then for \(\tcalZ_m\), we have
\begin{equation}\label{eq:a3_multi_error3}
    \begin{aligned}
    & -\bbE_{(X,y)\sim\tcalZ_m}[\langle \nabla_{w_{i,r}} L(F^{(t)};X,y), v_{i,l}\rangle]\\
    =& -\frac{1}{2}\bbE_{(X,y)\sim\tcalZ_m^l}[\langle \nabla_{w_{i,r}} L(F^{(t)};X,y), v_{i,l}\rangle] -\frac{1}{2}\bbE_{(X,y)\sim\tcalZ_m^{3-l}}[\langle \nabla_{w_{i,r}} L(F^{(t)};X,y), v_{i,l}\rangle]\\
    \ge&\frac{1}{2}\mathbb{E}_{(X,y) \sim \tcalZ_m^l} [ (0.89-C_2) \cdot \mathbb{I}_{y=i} ( 1 - \logit_i(F^{(t)}, X) ) - (0.41+C_3) \cdot \frac{s}{k} \mathbb{I}_{y \neq i}  \logit_i(F^{(t)}, X)] \\ 
    & + \frac{1}{2}\mathbb{E}_{(X,y) \sim \tcalZ_m^{3-l}} [ 0.89(C_1-C_2) \cdot \mathbb{I}_{y=i} ( 1 - \logit_i(F^{(t)}, X) ) - (0.41+C_3) \cdot \frac{s}{k} \mathbb{I}_{y \neq i}  \logit_i(F^{(t)}, X)] \\
    \ge&\mathbb{E}_{(X,y) \sim \tcalZ_m^l} [ (0.44-C_2) \cdot \mathbb{I}_{y=i} ( 1 - \logit_i(F^{(t)}, X) ) - (0.41+C_3) \cdot \frac{s}{k} \mathbb{I}_{y \neq i}  \logit_i(F^{(t)}, X)] \\ 
    \ge& \tilde{\Omega}(\frac{1}{k}) \mathbb{E}_{(X,y) \sim \tcalZ_m^l} [ 1 - \logit_i(F^{(t)}, X)].
    \end{aligned}
\end{equation}
    The second last step is due to \(\calA_1\) has equal probability to partial remove feature \(v_{i,l}\) or \(v_{i,3-l}\) from \((X,y)\sim \calZ_m\) to generate a sample in \(\tcalZ_m^l\) or \(\tcalZ_m^{3-l}\), so there is symmetry between \(\tcalZ_m^l\) and \(\tcalZ_m^{3-l}\).

    Similarly, we can derive that for \(\calZ_s'\) we have \eqref{eq:multi_error4} same as \cref{claim:multi_error}, and for \(\tcalZ_s\) we have
\begin{equation}\label{eq:a3_multi_error4_1}
    \begin{aligned}
    & -\bbE_{(X,y)\sim\tcalZ_s}[\langle \nabla_{w_{i,r}} L(F^{(t)};X,y), v_{i,l}\rangle]\\
    \ge& \frac{1}{2} \mathbb{E}_{(X,y) \sim \tcalZ_s^l} [ 0.89 (C_1-C_2) \cdot \mathbb{I}_{y=i} ( 1 - \logit_i(F^{(t)}, X) ) - C_3 \frac{s}{k} \mathbb{I}_{y \neq i}  \logit_i(F^{(t)}, X)] \\
    & + \frac{1}{2}\mathbb{E}_{(X,y) \sim \tcalZ_s^{3-l}} [ - \tilde{O}(\sigma_p P) \mathbb{I}_{y=i} ( 1 - \logit_i(F^{(t)}, X) ) - C_3 \frac{s}{k} \mathbb{I}_{y \neq i}  \logit_i(F^{(t)}, X)] \\
    \ge&\mathbb{E}_{(X,y) \sim \tcalZ_s^l} [ \Omega(1) \cdot \mathbb{I}_{y=i} ( 1 - \logit_i(F^{(t)}, X) ) - C_3 \frac{s}{k} \mathbb{I}_{y \neq i}  \logit_i(F^{(t)}, X)] \\ 
    \ge& \tilde{\Omega}(\frac{1}{k}) \mathbb{E}_{(X,y) \sim {\tcalZ_s}^l} [ 1 - \logit_i(F^{(t)}, X)] \ge 0.
    \end{aligned}
\end{equation}
    Here we use \({\calZ_s'}^l\) (\(\tcalZ_s^l\)) to denote \((X,y) \sim \calZ_s'\) (\(\tcalZ_s\)) that has \(v_{y,l}\) as the only semantic feature, and \({\calZ_s'}^{3-l}\) (\(\tcalZ_s^{3-l}\)) to denote \((X,y) \sim \calZ_s'\) (\(\tcalZ_s\)) that has \(v_{y,3-l}\) as the only semantic feature. Now we take \eqref{eq:multi_error2}, \eqref{eq:a3_multi_error3}, \eqref{eq:multi_error4}, \eqref{eq:a3_multi_error4_1} into \eqref{eq:a3_multi_error1}, we have 
\begin{equation}\label{eq:a3_multi_error5}
    \begin{aligned}
        \langle w_{i,r}^{(t+1)}, v_{i,l}\rangle 
        \ge& \langle w_{i,r}^{(t)}, v_{i,l}\rangle + \frac{(N-N_s)(1-\pi_3)}{N} \tilde{\Omega}(\frac{\eta}{k}) \mathbb{E}_{(X,y) \sim \calZ_m'} [ 1 - \logit_i(F^{(t)}, X)] \\ 
        &+ \frac{(N-N_s)\pi_3}{N} \tilde{\Omega}(\frac{\eta}{k}) \mathbb{E}_{(X,y) \sim \tcalZ_m} [ 1 - \logit_i(F^{(t)}, X)].
    \end{aligned}
\end{equation}
Since we have \(\pi_3\ge \frac{1}{\polylogk}\), \(m=\polylogk\), and \(N \ge N_s \polyk\), when summing up all \(r\in [m]\), we have 
\begin{equation}\label{eq:a3_multi_error6}
    \begin{aligned}
    \Phi_{i,l}^{(t+1)} \ge & \Phi_{i,l}^{(t)} + \tilde{\Omega}(\frac{\eta}{k}) \mathbb{E}_{(X,y) \sim \calZ_m'} [ 1 - \logit_i(F^{(t)}, X)] \\ 
    &+ \tilde{\Omega}(\frac{\eta\pi_3}{k}) \mathbb{E}_{(X,y) \sim \tcalZ_m} [ 1 - \logit_i(F^{(t)}, X)].
    \end{aligned}
\end{equation}
By telescoping \eqref{eq:a3_multi_error6} from \(T_0\) to \(T\) and applying \(\Phi_{i,l} \leq \tilde{O}(1)\) from the Induction Hypothesis~\ref{hyp}(d), we obtain the result stated in this claim.
\end{proof}

Next, we present a claim that complements to the multi-view individual error bound established in Claim D.16 of \citet{allen-zhu2023towards}. The following claim states that when training error on \(\tcalZ_m\) is small enough, the model has high probability to correctly classify any individual single-view data.
\begin{claim}[single-view individual error]
\label{claim:a3_individual}
When $\mathbb{E}_{(X,y) \sim \tcalZ_m} \left[1 - \logit_y \left( F^{(t)}, X \right) \right] \leq \frac{1}{k^4}$ is sufficiently small, we have for any \(i,j\in[k],l,l'\in [2]\),
\[
\begin{aligned}
& (0.8+2C_3) \Phi_{i,l}^{(t)} - (1+C_1-2C_2)\Phi_{j,l'}^{(t)} \le -\Omega(\log(k)), \\
& (0.1+C_1/2-C_2-C_3) \Phi_{i}^{(t)} \ge \Omega(\log(k)) \quad \text{and} \quad \Phi_{i,l}^{(t)} \ge \Omega(\log(k)),
\end{aligned}
\]
and therefore for every \((X,y)\in \calZ_s',\tcalZ_s\), and every \((X,y) \sim \calD_s\), 
\[
F_y^{(t)}(X)\geq \max_{j\neq y}F_j^{(t)}(X)+\Omega(\log k).
\]
\end{claim}

\begin{proof}
Denote by \(\tcalZ_*^l\) for the set of sample $(X,y) \in \tcalZ_m^l$ such that \(\sum_{p \in P_{v_{y,l}}(X)} z_p \leq 1 - C_2 + \frac{1}{100 \log(k)}\) and \(\sum_{p \in P_{v_{y,3-l}}(X)} z_p \leq C_1 - C_2 + \frac{1}{100 \log(k)}\). For a sample $(X,y) \in \tcalZ_*^l$, denote by $\mathcal{H}(X)$ as the set of all $i \in [k] \setminus \{y\}$ such that
\(\sum_{l \in [2]} \sum_{p \in P_{v_{i,l}}(X)} z_p \geq 0.8 + 2C_3 - \frac{1}{100 \log(k)}\).

Now, suppose $1 - \logit_y( F^{(t)}, X) = \mathcal{E}(X)$, with $\min(1,\beta) \leq 2(1 - \frac{1}{1 + \beta})$, we have
\[
\min ( 1, \sum\nolimits_{i \in [k] \setminus \{y\}} e^{F_i^{(t)}(X) - F_y^{(t)}(X)} ) \leq 2 \mathcal{E}(X)
\]

By Claim \ref{claim:func}, \cref{claim:logit} and our definition of $\mathcal{H}(X)$, this implies that
\[
\min ( 1, \sum\nolimits_{i \in \mathcal{H}(X)} e^{(0.4 + C_3) \Phi_i^{(t)} - (1-C_2) \Phi_{y,l}^{(t)} - (C_1-C_2) \Phi_{y,3-l}^{(t)}}) \leq 4 \mathcal{E}(X).
\]
If we denote by $\psi = \mathbb{E}_{(X,y) \sim \tcalZ_m} [ 1 - \logit_y \left( F^{(t)}, X \right) ]$, then
\begin{align*}
&\mathbb{E}_{(X,y) \sim \tcalZ_m^l} \left[ \min ( 1, \sum\nolimits_{i \in \mathcal{H}(X)} e^{(0.4 + C_3) \Phi_i^{(t)} - (1-C_2) \Phi_{y,l}^{(t)} - (C_1-C_2) \Phi_{y,3-l}^{(t)}} ) \right] \leq O(\psi), \\
\implies &\mathbb{E}_{(X,y) \sim \tcalZ_m^l} \left[ \sum\nolimits_{i \in \mathcal{H}(X)} \min (\frac{1}{k}, e^{(0.4 + C_3) \Phi_i^{(t)} - (1-C_2) \Phi_{y,l}^{(t)} - (C_1-C_2) \Phi_{y,3-l}^{(t)}}) \right] \leq O(\psi).
\end{align*}

Notice that we can rewrite the LHS so that
\begin{align*}
&\mathbb{E}_{(X,y) \sim \tcalZ_m^l} \left[ \sum\nolimits_{j \in [k]} \bbI_{j=y} \sum\nolimits_{i \in [k]}  \bbI_{i \in \mathcal{H}(X)} \min ( \frac{1}{k}, e^{(0.4 + C_3) \Phi_i^{(t)} - (1-C_2) \Phi_{j,l}^{(t)} - (C_1-C_2) \Phi_{j,3-l}^{(t)}} ) \right] \leq O (\psi), \\
\implies &\sum\nolimits_{j \in [k]} \sum\nolimits_{i\in [k]} \bbI_{i\neq y} \mathbb{E}_{(X,y) \sim \tcalZ_m^l} \left[ \bbI_{j=y} \bbI_{i \in \mathcal{H}(X)}\right] \min ( \frac{1}{k}, e^{(0.4 + C_3) \Phi_i^{(t)} - (1-C_2) \Phi_{j,l}^{(t)} - (C_1-C_2) \Phi_{j,3-l}^{(t)}} )  \leq O(\psi).
\end{align*}
Note for every $i \neq j \in [k]$, the probability of generating a sample $(X,y) \in \tcalZ_*^l$ with $y = j$ and $i \in \mathcal{H}(X)$ is at least $\tilde{\Omega}(\frac{1}{k} \cdot \frac{s^2}{k^2})$. This implies
\[
 \sum\nolimits_{i \in [k] \setminus \{j\}} \min ( \frac{1}{k}, e^{(0.4 + C_3) \Phi_i^{(t)} - (1-C_2) \Phi_{j,l}^{(t)} - (C_1-C_2) \Phi_{j,3-l}^{(t)}} ) \leq \tilde{O} \left( \frac{k^3}{s^2} \psi \right).
\]
Then, with $1 - \frac{1}{1+\beta} \leq \min(1,\beta)$, we have for every $(X,y) \in \tcalZ_m^l$ ($l\in [2]$),

\begin{equation}\label{eq:a3_single_error1}
\begin{aligned}
1 - \logit_y ( F^{(t)}, X ) \le & \min(1, \sum\nolimits_{i \in [k] \setminus \{y\}} 2 e^{(0.4 + C_3) \Phi_i^{(t)} - (1-C_2) \Phi_{y,l}^{(t)} - (C_1-C_2) \Phi_{y,3-l}^{(t)}}) \\ \leq &  k \cdot \sum\nolimits_{i \in [k] \setminus \{y\}} \min( \frac{1}{k}, e^{(0.4 + C_3) \Phi_i^{(t)} - (1-C_2) \Phi_{y,l}^{(t)} - (C_1-C_2) \Phi_{y,3-l}^{(t)}}) \leq \tilde{O} \left( \frac{k^4}{s^2} \psi \right).
\end{aligned}
\end{equation}

Thus, we can see that when \(\psi \le \frac{1}{k^4}\) is sufficiently small, we have for any \(i\in [k]\setminus \{y\} \text{ and } l\in [2]\),
\[
e^{(0.4 + C_3) \Phi_i^{(t)} - (1-C_2) \Phi_{y,l}^{(t)} - (C_1-C_2) \Phi_{y,3-l}^{(t)}} \le \frac{1}{k} \implies {(0.4 + C_3) \Phi_i^{(t)} - (1-C_2) \Phi_{y,l}^{(t)} - (C_1-C_2) \Phi_{y,3-l}^{(t)}} \le -\Omega(\log(k)).
\]
By symmetry and non-negativity of \(\Phi_{i,l}^{(t)}\), we know for any \(i,j\in[k],l\in [2]\), we have:
\begin{equation}\label{eq:a3_single_error2}
(0.4+C_3) \Phi_{i,1}^{(t)} + (0.4+C_3) \Phi_{i,2}^{(t)} - (1-C_2) \Phi_{j,l}^{(t)} - (C_1-C_2) \Phi_{j,3-l}^{(t)} \le -\Omega(\log(k)).
\end{equation}
Since we have \(C_1 < 0.4+C_2+C_3\), this implies for any \(i,j\in[k],l,l'\in [2]\):
\begin{equation}\label{eq:a3_single_error3}
\begin{aligned}
& (0.8+2C_3) \Phi_{i,l}^{(t)} - (1+C_1-2C_2)\Phi_{j,l'}^{(t)} \le -\Omega(\log(k)) \\
& (0.1+C_1/2-C_2-C_3) \Phi_{i}^{(t)} \ge \Omega(\log(k)) \quad \text{and} \quad \Phi_{i,l}^{(t)} \ge \Omega(\log(k)).
\end{aligned}
\end{equation}
Since \eqref{eq:a3_single_error3} holds for any \(i\in[k],l\in [2]\) at any iteration \(t\) such that $\mathbb{E}_{(X,y) \sim \tcalZ_m} \left[1 - \logit_y \left( F^{(t)}, X \right) \right] \leq \frac{1}{k^4}$, for \((X,y) \sim \calZ_s'\) (suppose \(v_{y,l^*}\) is its only semantic feature), by Claim \ref{claim:func} we have
\begin{equation}\label{eq:a3_single_error4}
\begin{aligned}
F_y^{(t)}(X) &\ge 1 \cdot \Phi_{y,l^*}^{(t)} - O\left(\frac{1}{\polylogk}\right) \ge \Omega(\log(k)), \\
F_j^{(t)}(X) &\le O(\Gamma) \cdot \Phi_{j,l}^{(t)} + O\left(\frac{1}{\polylogk}\right) \le O(1) \text{ for } j\neq y.
\end{aligned}
\end{equation}
Similarly, for \((X,y) \sim \tcalZ_s\) (suppose \(v_{y,l^*}\) as the only semantic feature), by \cref{claim:func} we have
\begin{equation}\label{eq:a3_single_error5}
\begin{aligned}
F_y^{(t)}(X) &\ge (C_1-C_2) \cdot \Phi_{y,l^*}^{(t)} - O\left(\frac{1}{\polylogk}\right), \\
F_j^{(t)}(X) &\le C_3 \cdot \Phi_{j,l}^{(t)} + O\left(\frac{1}{\polylogk}\right) \text{ for } j\neq y.
\end{aligned}
\end{equation}

Therefore, we have for \((X,y) \sim \calZ_s', \tcalZ_s\)
\begin{equation}\label{eq:a3_single_error6}
\begin{aligned}
F_y^{(t)}(X)\geq \max_{j\neq y}F_j^{(t)}(X)+\Omega(\log k).
\end{aligned}
\end{equation}
\end{proof}

\subsection{Proof of \cref{main:thm_a3}}
First, we prove the diverse feature learning of SL with \(\calA_3\) and its perfect test accuracy on clean dataet \(\calD\). Similar to \cref{appsec:hyp_proof}, we have \cref{hyp} holds for every \(t\le T\). Then, according to \cref{claim:a3_multi_error}, we have
\begin{align*}
\sum\nolimits_{t=T_0}^{T} \bbE_{(X,y)\sim\calZ_{m}'}[1-\logit_y(F^{(t)}, X)] &\leq \tilde{O}\left(\frac{k}{\eta}\right),\\
\sum\nolimits_{t=T_0}^{T} \bbE_{(X,y)\sim\tilde{\calZ}_{m}}[1-\logit_y(F^{(t)}, X)] &\leq \tilde{O}\left(\frac{k}{\eta\pi_3}\right).
\end{align*}

Then, we know when $T \geq \frac{\polyk}{\eta \pi_3}$,
\[
\frac{1}{T} \sum\nolimits_{t=T_0}^{T} \bbE_{(X,y)\sim\tilde{\calZ}_{m}} [1-\logit_y(F^{(t)}, X)] \leq \frac{1}{\polyk}.
\]

Moreover, since we are using full gradient descent and the objective function is $O(1)$-Lipschitz continuous, the objective value decreases monotonically. Specifically, this implies that
\[
\mathbb{E}_{(X,y) \sim \tcalZ_m} [1 - \logit_y(F^{(T)}, X)] \leq \frac{1}{\polyk}.
\]
for the last iteration $T$. Then for the test accuracy on clean dataset \(\calD\), recall from Claim \ref{claim:individual}, we have 
\[
(0.8+2C_3) \Phi_{i,l}^{(t)} - (1+C_1-2C_2)\Phi_{j,l'}^{(t)} \le -\Omega(\log(k)),  \quad \Phi_{i,l}^{(t)} \ge \Omega(\log(k)),
\]
for any \(i,j \in [k]\) and \(l,l' \in [2]\). This combined with the function approximation Claim \ref{claim:func} shows that with high probability $F_y^{(T)}(X) \geq \max_{j \neq y} F_j^{(T)}(X) + \Omega(\log k)$ for every $(X, y) \in \calD_m, \calD_s$, which implies that the test accuracy on both multi-view data and single-view data is perfect.

Now, we prove the robust feature learning of SL with \(\calA_3\) and its improved generalization on noisy dataset \(\calD_{\text{noisy}}\). Same as \eqref{appeq:a2_f}, we have the prediction function as 
\begin{equation}
\label{appeq:a3_f}
F_i^{(T)}(X) = \sum_{l \in [2]} \big( \Phi_{i, l}^{(T)} \cdot Z_{i, l}^{(T)}(X) \big) + \sum_{r \in [m]} \sum_{p \in [P] \setminus \calP(X)} \trelu(\langle w_{i,r}^{(T)}, x_p \rangle) 
\pm O\left(\frac{1}{\polylogk}\right).
\end{equation}
For vanilla SL \citep{allen-zhu2023towards}, after training for \(T=\frac{\polyk}{\eta}\) iterations, we have for any \(i,j\in [k]\):
\begin{equation} \label{appeq:a3_vsl}
0.4 \Phi_{i}^{\text{SL}} - \Phi_{j}^{\text{SL}} \le -\Omega(\log(k)), \quad 0.6 \Phi_{i}^{\text{SL}} \ge \Omega(\log(k)).
\end{equation}
By \cref{claim:a3_individual}, for SL with \(\calA_3\), after training for \(T=\frac{\polyk}{\eta}\) iterations, we have for any \(i,j\in [k]\) and \(l\in [2]\):
\begin{equation} \label{appeq:a3_res}
(0.8+2C_3) \Phi_{i,l}^{\calA_3} - (1+C_1-2C_2)\Phi_{j,l'}^{\calA_3} \le -\Omega(\log(k)), \quad
(0.1+C_1/2-C_2-C_3) \Phi_{i}^{\calA_3} \ge \Omega(\log(k)).
\end{equation}

Recall for \((X,y)\sim\calD_{\text{noisy}}\) and \(p\in [P]\setminus \calP(X)\), \(\xi_p \sim \mathcal{N}(0, \sigma_n \mathbf{I})\) with \(\sigma_n=\frac{\polylogk}{d}\), we know with probability \(1-o(1)\):
\begin{equation} \label{appeq:a3_low_high}
\sum_{r \in [m]} \sum_{p \in [P] \setminus \calP(X)} \trelu(\langle w_{i,r}^{(T)}, \xi_p \rangle) \geq \Omega(\log k), \quad
\sum_{r \in [m]} \sum_{p \in [P] \setminus \calP(X)} \trelu(\langle w_{i,r}^{(T)}, \xi_p \rangle) \leq \polylogk.
\end{equation}
Therefore, taking \eqref{appeq:a3_vsl}, \eqref{appeq:a3_res}, \eqref{appeq:a3_low_high} into \eqref{appeq:a3_f},  we have for any \(C_1,C_2,C_3\) satisfies conditions in \cref{appassum:a3}, the following holds:
\[
\Pr_{(X,y)\sim\calD_{\text{noisy}}}\left[ F^{\calA_3}_y(X)\geq \max_{j\neq y}F^{\calA_3}_j(X)\right] > \Pr_{(X,y)\sim\calD_{\text{noisy}}}\left[ F^{\text{SL}}_y(X)\geq \max_{j\neq y}F^{\text{SL}}_j(X)\right].
\]
Furthermore, if \(0.1+C_1/2-C_2-C_3 = O(\frac{1}{\polylogk})\), we have the following result:
\[
\Pr_{(X,y)\sim \calD_{\text{noisy}}}\left[ F^{\calA_3}_y(X)\geq \max_{j\neq y}F^{\calA_3}_j(X)+\Omega(\log k)\right]\geq 1-e^{-\Omega(\log^2k)}
\]
and
\[
\Pr_{(X,y)\sim \calD_{\text{noisy}}}\left[ F^{\text{SL}}_y(X)\geq \max_{j\neq y}F^{\text{SL}}_j(X)-\Omega(\log k)\right]\leq o(1).
\]

\end{document}